\documentclass{article}

\usepackage{microtype}
\usepackage{graphicx}
\usepackage{subcaption}
\usepackage{booktabs} 
\usepackage{multicol}
\usepackage{multirow}
\usepackage{adjustbox}
\usepackage{xcolor}
\usepackage{float}
\usepackage{dblfloatfix} 
\usepackage{enumitem}
\usepackage{etoc}
\usepackage{titletoc}

\usepackage{hyperref}

\newcommand{\ttree}{\texttt{TreeSearch}}
\newcommand{\tsearch}{\texttt{HillSearch}}
\newcommand{\tmaxsat}{\texttt{MaxSatSearch}}

\usepackage[preprint]{icml2026}

\usepackage{amsmath}
\usepackage{amssymb}
\usepackage{mathtools}
\usepackage{amsthm}
\usepackage[skins]{tcolorbox}

\usepackage[capitalize,noabbrev]{cleveref}

\theoremstyle{plain}
\newtheorem{theorem}{Theorem}[section]

\theoremstyle{definition}

\theoremstyle{remark}

\usepackage[textsize=tiny]{todonotes}

\icmltitlerunning{Failing to Explore: Language Models on Interactive Tasks}

\begin{document}
\etocdepthtag.toc{mtmain}
\twocolumn[
  \icmltitle{Failing to Explore: Language Models on Interactive Tasks}



  \icmlsetsymbol{equal}{*}

  \begin{icmlauthorlist}
    \icmlauthor{Mahdi JafariRaviz}{equal}
    \icmlauthor{Keivan Rezaei}{equal}
    \icmlauthor{Arshia Soltani Moakhar}{equal}
    \\ \vspace{1.5pt}
    \icmlauthor{Zahra Sodagar}{}
    \icmlauthor{Yize Cheng}{}
    \icmlauthor{Soheil Feizi}{}
    \\ \vspace{3pt}
    {\small Department of Computer Science, University of Maryland}
    \\ \vspace{3pt}
    {\small 
    \texttt{\{mahdij, krezaei, asoltan3, zsodagar, yzcheng, sfeizi\}@umd.edu}}
  \end{icmlauthorlist}

  

  \icmlcorrespondingauthor{Mahdi JafariRaviz}{mahdij@umd.edu}


  \vskip 0.3in
]



\printAffiliationsAndNotice{\icmlEqualContribution}

\begin{abstract}
We evaluate language models on their ability to \textit{explore} interactive environments under a limited interaction budget.
We introduce three \textit{parametric} tasks with controllable exploration difficulty,
spanning continuous and discrete environments.
Across state-of-the-art models,
we find systematic under-exploration and suboptimal solutions,
with performance often significantly worse than simple explore–exploit heuristic baselines and scaling weakly as the budget increases.
Finally, we study two interventions: splitting a fixed budget into $p$ parallel executions, which surprisingly improves performance despite a no-gain theoretical result for our tasks,
and periodically summarizing the interaction history, which preserves key discoveries and further improves exploration \footnote{Code is at \href{https://github.com/mahdi-jfri/explore-exploit-bench}{github.com/mahdi-jfri/explore-exploit-bench}.}.
\end{abstract}

\section{Introduction}
\label{sec:intro}
\begin{figure*}[t]
    \centering
    \includegraphics[width=\textwidth]{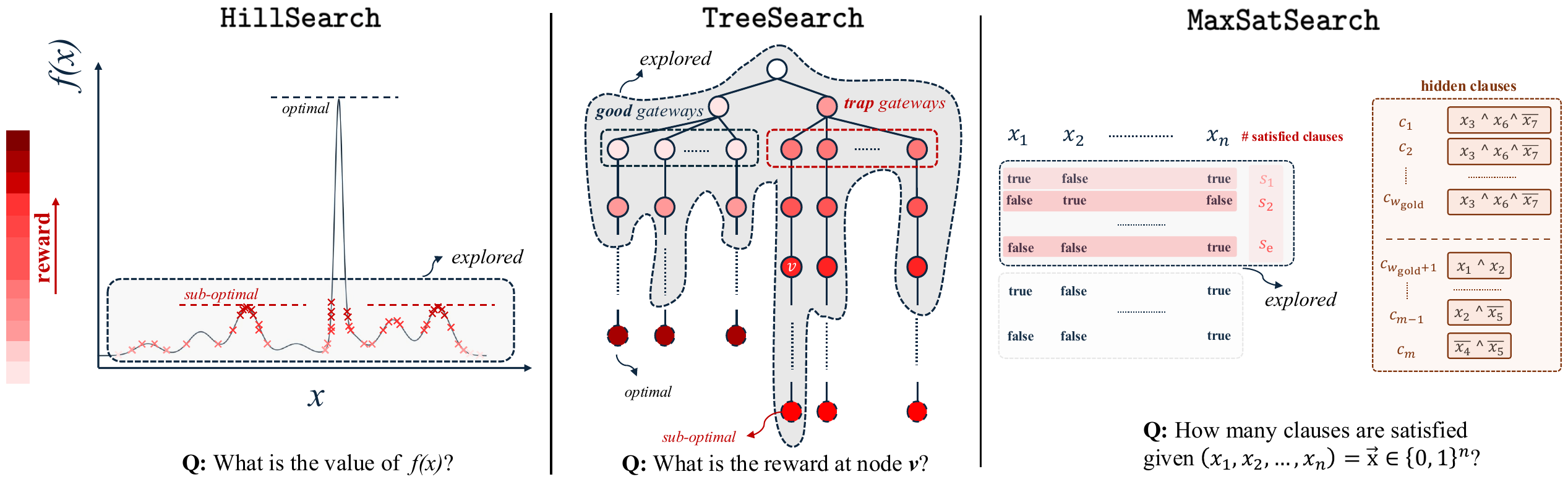}
    \caption{
    Overview of our environments.
    Across all tasks, the model \textbf{interacts with a partially observed environment} under a \textbf{limited budget} and must balance \textbf{exploration with identifying optimal solutions} while avoiding sub-optimal decisions.
    \textbf{Left (\tsearch{})}: there is a hidden function $f(x)$.
    At each interaction, the model observes the value of the function at a chosen point.
    The model should avoid identifying suboptimal solutions and explore sufficiently to find optimal points.
    \textbf{Middle (\ttree{})}: there is a tree in which each node has a hidden reward.
    At each interaction, the model can query the reward of a node that has already been explored or is adjacent to explored nodes.
    The model should not get trapped by early high rewards from trap gateways and must explore enough to eventually find optimal leaves of good gateways.
    \textbf{Right (\tmaxsat{})}: there are $m$ hidden clauses, among which one important clause exists and is repeated $w_{\text{gold}}$ times.
    At each iteration, the model queries the number of satisfied clauses for a given assignment of the variables.
    While the model may increase the number of satisfied clauses by modifying variable assignments,         achieving a high reward requires satisfying the important clause, necessitating sufficient exploration of assignments for variables involved.
    }
    \label{fig:teaser}
    \vspace{-5pt}
\end{figure*}

\textit{Agentic} AI is an increasingly prominent direction, where language models (LMs) act as agents that go beyond a single static response by iteratively taking actions—e.g., calling tools and APIs \cite{shen2024llm},
invoking other LMs \cite{tran2025multiagentcollaborationmechanismssurvey}, processing feedback, and choosing follow-up steps.
This paradigm is increasingly used in practical settings such as tool-using LMs \citep{schick2023toolformer, yao2023react}, web-based agents in realistic web environments \citep{yao2022webshop,zhou2024webarena,ning2025surveywebagentsnextgenerationai}, software engineering agents \citep{jimenez2024swebench,liu2025largelanguagemodelbasedagents}, and grounded robotic decision making \citep{ahn2022saycan,li2025largelanguagemodelsmultirobot}.
In these settings, agents operate in \textit{environments} with partial information and costly interactions (e.g., a physical robot action, an experiment, or a paid API call).
Achieving strong performance therefore requires effective use of a limited \textit{interaction budget}, which raises a critical need to evaluate whether LMs can act well in such environments under resource constraints.

Most current evaluations of LMs focus on static performance \cite{rein2023gpqagraduatelevelgoogleproofqa,romanou2024includeevaluatingmultilinguallanguage,salazar2025kaleidoscopeinlanguageexamsmassively,lin2025zebralogicscalinglimitsllms}.
In this work, we focus on \textit{discovery}: the model interacts with a black-box environment that contains hidden \textit{solutions},
and must use a limited budget of queries to an \textit{oracle} to iteratively learn about the environment and identify high-quality solutions.
More specifically, we study how well LMs can \textit{explore} such environments---i.e., how effectively they probe to discover better solutions, rather than committing early to sub-optimal ones.

We provide a suite of three \textit{parametric} tasks, \tsearch{}, \ttree{}, and \tmaxsat{}, each specifying a~family of problem instances that can be instantiated by choosing interpretable parameters; importantly, these parameters explicitly control the exploration required to discover high-quality solutions, making the explorative difficulty tunable.
These include discovering the maximum of a hidden continuous function (\tsearch{}),
discovering high-reward nodes in a tree with hidden node values (\ttree{}) and
discovering a high-satisfaction variable assignment under a hidden clause structure (\tmaxsat{})
(See Figure~\ref{fig:teaser} for an~illustration).
To require meaningful exploration, we explicitly incorporate \textit{traps}:
sub-optimal solutions that are easy to find early in the environment,
but can distract the model from continued exploration and lead to premature commitment.
All tasks are easy to describe in short natural language, resembling familiar problems in LMs' training data.
The model is given a total budget of $N$ interactions (i.e., $N$ rounds of querying the oracle) to gather information and discover solutions.

We evaluate a suite of recent LMs (including reasoning models) on instances of our task suite,
and compare their performance to simple explore--exploit heuristic baselines.
These baselines are intentionally lightweight (a few lines of code) and easy to describe,
and they serve as a strong sanity check that additional interaction budget indeed translates into improved discovery.
Our evaluation reveals consistent underperformance of LMs compared to these simple baselines.
This gap signals systematic \textit{under-exploration}:
models often commit early to trap solutions and fail to use additional interaction budget effectively.
In particular, we observe that performance scales weakly with $N$,
despite the fact that the environments contain better solutions that can be discovered with further exploration.

To improve performance in budgeted interactive tasks, we propose two interventions.
(1) {\textit{Parallel} budget allocation:} we split the total budget into $p$ independent threads, run the model separately in each, and select the best solution found across threads.
Although we prove that, on our tasks, parallelizing a fixed budget into $p$ threads cannot improve upon an optimal single-thread strategy, we nonetheless observe consistent gains for LMs, suggesting their single-thread behavior is far from optimal.
We further provide theory clarifying when such gains can arise.
(2) {Periodic \textit{summarization}:} we periodically summarize the interaction history every fixed number of interactions,
remove the full context, and continue from the summary.
This guides the model to retain the key takeaways while reducing context-related failures, improving discovery and increasing performance.

\paragraph{Contributions.}
Our contributions are as follows:
\begin{itemize}
    \item We go beyond static evaluation by introducing a suite of parametric, controllable tasks that make LM exploration \textit{measurable} and \textit{difficulty-tunable} across continuous and discrete/combinatorial environments.
    \item We find that LMs—including frontier GPT-5 models—underperform simple explore–exploit methods, committing prematurely to sub-optimal solutions and scaling weakly with increased interaction budget.
    \item We study two lightweight interventions—parallelizing a fixed budget into independent threads and periodically summarizing the interaction history—that consistently improve explorative performance.
\end{itemize}
\section{Related Work}
\label{sec:related_work}


\paragraph{LMs as combinatorial optimizers and algorithmic simulators}
Recent studies demonstrate that LMs can achieve strong performance at solving combinatorial problems \cite{sun2025co, da2025large, jiang2025large, yang2025heuragenix}, performing in-context algorithmic simulation \cite{zhou2022teachingalgorithmicreasoningincontext, lyu2024large}, and executing complex optimization tasks \cite{jiang2025largelanguagemodelsendtoend}.
Collectively, these capabilities suggest that LMs are well-equipped for the requirements of our study, providing a strong foundation for their application in our setting.

\paragraph{Summarization and parallelization for LM exploration}
Both parallelization and summarization have been widely studied for LMs.
Parallel sampling is commonly used to encourage exploration (e.g., pass@$k$).
However, prior work on parallel inference typically increases the total compute or budget rather than splitting a fixed budget across parallel branches~\cite{wang2023selfconsistencyimproveschainthought, chow2024inference}.
In contrast, our work allocates a fixed budget across branches. 
Other work trains models specifically for parallel reasoning~\cite{wen2025parathinker, zheng2025parallel}, which is outside the scope of this paper.
Summarization is widely used to compress history to fit within the context window~\cite{fei2024extending, wang2024context, wu2025resum}, but its impact on exploration behavior is rarely measured.

\paragraph{LMs on interactive benchmarks} Closest to our setup are single-player interactive environments and navigation-style tasks.
Single-player games and puzzles often require little exploration beyond deterministic reasoning or planning \cite{gong2024mindagent, long2025puzzleplex}.
Maze and navigation benchmarks study sequential decision-making under partial observability \cite{abdulhai2023lmrl, chevalier2018babyai, bianchi2024well, wu2023smartplay, xi2025agentgym, einarsson2025mazeevalbenchmarktestingsequential}, but typically emphasize goal reaching, memory, or control rather than explore–exploit tradeoffs.
Some other benchmarks (e.g., multi-turn puzzles, feedback-based learning, graph navigation, bandits, or black-box optimization) \cite{badola2025multiturnpuzzlesevaluatinginteractive, cheng2023llfbenchbenchmarkinteractivelearning, wu2023smartplay} are closer to our setup. However, they are primarily designed to evaluate multi-turn conversational behavior, in-context learning, or agentic capabilities, rather than exploration itself. Moreover, they do not provide a controlled, parameterized family of tasks that enables isolating and interpreting how LLMs explore.
\paragraph{Premature commitment and limited self-correction}
LMs often latch onto an initial idea and struggle to revise it. \citet{wen2025parathinker} show that seeding a rollout with the first few tokens of an incorrect trajectory can significantly reduce accuracy. Relatedly, \citet{huang2023large} describe this as limited self-correction,
and \citet{zhang2023languagemodelhallucinationssnowball} find that models often preserve consistency with earlier errors over factual truth.
\section{Exploring Interactive Environments}

\label{sec:problem_formulation}

In this work, we consider a model $M$ that interacts with an~environment induced by an underlying \textit{problem instance},
containing a set of solutions that are initially unrevealed (\textit{hidden})
and only become partially revealed through interaction.
The model $M$ begins with prior knowledge of the environment’s structure and the prescribed interaction protocol.
The environment is inherently explorative,
requiring the model to probe the problem instance to find high-quality solutions.
To simulate real-world interactive constraints, we impose a fixed interaction budget $N$.

Formally, the model $M$ interacts with an oracle $O$ over $N$ discrete rounds,
generating a sequence of queries $q_1, q_2, \dots, q_N$.
At each round $i$, the oracle returns feedback determined by the current query and the \textit{static} underlying problem instance.
Although the environment is static---in the sense that the oracle’s feedback is a deterministic function of the query and a fixed state---
the interaction is \emph{sequentially constrained}:
the set of admissible queries at round $i$ may depend on the previously issued queries $q_{<i}$.
The overall performance is measured by a reward function that quantifies the \textit{quality of the best solution discovered by the end of $N$ rounds},
reflecting the model's ability to explore effectively and avoid converging on suboptimal solutions.

As shown in Figure~\ref{fig:teaser}, we introduce a suite of three tasks:
\tsearch{}, which features a continuous environment;
\ttree{}, which features a discrete, graph-structured environment; and \tmaxsat{},
which features a discrete combinatorial environment.
We evaluate two aspects of a~model's behavior in these settings:
(i)~their absolute performance compared to mostly simple baseline algorithms that implement variants of explore--exploit strategies; and
(ii)~how increasing the search budget $N$ helps models achieve better solutions, and how this improvement compares to the growth observed for baseline algorithms.
Acceptable performance is characterized by satisfying {(i)}~through achieving comparable performance,
and {(ii)}~by exhibiting a similar growth rate as budget increases.


\section{Tasks}
\label{section:tasks}

In this section, we formally define our task suite.
For each task, we specify the environment induced by the underlying problem instance; the model's interaction protocol with the oracle;
the queries and the feedback revealed at each round;
and the final reward under a fixed interaction budget $N$.
We also describe the solution space and the parameterization of each problem instance, including how instances are generated.
Note that the task description is provided to the model before the interaction begins.
We also \textit{explicitly instruct the model to maximize the task reward }(as defined for each environment) under the interaction budget $N$.
For each task, we also provide a \textit{simple} explore--exploit algorithm as a baseline for interacting with the environment.

\subsection{\tsearch}
In this task, the problem instance is a hidden smooth function $f : [0,10] \to \mathbb{R}$ constructed as a sum of Gaussian \textit{hills}.
At each interaction round, the model queries a point $x \in [0,10]$, and the oracle reveals the value $f(x)$.
The final reward is the maximum value observed over $N$ rounds.

We construct $f$ to contain many moderate peaks (\textit{decoys}) and a single very high but narrow peak (\textit{needle}).
See Figure~\ref{fig:teaser} (Left) for an illustration of an interaction on this instance.
This setup can cause a model to focus its budget around an early local maximum, rather than exploring the domain sufficiently to locate the needle.
The task becomes harder as the needle becomes narrower (e.g., decreasing its Gaussian width / standard deviation),
since hitting it requires more effective exploration.
We refer to Appendix~\ref{app:subsec:task} for more details on the task setup.

\paragraph{A simple explore--exploit baseline}
We first query $\alpha N$ points drawn using stratified random sampling.
Let $\hat{x}$ denote the best point found so far.
We then use the remaining budget to query points drawn uniformly from a small window (of size $\beta$) around $\hat{x}$ and whenever a better point is found, $\hat{x}$ is updated accordingly
(i.e., we locally refine around the current best point). We use $\alpha = 0.8$ and $\beta = 0.5$.


\subsection{\ttree}
In this task, the problem instance is a rooted tree where each node has a hidden reward, and the oracle reveals rewards upon query.
At each interaction round, the model queries the reward of a node (the root's reward is $0$ and is revealed at the beginning),
and gradually increases its knowledge of the tree under the constraint that explored nodes must always form a connected component.
The final reward is the best (maximum) reward observed over the interaction episode.

We construct tree instances by assembling a set of disjoint chains.
Some chains, which we call \emph{traps}, yield relatively high reward early on
---providing short-term positive momentum by increasing for only a small number of nodes---
but then improve only sparsely and do not grow substantially along the remainder of the chain.
In contrast, \emph{good} chains start with low reward but increase steadily, reaching high (optimal) rewards near the leaves.
Throughout, we design rewards to vary smoothly along edges: the reward difference between any two adjacent nodes is at most $4$.

More specifically, the root has two types of children: \emph{trap gateways} and \emph{good gateways}.
From each gateway node, we attach \texttt{fanout} disjoint simple paths (i.e., independent chains).
Trap gateways expand into chains that are intentionally longer than those of good gateways (which are shorter).
As a result, once a model enters a trap chain, it tends to spend its budget traversing these nodes,
leaving less budget to explore good branches and reach higher-reward nodes.
Conversely, good gateways start low but improve along their chains, so reaching high-reward leaves requires sustained exploration.
See Figure~\ref{fig:teaser} (Middle) for an illustration.
Instances are harder when traps are more common, as this increases the search cost to locate a good chain.
Difficulty also increases when traps are more deceptive---for example, offering higher initial rewards or sustaining increases for more steps---as this delays the evidence that their long-term growth is limited.
See Appendix~\ref{app:subsec:task} for further details.

\textbf{A simple explore--exploit baseline}
As a simple heuristic, we randomly select a node to explore, preferring nodes whose parent achieved a higher reward.
We implement this preference with a softmax-style sampling rule: higher parent rewards make a node more likely to be chosen.
Node probabilities are defined via a softmax over parent rewards with temperature $\tau$. $\tau$ controls how strongly we favor high-reward parents.
We use $\tau=4$ in all experiments.

\subsection{\tmaxsat}

In this task, the problem instance comes from the Boolean Satisfiability (\textsc{Sat}) family of problems over $n$ variables with $m$ hidden clauses.
At each interaction round, the model proposes a full assignment
$\mathbf{x} = (x_1, x_2, \dots, x_n)$ where $x_i \in \{\texttt{true}, \texttt{false}\}$,
and the oracle reveals the number of satisfied clauses (out of $m$) under $\mathbf{x}$.
The final reward is the maximum number of satisfied clauses over the $N$ interactions (i.e., over the $N$ proposed assignments).

The hidden formula includes $m$ clauses, where each clause is an \textsc{And} of literals (variables or their negations).
We generate clauses so that achieving a high score requires exploration.
More specifically, we include a \emph{gold clause} containing $k_{\text{gold}}$ literals, and repeat this gold clause $w_{\text{gold}}$ times.
The remaining $m - w_{\text{gold}}$ clauses each contain $k_{\text{other}}$ variables and only have variables that do not appear in the gold clause.
We guarantee that there exists an assignment $\mathbf{x}^*$ that satisfies all $m$ clauses.
Crucially, obtaining a high score requires satisfying the gold clause, which amounts to finding the correct assignment to its $k_{\text{gold}}$ variables; since these variables do not appear in any other clause, the feedback from non-gold clauses provides no information about them, and the model must effectively \emph{explore} by trying assignments until it hits one that satisfies the gold clause.
See Figure~\ref{fig:teaser} (Right) for an illustration of the task and Appendix~\ref{app:subsec:task} for more details on it.
Difficulty increases with larger $k_{\text{gold}}$: a uniformly random assignment satisfies the gold clause with probability $2^{-k_{\text{gold}}}$,
making accidental discovery increasingly rare as $k_{\text{gold}}$ grows.
Difficulty also typically increases with larger $k_{\text{other}}$,
which introduces more hidden constraints that must be satisfied simultaneously.

\textbf{A simple explore--exploit baseline}
We first query $\alpha N$ random assignments, and then in the remaining rounds, we take the best assignment found so far, flip the value of one randomly chosen variable, and query the resulting assignment.
This gradually evolves the assignment toward better ones. We use $\alpha =0.5$ in our experiments.

\textbf{Details on baselines} We refer to Appendix~\ref{app:baseline_details} for baseline pseudo-code and hyperparameter choices.
\newcommand{\updelta}[1]{\textcolor{blue}{\tiny($\uparrow$#1\%)}}
\newcommand{\downdelta}[1]{\textcolor{red}{\tiny($\downarrow$#1\%)}}

\newcommand{\relbaseline}[1]{\textcolor{purple}{\tiny(#1\%)}}

\begin{table*}[t]
\centering
\scriptsize
\caption{
Reward of language models on our tasks: \textbf{they achieve sub-optimal reward across all tasks}, compared to simple {explore-exploit} baselines.
Results are reported on an instance of each task with various interaction budgets.
We report the model reward and compare it to the baseline.
Relative performance to the baseline is shown in {\tiny \textcolor{purple}{$(x\%)$}}.
The last three models are evaluated in reasoning mode.
}
\label{tab:absolute_p1}
\begin{adjustbox}{width=0.9\linewidth}
\begin{tabular}{l|cc||ccc||cc}
\toprule
 & \multicolumn{2}{c||}{\tsearch} & \multicolumn{3}{c||}{\ttree} & \multicolumn{2}{c}{\tmaxsat} \\
Model & $N=36$ & $N=48$ & $N=36$ & $N=48$ & $N=60$ & $N=36$ & $N=48$ \\
\midrule
\midrule
{\tiny \texttt{Qwen2.5-7B-Instruct}} & 0.30 \relbaseline{32} & 0.33 \relbaseline{34} & 0.71 \relbaseline{75} & 0.71 \relbaseline{74} & 0.72 \relbaseline{74} & 0.38 \relbaseline{50} & 0.45 \relbaseline{54} \\
{\tiny \texttt{Qwen3-4B-Instruct}} & 0.20 \relbaseline{21} & 0.27 \relbaseline{28} & 0.75 \relbaseline{80} & 0.91 \relbaseline{95} & 0.95 \relbaseline{98} & 0.41 \relbaseline{53} & 0.39 \relbaseline{47} \\
{\tiny \texttt{Qwen3-8B}} & 0.71 \relbaseline{75} & 0.59 \relbaseline{61} & 0.64 \relbaseline{68} & 0.68 \relbaseline{71} & 0.68 \relbaseline{70} & 0.52 \relbaseline{68} & 0.48 \relbaseline{57} \\
{\tiny \texttt{gemini-2.5-flash-lite}} & 0.36 \relbaseline{38} & 0.41 \relbaseline{43} & 0.64 \relbaseline{68} & 0.63 \relbaseline{66} & 0.76 \relbaseline{78} & 0.54 \relbaseline{70} & 0.62 \relbaseline{75} \\
{\tiny \texttt{gpt-5-nano}} & 0.41 \relbaseline{43} & 0.31 \relbaseline{32} & 0.69 \relbaseline{74} & 0.69 \relbaseline{72} & 0.79 \relbaseline{82} & 0.56 \relbaseline{73} & 0.62 \relbaseline{74} \\
{\tiny \texttt{gpt-5-mini}} & 0.85 \relbaseline{90} & 0.77 \relbaseline{79} & 0.67 \relbaseline{71} & 0.64 \relbaseline{67} & 0.73 \relbaseline{75} & 0.61 \relbaseline{79} & 0.67 \relbaseline{80} \\
{\tiny \texttt{gpt-5}} & 0.85 \relbaseline{90} & 0.72 \relbaseline{75} & 0.58 \relbaseline{62} & 0.63 \relbaseline{65} & 0.77 \relbaseline{79} & 0.68 \relbaseline{88} & 0.78 \relbaseline{93} \\
\midrule
{\tiny \texttt{Qwen3-8B-medium}} & 0.30 \relbaseline{32} & 0.29 \relbaseline{30} & 0.72 \relbaseline{77} & 0.70 \relbaseline{73} & 0.85 \relbaseline{88} & 0.48 \relbaseline{62} & 0.54 \relbaseline{64} \\
{\tiny \texttt{gpt-5-nano-medium}} & 0.48 \relbaseline{51} & 0.48 \relbaseline{49} & 0.65 \relbaseline{69} & 0.68 \relbaseline{70} & 0.82 \relbaseline{84} & 0.54 \relbaseline{70} & 0.54 \relbaseline{64} \\
{\tiny \texttt{gpt-5-mini-medium}} & 0.61 \relbaseline{64} & 0.58 \relbaseline{59} & 0.53 \relbaseline{56} & 0.54 \relbaseline{56} & 0.68 \relbaseline{70} & 0.59 \relbaseline{77} & 0.66 \relbaseline{79} \\
\midrule
{\tiny \texttt{explore-exploit baseline}} & 0.94 & 0.97 & 0.94 & 0.96 & 0.97 & 0.77 & 0.84 \\
\bottomrule
\end{tabular}
\end{adjustbox}
\end{table*}

\begin{figure*}[b]
  \centering
  \begin{subfigure}[t]{0.32\textwidth}
    \centering
    \includegraphics[width=\linewidth]{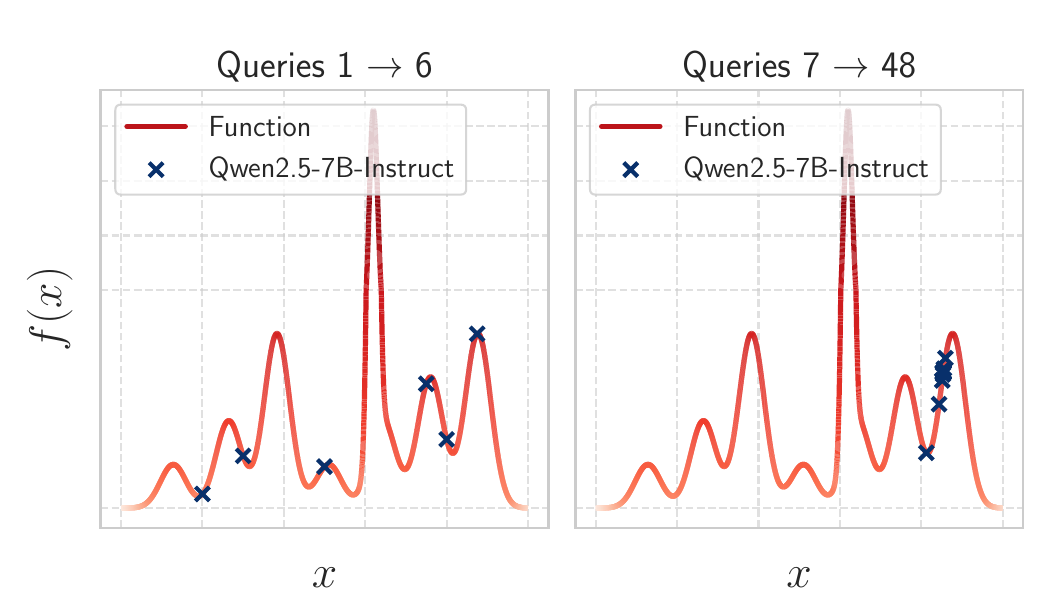}
    \caption{\tsearch}
  \end{subfigure}\hfill
  \begin{subfigure}[t]{0.32\textwidth}
    \centering
    \includegraphics[width=\linewidth]{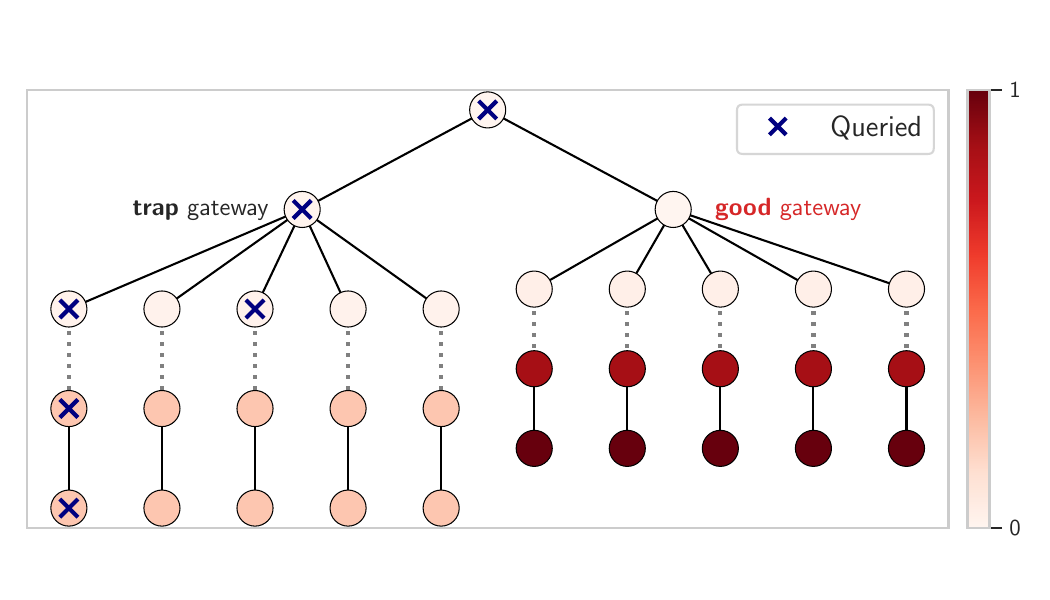}
    \caption{\ttree}
  \end{subfigure}\hfill
  \begin{subfigure}[t]{0.32\textwidth}
    \centering
    \includegraphics[width=\linewidth]{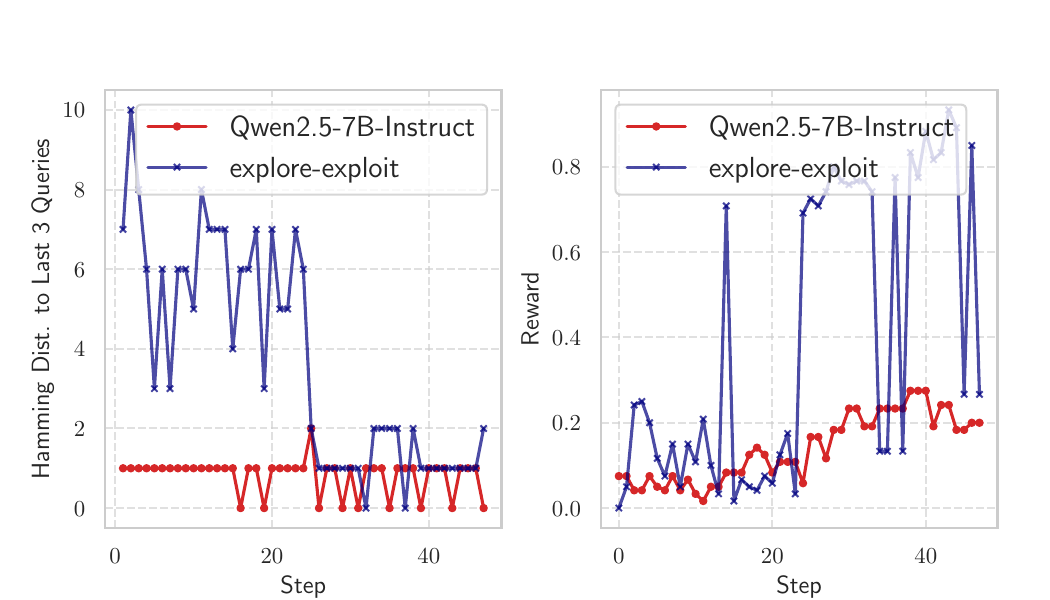}
    \caption{\tmaxsat}
  \end{subfigure}\hfill
  
  \caption{
        Visualization of \texttt{Qwen2.5-7B-Instruct} interactions during a failed episode in our environments ($N=48$).
        \textbf{Left (\tsearch):} Queries (crosses) on the hidden function. Early steps ($1$--$6$) explore the space, while later steps ($7$--$48$) cluster near local maxima.
        \textbf{Middle (\ttree):} Darker nodes denote higher reward. The model descends a trap gateway branch. 
        \textbf{Right (\tmaxsat):} Minimum Hamming distance to the last three queries and reward per step.
        The model makes only small variations with limited gains, whereas the baseline shifts from broad exploration to local refinement. See Appendix~\ref{app:interpretability} for additional visualizations.
  }
  \label{fig:interp}
\end{figure*}

\section{Experiments}
\label{sec:experiments}

In this section, we instantiate tasks from our suite and evaluate both LLMs and LRMs under a fixed interaction budget $N$.
We compare model performance against the simple explore--exploit baselines (discussed in \S\ref{section:tasks}) and study how reward scales with $N$ (\S\ref{subsec:eval}).
We then evaluate two lightweight interventions---parallel threads (\S\ref{subsec:parllel}) and periodic summarization (\S\ref{sec:summary})---and finally test robustness to task difficulty variations (\S\ref{subsec:difficulty}).
We defer full experimental details (task instances, models list, and interaction protocol) to the end of this section (\S\ref{subsec:experimental_setup}).

\subsection{Evaluating Models on Tasks}
\label{subsec:eval}

For each task, we run the models and the baselines with different interaction budgets ($N$).
We normalize rewards to the $[0,1]$ range by dividing the achieved rewards by the maximum possible reward of the instance.
Table~\ref{tab:absolute_p1} shows evaluation results on an instance of each task (see \S\ref{subsec:experimental_setup} for details).
As seen there, models are consistently and significantly outperformed by simple baselines that implement simple explore--exploit strategies,
across all tasks and budgets.
These results indicate that even when capable of reasoning, the models often fail to explore effectively and consequently converge to sub-optimal solutions within the environment.
We refer to Appendix~\ref{app:experiment_full:add-instances} and Appendix~\ref{app:fifty_instance} for experiments on $50$ instances of each task, which confirm the same overall trends.
Next,
we interpret a model's behavior (\texttt{Qwen2.5-7B-Instruct})
in our tasks by analyzing episodes where it achieves sub-optimal solutions.

\textbf{\tsearch{}: Models spend their budget around local maxima}
In \tsearch{}, models often commit early to a local maximum and then spend most of their remaining budget querying nearby points.
Figure~\ref{fig:interp}(a) visualizes the points queried by the model in an episode where the global maximum is not found.



\begin{figure*}[t]
  \centering
  \begin{subfigure}[t]{0.32\textwidth}
    \centering
    \includegraphics[width=\linewidth]{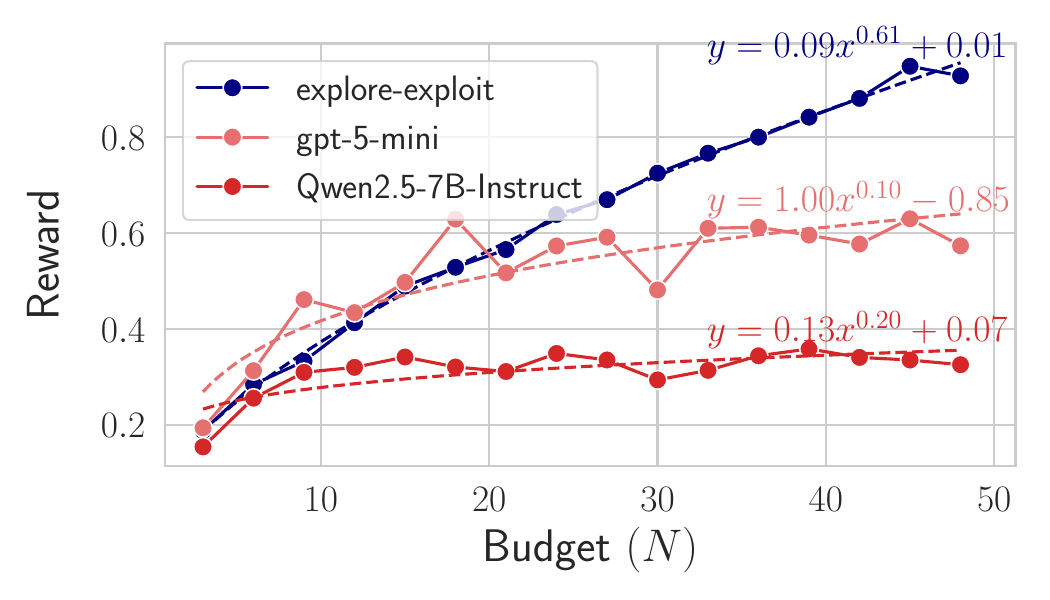}
    \caption{\tsearch}
  \end{subfigure}\hfill
  \begin{subfigure}[t]{0.32\textwidth}
    \centering
    \includegraphics[width=\linewidth]{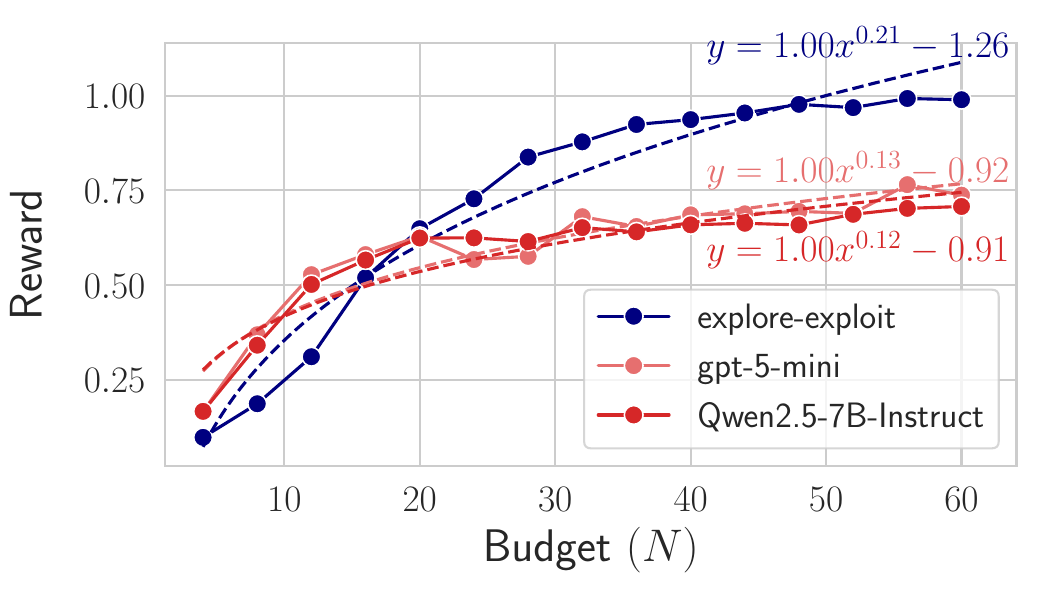}
    \caption{\ttree}
  \end{subfigure}
  \begin{subfigure}[t]{0.32\textwidth}
    \centering
    \includegraphics[width=\linewidth]{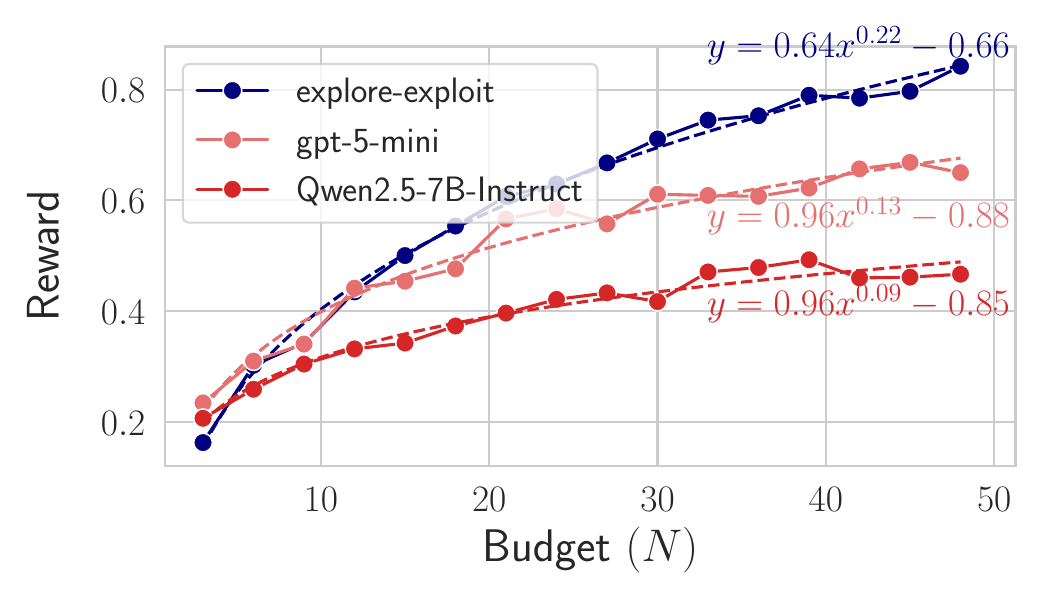}
    \caption{\tmaxsat}
  \end{subfigure}\hfill
  
  \caption{
  Scaling behavior of models and explore-exploit baseline rewards as a function of the interaction budget.
  Despite their simplicity, the baselines exhibit stronger reward growth as the budget increases.
  \textbf{In contrast, the models show limited improvement with additional budget, indicating inefficient use of interactions and a tendency to prematurely exploit sub-optimal solutions}.
  This gap suggests persistent under-exploration and poor budget utilization by the models.
  }
  \label{fig:scaling_law_curves}
  \vspace{-5pt}
\end{figure*}

\paragraph{\ttree{}: Models commit to depth-first traversal}
In \ttree{}, models often explore multiple branches but quickly commit to one once chosen.
Figure~\ref{fig:interp}(b) shows an episode where, after entering a branch, the model continues descending to a leaf regardless of branch quality.
As a result, if it enters a trap gateway, it spends most of its remaining budget on that path and ends at a low-reward leaf.

\vspace{-8pt}

\paragraph{\tmaxsat{}: Models explore locally around their initial guess}

For \tmaxsat{}, we observe that the model attempts to increase the number of satisfied clauses by modifying only a small number of variables at each step.
However, because our setup repeats a specific gold clause, failing to satisfy it prevents the model from achieving high rewards.
While a non-local change (e.g., random re-initialization) may satisfy the gold clause, local changes provide limited signal on how to do so.
Figure~\ref{fig:interp}(c) shows that the Hamming distance (defined as the number of variables with differing assignments) between consecutive interactions remains low throughout an episode.
This confirms that the model is restricted to local neighborhoods, yielding only incremental gains and missing the global optimum.

\paragraph{How do models leverage the interaction budget?}

To systematically understand how performance in our task setup grows with the interaction budget,
we evaluate \texttt{Qwen-2.5-7B-Instruct} and \texttt{gpt-5-mini} across tasks with different values of interaction budget ($N$).
Figure~\ref{fig:scaling_law_curves} shows the performance of the models and baselines as a function of $N$.
As seen there, model performance does not scale in the same way as the baselines; in particular, its growth rate is substantially smaller.
This suggests that additional budget is often not used effectively to achieve better results:
even with more interactions, the model struggles to improve its solutions, leading to an increasing gap from simple explore--exploit baselines as $N$ grows.

\subsection{Intervention: Parallel Threads}
\label{subsec:parllel}

As seen in Figure~\ref{fig:scaling_law_curves}, 
simply allowing the model to interact longer with the environment does not substantially improve performance.
We hypothesize that splitting the interaction budget into several independent threads,
and then merging their outcomes---i.e., taking the best solution found across threads
\footnote{This matches \texttt{best-of-$N$} selection in the RLHF literature: sample $N$ candidates and return the highest-scoring one under a proxy reward model \citep{stiennon2020summarize,liu2024statistical}.}
---may yield improved performance.
More specifically, we consider $p$ parallel threads with no interaction between them,
each equipped with an interaction budget of $N/p$ on the same problem instance.

In what follows, we prove that splitting the budget $N$ into $p$ independent threads and selecting the best outcome among them is a valid interaction strategy (\textit{parallel} method),
and that it provides no theoretical advantage over a single interaction trace of length $N$.
Therefore, this method can be compared fairly against other methods, including the standard single-thread interaction with budget $N$.

\vspace{-3pt}
\paragraph{\tsearch{}: Queries of parallel threads can be asked in a single-thread interaction}
Consider $p$ parallel threads, each querying $N/p$ points.
Let $[p] := \{1,2,\dots,p\}$ and $\mathcal{X} := [0,10]$ denote the query domain.
For each thread $i \in [p]$, let $X_i \subseteq \mathcal{X}$ denote the set of queried points, with $|X_i| \le N/p$.
Thread $i$ achieves final reward $\max_{x \in X_i} f(x)$.
We merge the solutions by taking the best one, achieving
$
f_{\text{parallel}}^* \;=\; \max_{i \in [p]} \; \max_{x \in X_i} f(x).
$

Now consider a single-thread interaction with budget $N$ that queries the union of points asked by the parallel threads,
\[
\mathcal{P} \;=\; \bigcup_{i \in [p]} X_i,
\qquad \text{so that } |\mathcal{P}| \leq N.
\]
This single-thread episode achieves reward at least $\max_{x \in \mathcal{P}} f(x)$, and
$
\max_{x \in \mathcal{P}} f(x) \;\geq\; f_{\text{parallel}}^* .
$
Thus, selecting the best outcome among $p$ independent threads with total budget $N$ is upper-bounded by what a single-thread interaction with budget $N$ could achieve.
This finishes the proof.
See Appendix~\ref{app:proofs} for the same proof on other tasks.

\begin{table*}[t]
\caption{Comparing parallel execution with various values of $p$ under the same total budget ($N=48$) across all tasks.
Results are reported on a single instance of each task.
\textbf{Although we theoretically show that parallelism should not yield gains over single-thread execution $(p=1)$, we observe improved performance across all models and tasks.}
We use \updelta{x} to denote the relative performance of parallel execution compared to single-thread execution when positive, and \downdelta{x} when negative.
}
\label{tab:all_tasks_n_48}
\begin{adjustbox}{width=\linewidth}
\begin{tabular}{l|rrrr||rrrr||rrrr}
\toprule
 & \multicolumn{4}{c||}{\tsearch} & \multicolumn{4}{c||}{\ttree} & \multicolumn{4}{c}{\tmaxsat} \\
Model & $p=1$ & $p=2$ & $p=3$ & $p=4$ & $p=1$ & $p=2$ & $p=3$ & $p=4$ & $p=1$ & $p=2$ & $p=3$ & $p=4$ \\
\midrule
\midrule
{\small\texttt{Qwen2.5-7B-Instruct}} & 0.33 & 0.52 \updelta{59} & \textbf{0.74 \updelta{125}} & 0.66 \updelta{102} & 0.71 & 0.83 \updelta{16} & 0.93 \updelta{31} & \textbf{0.96 \updelta{35}} & 0.45 & 0.53 \updelta{18} & 0.58 \updelta{28} & \textbf{0.63 \updelta{39}} \\
{\small\texttt{Qwen3-4B-Instruct}} & 0.27 & \textbf{0.52 \updelta{95}} & 0.45 \updelta{70} & 0.30 \updelta{12} & 0.91 & 0.85 \downdelta{7} & 0.84 \downdelta{7} & \textbf{0.91 \updelta{0}} & 0.39 & 0.58 \updelta{48} & \textbf{0.66 \updelta{69}} & 0.57 \updelta{47} \\
{\small\texttt{Qwen3-8B}} & 0.59 & 0.71 \updelta{21} & 0.57 \downdelta{3} & \textbf{0.72 \updelta{23}} & 0.68 & 0.76 \updelta{13} & 0.88 \updelta{30} & \textbf{0.93 \updelta{36}} & 0.48 & 0.70 \updelta{45} & 0.69 \updelta{45} & \textbf{0.72 \updelta{50}} \\
{\small\texttt{gemini-2.5-flash-lite}} & 0.41 & \textbf{0.58 \updelta{40}} & 0.57 \updelta{38} & 0.55 \updelta{34} & 0.63 & 0.75 \updelta{19} & 0.80 \updelta{26} & \textbf{0.85 \updelta{35}} & 0.62 & \textbf{0.71 \updelta{14}} & 0.70 \updelta{13} & 0.68 \updelta{9} \\
{\small\texttt{gpt-5-nano}} & 0.31 & 0.47 \updelta{51} & 0.55 \updelta{77} & \textbf{0.74 \updelta{138}} & 0.69 & 0.83 \updelta{20} & 0.91 \updelta{31} & \textbf{0.96 \updelta{39}} & 0.62 & 0.66 \updelta{6} & 0.60 \downdelta{3} & \textbf{0.72 \updelta{16}} \\
{\small\texttt{gpt-5-mini}} & 0.77 & 0.82 \updelta{7} & 0.87 \updelta{13} & \textbf{0.91 \updelta{18}} & 0.64 & 0.84 \updelta{31} & 0.90 \updelta{40} & \textbf{0.91 \updelta{41}} & 0.67 & 0.66 \downdelta{1} & \textbf{0.82 \updelta{23}} & 0.67 \updelta{1} \\
{\small\texttt{gpt-5}} & 0.72 & \textbf{0.83 \updelta{15}} & 0.77 \updelta{7} & 0.75 \updelta{4} & 0.63 & 0.74 \updelta{18} & 0.78 \updelta{25} & \textbf{0.94 \updelta{50}} & 0.78 & \textbf{0.82 \updelta{5}} & 0.74 \downdelta{6} & 0.76 \downdelta{3} \\
\midrule
{\small\texttt{Qwen3-8B-medium}} & 0.29 & 0.35 \updelta{20} & \textbf{0.38 \updelta{32}} & 0.30 \updelta{3} & 0.70 & 0.85 \updelta{22} & 0.94 \updelta{35} & \textbf{0.94 \updelta{35}} & 0.54 & \textbf{0.70 \updelta{30}} & 0.69 \updelta{30} & 0.64 \updelta{20} \\
{\small\texttt{gpt-5-nano-medium}} & 0.48 & 0.75 \updelta{56} & 0.88 \updelta{85} & \textbf{0.96 \updelta{101}} & 0.68 & 0.80 \updelta{18} & 0.89 \updelta{31} & \textbf{0.95 \updelta{40}} & 0.54 & 0.59 \updelta{9} & 0.61 \updelta{15} & \textbf{0.65 \updelta{22}} \\
{\small\texttt{gpt-5-mini-medium}} & 0.58 & 0.84 \updelta{45} & \textbf{0.86 \updelta{49}} & 0.79 \updelta{38} & 0.54 & 0.58 \updelta{8} & \textbf{0.67 \updelta{24}} & 0.66 \updelta{22} & 0.66 & 0.65 \downdelta{2} & \textbf{0.73 \updelta{9}} & 0.71 \updelta{8} \\
\midrule
{\small\texttt{explore-exploit}} & \multicolumn{4}{c||}{0.97} & \multicolumn{4}{c||}{0.96} & \multicolumn{4}{c}{0.84} \\
\bottomrule
\end{tabular}
\end{adjustbox}
\end{table*}

\paragraph{Evaluating parallel threads}
We consider parallelization with $p \in \{2,3,4\}$ threads on the same task instances evaluated in \S\ref{subsec:eval}, and report results and interpretations across our task suite. As shown in Table~\ref{tab:all_tasks_n_48} ($N=48$),
across all models, parallelization improves over single-thread interaction ($p=1$) on all tasks:
for \tsearch{}, multiple threads increase the chance that at least one samples near the global maximum and uses the remaining budget to refine;
for \ttree{}, parallel runs make it more likely that a thread enters a good gateway and commits to exploring it, reaching high-reward nodes;
and for \tmaxsat{}, multiple independent initial assignments increase the chance that one thread satisfies the gold clause by chance, after which local refinement yields higher reward.
We refer to Appendix~\ref{app:experiment_full:diff-budgets} for evaluation on the same problem instances with $N=36$, provide complementary experiments on additional instances of each task in Appendix~\ref{app:experiment_full:add-instances}, and discuss standard errors in Appendix~\ref{app:subsec:standard-errors}.

\paragraph{Theoretical analysis}
To understand why parallelization is beneficial, we analyze \textit{success probability} rather than reward.
Let $q(x)$ denote the probability of success with budget $x\in[0,1]$ (treated as continuous).
Success can be defined as achieving reward above a threshold (e.g., a fraction of maximum possible reward).
Parallelizing budget $x$ into $p>1$ threads is beneficial if
\begin{equation}\label{eq:p-success-condition}
1-\bigl(1-q(x/p)\bigr)^p > q(x).
\end{equation}

We model success with a sublinear (concave) power law,
\begin{equation}
q(x)=c x^\alpha,\qquad 0<c\le 1,\ \ 0<\alpha<1.
\end{equation}
Empirically, we often observe $q(x)\ll 1$, indicating a low-success regime (e.g., small effective $x$ and/or small $c$).

\newcommand{\thmcxalpha}{
Let $q(x)=c x^\alpha$ with $0<c\le 1$ and $0<\alpha<1$. For any integer $p>1$, there exists $v_p\in(0,1]$ such that \eqref{eq:p-success-condition} holds for all $x<v_p$ and fails for all $x\ge v_p$.
}
\begin{theorem}[Parallelization under a sublinear power law]\label{thm:cxalpha}
\thmcxalpha
\end{theorem}

Theorem~\ref{thm:cxalpha} shows that parallelization helps in the low-success regime $x<v_p$, matching our empirical results where $q(x)$ is small.
We refer to Appendix~\ref{app:proofs} for the proof and further discussion for $p=2$.

\begin{figure*}[t]
  \centering
  \begin{subfigure}[t]{0.32\textwidth}
    \centering
    \includegraphics[width=\linewidth]{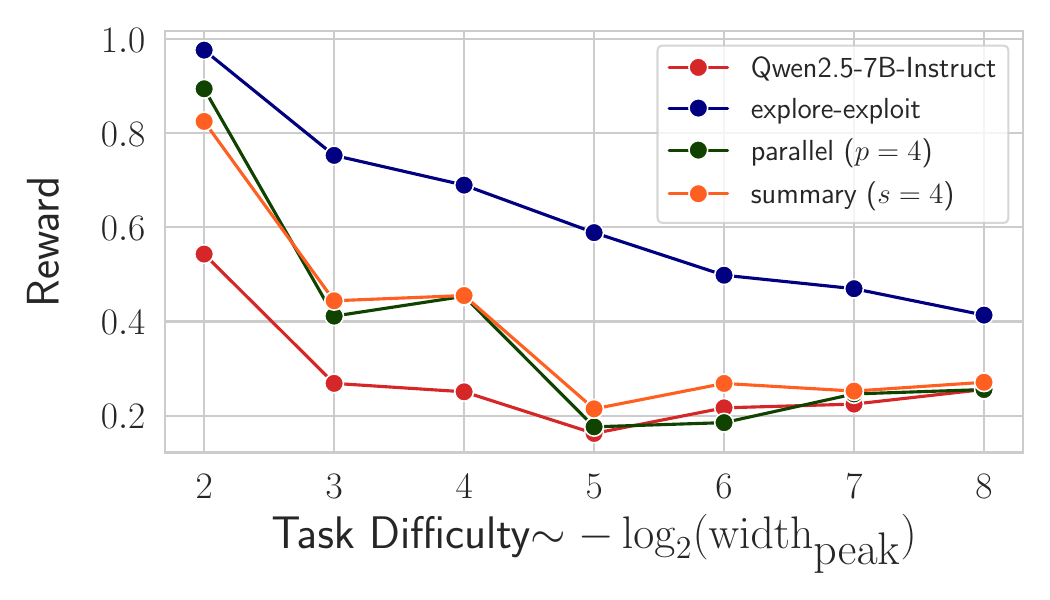}
    \caption{\tsearch}
  \end{subfigure}\hfill
  \begin{subfigure}[t]{0.32\textwidth}
    \centering
    \includegraphics[width=\linewidth]{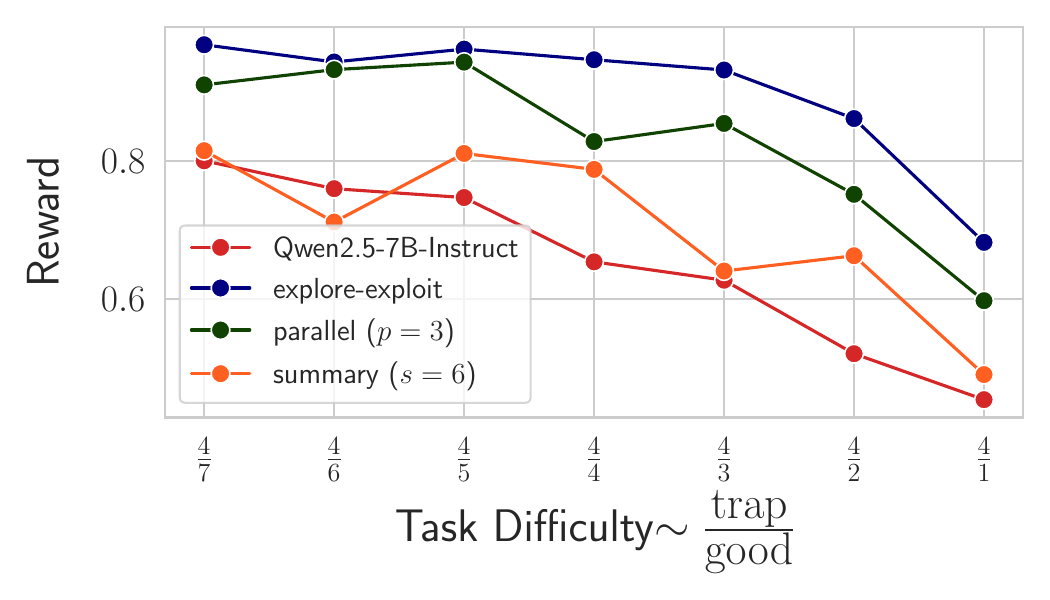}
    \caption{\ttree} 
  \end{subfigure}\hfill
  \begin{subfigure}[t]{0.32\textwidth}
    \centering
    \includegraphics[width=\linewidth]{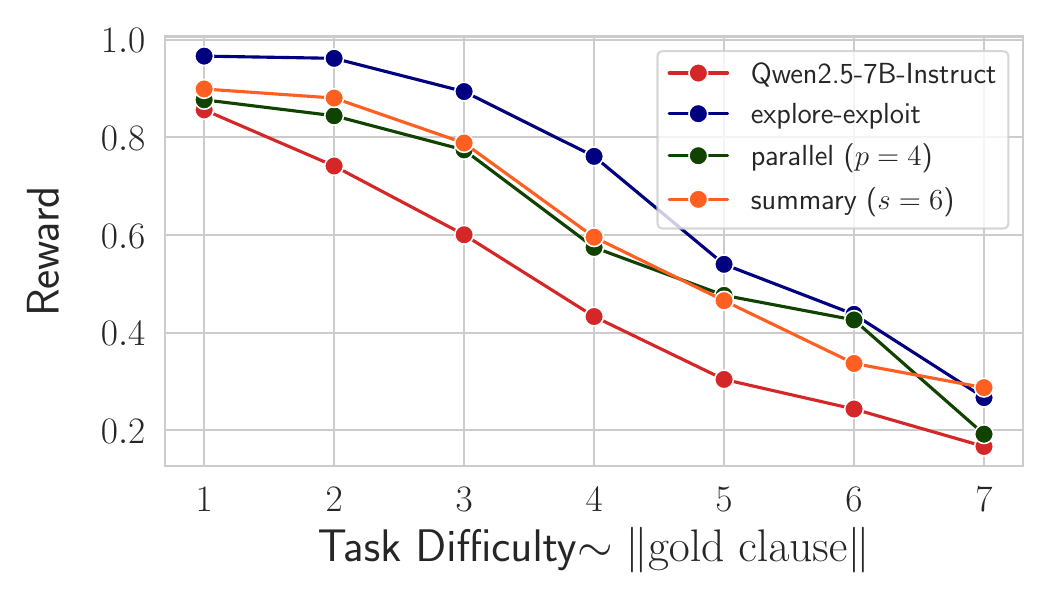}
    \caption{\tmaxsat}
  \end{subfigure}\hfill
  
  \caption{
Task difficulty variations for \tsearch{}, \ttree{}, and \tmaxsat{}.
We generate \textbf{multiple task instances by varying key difficulty-controlling parameters}---the peak width for \tsearch{}, the ratio of good to trap gateways for \ttree{}, and the size of the gold clause for \tmaxsat{}---with harder settings yielding lower baseline performance.
Across all tasks and difficulty levels, \textbf{the parallel and the summary methods consistently improves} over single-thread execution of \texttt{Qwen2.5-7B-Instruct}, 
narrowing the gap to explore--exploit baselines.
$N=36$ is the budget for these episodes.
}
\label{fig:task_difficulty}
\vspace{-8pt}
\end{figure*}


\subsection{Intervention: Summary}
\label{sec:summary}
We explore another single-thread approach (in contrast to the parallel method) to improve over the standard execution.
We hypothesize that the long context accumulated during interaction may mislead the model into committing to low-quality solutions and under-exploring the solution space.
To test this, we introduce a \textit{summary} method: over an episode with budget $N$, we provide the model with $s$ summaries.
After every $N/s$ interactions, we remove the earlier interaction content and continue the episode by conditioning the model on summary rather than the full history.

More specifically, the summary is generated from the interaction history (the model's past queries and observed feedback) and the task description.
It includes reflective questions (e.g., whether it has explored all parts of the search space)
and highlights key takeaways from the interactions so far (e.g., the current best solutions).
The summary guides future interactions by tracking progress and blind spots, but adds no new information—it only organizes what was revealed.
Next, we describe how we generate task-specific summaries and evaluate \texttt{Qwen-2.5-7B-Instruct} equipped with summaries.

\begin{table}[t]
\caption{
Summary method results ($N=48$) .
We report reward when using summary method with parameter $s$.
\texttt{Qwen2.5-7B-Instruct} corresponds to {no summarization} (standard interaction).
Parentheses report relative change compared to standard interaction.
We use \updelta{x} to denote the relative improvement.
\textbf{Summarization helps over no summary.}
}
\label{tab:all_tasks_n_48_summary}
\begin{adjustbox}{width=\linewidth}
\begin{tabular}{rc|c||c||c}
\toprule
\multicolumn{2}{c|}{Method} & \tsearch & \ttree & \tmaxsat \\
\midrule
\multicolumn{2}{c|}{\texttt{Qwen2.5-7B-Instruct}} & 0.33 & 0.71 & 0.45 \\
\midrule
\multirow{4}{*}{summary} & $s=2$ & 0.45 \updelta{36} & 0.78 \updelta{9} & 0.55 \updelta{23} \\
 & $s=3$ & 0.43 \updelta{29} & 0.82 \updelta{16} & 0.57 \updelta{26} \\
 & $s=4$ & 0.52 \updelta{57} & \textbf{0.82 \updelta{16}} & \textbf{0.66 \updelta{46}} \\
 & $s=6$ & \textbf{0.62 \updelta{87}} & 0.80 \updelta{13} & 0.60 \updelta{34} \\
\midrule
\multicolumn{2}{c|}{\texttt{explore-exploit baseline}} & 0.97 & 0.96 & 0.84 \\
\bottomrule
\end{tabular}
\end{adjustbox}
\vspace{-5pt}
\end{table}

\paragraph{Summarizing interaction history}
We summarize episodes via a short ``mission hand-off’’ that replaces the full interaction transcript every $N/s$ steps.
For \tsearch{}, the hand-off lists all queried points $(x,f(x))$ sorted by $x$, reports remaining budget, explicitly highlights unexplored intervals (gaps) in the domain,
and prompts a brief reflection on whether the trajectory is stuck near a local maximum.
For \ttree{}, it summarizes the explored connected subgraph by including the full query history (chronological), the current known node with the highest reward,
and a frontier of actionable next nodes (unknown nodes adjacent to explored ones); we group frontier nodes by height (distance from the root). 
For \tmaxsat{}, we summarize progress across assignment queries.
We provide an ordered history of past (score, assignment) pairs and explicitly restate the highest-scoring query.
Additionally, we append a coverage summary that quantifies how often each variable is set to 0 or 1, highlighting those with the most imbalanced distributions.
We refer to Appendix~\ref{app:summary} for full details.




\paragraph{Providing summaries boosts performance}
Table~\ref{tab:all_tasks_n_48_summary} reports results of the summary method for $s \in \{2,3,4,6\}$,
on the same instances evaluated in \S\ref{subsec:eval}.
As seen there, summaries can improve performance compared to standard protocol by encouraging broader exploration and reducing premature commitment to suboptimal solutions.
We present standard errors for these results in Appendix~\ref{app:subsec:standard-errors}.

\subsection{Robustness to Task Difficulty Variations}
\label{subsec:difficulty}
We consider multiple task variants by generating instances from our controllable task suite, as discussed in \S\ref{section:tasks}.
For all tasks, we construct a range of difficulty levels, calibrated by baseline performance on the corresponding instances (harder settings yield lower baseline rewards).
Difficulty is varied by tuning task-specific influential parameters: for \tsearch{}, we adjust the peak width; for \ttree{}, we vary the ratio of good branches to trap branches; and for \tmaxsat{}, we increase the number of variables in the gold clause.

As seen in Figure~\ref{fig:task_difficulty}, the parallel and summary methods consistently outperform standard execution across all tasks and difficulty levels,
highlighting their robustness and effectiveness across a wide range of environments.
See Appendix~\ref{app:task_scale_difficulty} for full results and detailed task configurations.

\subsection{Experimental Setup}
\label{subsec:experimental_setup}

In this section, we describe task instances used in \S\ref{subsec:eval}, \S\ref{subsec:parllel}, and \S\ref{sec:summary}; the evaluated models and interaction protocol.

\paragraph{Task Instances}
We consider three instances, one per task.
First, an instance of \tsearch{} with 8 hills: one needle (peak = 20) and 7 decoys (peak $\le 5$).
Second, an instance of \ttree{} with $772$ nodes, containing $3$ trap gateways and $3$ good gateways.
Third, an instance of \tmaxsat{} with $15$ variables and $120$ hidden clauses
with $w_\text{gold}=80, k_{\text{gold}} = 4$, and $k_{\text{other}} = 2$.
See Appendix~\ref{app:subsec:instances} for full parameters used in their instantiation.

\paragraph{Models}
We evaluate \texttt{Qwen2.5-7B-Instruct} \cite{yang2024qwen2}, \texttt{Qwen3-4B-Instruct}, and \texttt{Qwen3-8B} \cite{yang2025qwen3} in non-reasoning mode, as well as \texttt{Qwen3-8B} in reasoning mode with the reasoning length capped at 2048 tokens (denoted by \texttt{Qwen3-8B-medium}).
We also evaluate \texttt{gemini-2.5-flash-lite}~\cite{comanici2025gemini}.
Finally, we evaluate \texttt{gpt-5-nano}, \texttt{gpt-5-mini}, and \texttt{gpt-5} (with reasoning effort set to minimal), and consider medium-reasoning variants for the small models, denoted by \texttt{gpt-5-nano-medium} and \texttt{gpt-5-mini-medium}.
Overall, we cover 4–8B open models and proprietary models, with and without reasoning.
See Appendix~\ref{app:generation_params} for generation parameters and Appendix~\ref{app:evaluation_setup_details:reasoning} for reasoning specifications.


\paragraph{Interacting with the environment}
To evaluate a model under an interaction budget $N$, we provide the task description and budget, and then run an $N$-round episode.
At each round, the model outputs a query and the oracle returns the corresponding feedback.
The full interaction history is kept in the model context throughout the episode.
We refer to Appendix~\ref{app:prompts_interaction} for prompt and protocol details, and Appendix~\ref{app:evaluation_setup_details:output} for output formatting and error handling.
For robustness, we evaluate each model–instance pair with at least 40 runs (see Appendix~\ref{app:evaluation_setup_details:n_runs}) and report the mean reward across runs.

\vspace{-5pt}
\section*{Conclusion}
In this work, we propose an evaluation framework for LMs in interactive domains, focusing on their ability to interact with an environment, explore the space, and obtain good solutions within a limited budget.
We instantiate this framework with simple, controllable tasks, and observe sub-optimal LM performance, which we attribute to early commitment to sub-optimal solutions.
We study two lightweight paradigms that improve performance in practice: parallel runs and periodic summarization to track progress and blind spots.
Overall, improving agentic discovery may require more than a longer context, and instead benefit from mechanisms that encourage sustained exploration and revision.

\section*{Acknowledgment}
This project was supported in part by a grant from an NSF CAREER AWARD 1942230, the ONR PECASE grant N00014-25-1-2378, ARO’s Early Career Program Award 310902-00001, Army Grant No. W911NF2120076, the NSF award CCF2212458, NSF Award No. 2229885 (NSF Institute for Trustworthy AI in Law and Society, TRAILS), a MURI grant 14262683, DARPA AIQ grant HR00112590066  and an award from meta 314593-00001.

\section*{Impact Statement}
This paper introduces a suite of three controllable tasks that define interactive environments for measuring how well language models explore—given partial prior knowledge—and find high-quality solutions under a fixed interaction budget. By revealing systematic under-exploration, and showing that lightweight interventions such as parallel execution and periodic summarization can improve performance, our results help developers diagnose and reduce premature commitment in interactive agentic systems. We anticipate positive impacts on safer and more reliable deployment of agentic models, including better search behavior and improved detection of failure modes. To the best of our knowledge, this work does not have any negative societal impact.

\bibliography{ref}
\bibliographystyle{icml2026}
\appendix
\newpage
\onecolumn

\section*{Appendix Outline}
This appendix provides additional details needed to reproduce and contextualize our evaluation.
We first include an extended review of related work (Appendix~\ref{app:related_work}), situating our interactive benchmark relative to prior work on LLM-based optimization, delegated search, and interactive evaluation.
We then describe how task instances are procedurally generated for each task (Appendix~\ref{app:subsec:task}) and report the exact evaluated instances used in the main text (Appendix~\ref{app:subsec:instances}).
We then document how tasks are explained to models and how interactions are structured (Appendix~\ref{app:prompts_interaction}).
We next provide evaluation setup details, including decoding parameters, reasoning budgets, structured output enforcement, error handling, and the number of runs (Appendix~\ref{app:evaluation_setup_details}).
We then present complementary results, including interpretability analyses, budget sweeps, additional instances, evaluation on 50 randomly generated instances, and difficulty variation studies with standard errors (Appendix~\ref{app:experiment_full}).
We then provide the prompts used to generate summaries (Appendix~\ref{app:summary}).
We next give baseline details, including pseudocode and hyperparameter sensitivity analyses (Appendix~\ref{app:baseline_details}).
We finally present theoretical results regarding parallelization (Appendix~\ref{app:proofs}).


\section{Extended Related Work}
\label{app:related_work}

In this appendix, we provide an extended review of related work most relevant to our setting and contributions.

\subsection{LLMs as Combinatorial Optimizers and Algorithmic Simulators}
Recent literature has shown the surprising efficacy of LLMs in functioning as autonomous combinatorial optimizers and algorithmic simulators. For example, \citet{yang2024largelanguagemodelsoptimizers} introduced OPRO, demonstrating that LLMs can iteratively generate and refine solutions for NP-hard problems like the Traveling Salesperson Problem (TSP). Similarly, HeurAgenix~\cite{yang2025heuragenixleveragingllmssolving} demonstrates that LLMs can function as hyper-heuristics, selecting and evolving heuristic strategies for complex scheduling tasks without requiring gradient updates. \citet{jiang2025largelanguagemodelsendtoend} further proposed a framework where LLMs operate as end-to-end solvers, mapping natural language constraints directly to feasible solutions. On the algorithmic simulation front, LLMs have long been shown to be able to simulate algorithms via algorithmic prompting~\cite{zhou2022teachingalgorithmicreasoningincontext}, and some recent work on reasoning, such as Graph-of-thought~\cite{Besta_2024}, illustrates that modeling LLM reasoning as graph traversals allows for dynamic backtracking and aggregation of thought paths, mimicking complex search algorithms to solve elaborate problems like sorting and set operations. Similarly, \citet{jin2024graphchainofthoughtaugmentinglarge} developed Graph Chain-of-Thought (Graph-CoT), which prompts models to interact with the graph and perform iterative, step-by-step graph traversal. In our work, we leverage such capabilities of LLMs and design interactive task suites to evaluate the explorative searching capability of LLMs.

\subsection{Delegated Search and Constrained Interaction}
A key motivation for our interactive task suite is \emph{delegation},
where a principal relies on an agent to explore an unknown solution space and propose candidates under limited feedback and interaction constraints.
Prior work has formalized delegated search and shown that simple acceptance mechanisms can obtain meaningful approximation guarantees relative to the principal’s optimal choice~\cite{kleinberg2018delegated}.
More recent work studies richer delegation settings, including multi-agent delegation~\cite{shin2023delegating} and repeated delegation with online adaptation and regret guarantees~\cite{hajiaghayi2024regret}.
These results highlight delegation as a natural application domain where strategic exploration under a constrained interaction protocol is essential.

\subsection{Task-Oriented vs. Exploratory Interactive LLM Benchmarks}

While a growing body of work evaluates LLMs in \emph{interactive} settings, most existing benchmarks are "task-oriented", that is they primarily measure task completion, code generation quality, or strategic competition. They rarely isolate an agent's ability to \emph{systematically explore} an unknown environment under a strict query budget.

One category of benchmarks utilizes real-world task execution as the interactive environment, notably web navigation and tool-use settings such as shopping or web-based workflows~\cite{chang2024agentboard, xi2025agentgym, anupam2025browserarenaevaluatingllmagents, lù2024weblinxrealworldwebsitenavigation, zhou2024webarenarealisticwebenvironment}. While highly realistic, performance in these settings is typically dominated by planning, instruction-following, and UI manipulation skills rather than pure exploratory capability.

A second category focuses on interactive coding, where the agent receives execution feedback, unit-test results, or evaluator responses~\cite{pan2025benchmarkstalkreevaluatingcode, chi2025copilot, mozannar2024realhumaneval, wang2025codeifbenchevaluatinginstructionfollowingcapabilities,NEURIPS2023_4b175d84}. In these environments, the interaction serves largely as a correctness oracle for a well-specified objective; the challenge lies in synthesis and debugging rather than navigating an unknown solution space.

A third category spans game-based evaluations and dialogue protocols, including two-player and host--player formats~\cite{schlangen2025third, yu2025guessarena, guertler2025textarena, duan2024gtbench, costarelli2024gamebench, chang2024agentboard, hoover2020many, long2025puzzleplex, liu2023agentbench, de2025infoquest}. These settings effectively probe strategic reasoning and social dynamics, but the signal for exploration quality is often confounded by opponent behavior and game-theoretic considerations.

Closest to our setup are single-player interactive environments and navigation tasks. However, many puzzles in this domain require little exploration beyond deterministic reasoning~\cite{gong2024mindagent, long2025puzzleplex}. Similarly, while maze and navigation benchmarks involve sequential decision-making under partial observability~\cite{abdulhai2023lmrl, chevalier2018babyai, bianchi2024well, wu2023smartplay, xi2025agentgym, einarsson2025mazeevalbenchmarktestingsequential}, the objective is typically goal-state navigation. The design of these benchmarks tends to emphasize memory, control, or policy learning rather than isolating explore--exploit tradeoffs.

Finally, while several benchmarks incorporate partial observability and interactive learning—such as multi-turn puzzles or black-box optimization with natural language feedback~\cite{badola2025multiturnpuzzlesevaluatinginteractive, cheng2023llfbenchbenchmarkinteractivelearning, wu2023smartplay}—they often lack a controlled, parameterized family of environments specifically constructed to disentangle exploration strategies.

In contrast, our task suite is explicitly designed to measure \emph{exploration quality} under a limited budget across diverse problem geometries. By utilizing simple, parameterized environments with controllable difficulty, we are able to isolate specific failures—such as premature commitment, under-exploration, and inefficient budget utilization—that are often obscured in more complex, entangled task environments.

\section{Task Instances Details}
\label{app:tasks_instances}

In this section, we first describe how we generate task instances, as detailed in Appendix~\ref{app:subsec:task}.
We then provide details on the specific instances used in the main-text tables in Appendix~\ref{app:subsec:instances}.

\subsection{Details on Creating Instances}
\label{app:subsec:task}

In this section, we will describe in greater detail how we generate instances for each of our tasks. 
\paragraph{\tsearch:}
\newcommand{\dec}{\text{decoy}}
\newcommand{\need}{\text{needle}}
We construct functions for measuring exploration by defining $f:[0,10]\to\mathbb{R}$ as a sum of Gaussian hills. Each hill is set by a center $c$, width $w>0$, and height $h>0$, and is given by
\begin{equation*}
g(x) = h \cdot \exp\left(-\frac{(x-c)^2}{w}\right).
\end{equation*}

Fix an integer level $k$ and let $\Delta_k = 10/2^k$. Define the coarse grid points $x_m = 10m/2^k$ for $m\in\{0,1,\dots,2^k\}$. For each interior point $m\in\{1,2,\dots,2^k-1\}$, we place one $\dec$ hill with center
\begin{equation*}
c_m = x_m + \varepsilon_m, \qquad \varepsilon_m \sim \mathrm{Uniform}\left[-j_\dec \Delta_k,\; j_\dec \Delta_k\right].
\end{equation*}
All $\dec$ hills have width $w_m=\alpha_\dec \Delta_k$, and their heights are sampled uniformly from ${1,2,3,4,5}$.

We then place one $\need$ hill on a finer grid with level $k'\ge k$ and spacing $\Delta_{k'} = 10/2^{k'}$. We choose an odd index $m\in\{1,3,\dots,2^{k'}-1\}$ so that the base point $x^\star = 10m/2^{k'}$ is not on the coarse grid. The $\need$ center is
\begin{equation*}
c_\need = x^\star + \varepsilon_\need, \qquad \varepsilon_\need \sim \mathrm{Uniform}\left[-j_\need \Delta_{k'}, j_\need \Delta_{k'}\right].
\end{equation*}
The $\need$ has fixed height $h_\need=20$ and width $w_\need=\alpha_\need \Delta_{k'}$.

In total, we have $2^k$ hills. Let $(c_i,w_i,h_i)$ denote the center, width, and height of hill $i$. We define
\begin{equation*}
f(x) = \sum_{i=1}^{2^k} h_i \exp\left(-\frac{(x-c_i)^2}{w_i}\right), \qquad x\in[0,10].
\end{equation*}

Each instance is generated by the parameters $k$, $k'$, $j_\dec$, $j_\need$, $\alpha_\dec$, and $\alpha_\need$.

\paragraph{\ttree:}
\newcommand{\trap}{\text{trap}}
\newcommand{\good}{\text{good}}
As described earlier, we construct rooted trees where the root has $r_\trap$ children, called \emph{trap gateways}, and $r_\good$ children, called \emph{good gateways}. From each gateway node, we attach $b$ (called \texttt{fanout} in the main text) disjoint simple paths (independent chains). Each trap gateway has $b$ chains of length $d_\trap - 1$, and each good gateway has $b$ chains of length $d_\good - 1$. After constructing the tree, we randomly permute node IDs to remove any link between identifiers and structural roles.

We assign values by setting the root value to $0$ and adding edge increments along each root-to-node path.
\begin{itemize}
\item \textbf{Trap gateways.} Entering a trap gateway adds a one-time increment of $2$. For the next $6$ steps along any of its $b$ chains, values increase by $1$ per step, creating short-term positive momentum. After that, increases are sparse: most steps add $0$, with a $+1$ added once every $4$ steps.
\item \textbf{Good gateways.} Entering a good gateway adds a one-time increment of $1$. After that, each additional step along its chains adds $4$, giving steady growth.
\end{itemize}

Each instance is defined by the parameters $r_\trap$, $r_\good$, $b$, $d_\trap$, and $d_\good$.

\paragraph{\tmaxsat:}
\newcommand{\gold}{\text{gold}}
\newcommand{\other}{\text{other}}
\newcommand{\x}{\mathbf{x}}
To generate a formula with $n$ variables and $m$ clauses, we first sample an assignment $\x^*_1,\dots,\x^*_n$, where each $\x^*_i$ is drawn uniformly from \texttt{true} and \texttt{false}. We then form a gold clause as a conjunction over $k_\gold$ variables and include this same clause $w_\gold$ times. For the remaining $m-w_\gold$ clauses, we exclude the $k_\gold$ variables used in the gold clause. Each remaining clause is constructed by conjoining $k_\other$ variables selected uniformly at random. We then negate literals as needed so that the assignment $\x^*_1,\dots,\x^*_n$ satisfies every clause.

Each instance is generated by the parameters $n$, $m$, $k_\gold$, $k_\other$, and $w_\gold$.

\subsection{Details on Evaluated Instances}
\label{app:subsec:instances}

\paragraph{\tsearch}
The instance used in \S\ref{subsec:eval} has decoy hills centered at
$[1.33, 2.77, 4.01, 5.31, 6.45, 7.82, 8.95]$,
with widths
$[0.1, 0.2, 0.1, 0.1, 0.1, 0.2, 0.1]$,
and heights
$[1, 5, 2, 1, 2, 3, 4]$.
The needle hill is centered at $1.3$ with height $20$ and width $0.01$.
Figure~\ref{fig:teaser} (left) shows this function.

\paragraph{\ttree}
The instance used in \S\ref{subsec:eval} is generated as described in Appendix~\ref{app:subsec:task} with parameters $r_\trap = r_\good = 3$, $b = 5$, $d_\trap = 40$, and $d_\good = 12$. This produces a tree with 772 nodes and a maximum value of 45.

\paragraph{\tmaxsat}
The instance used in \S\ref{subsec:eval} is generated as described in Appendix~\ref{app:subsec:task} with parameters $n=15$, $m=120$, $k_\gold=4$, $k_\other=2$, and $w_\gold=80$.

\section{Explaining Tasks to Models}
\label{app:prompts_interaction}
Each conversation begins with a system prompt that defines the task and may include some instance details. A user prompt then provides all remaining instance details and asks the model for its first query. The conversation templates are shown in Figures~\ref{fig:prompt_dataset_hill}, \ref{fig:prompt_dataset_tree}, and \ref{fig:prompt_dataset_sat}.

For the \ttree{} task, we require that each queried node be connected to a previously queried node. We find that models often struggle to identify all nodes that satisfy this rule. Therefore, after answering each model query, we also return a randomly shuffled list of nodes that have just become available to query. If no new node becomes available, we instead return the full set of available nodes. In this case, we sort nodes by when they became available (most recent first) and break ties at random.

\begin{figure}[ht]
\begin{tcolorbox}[
  enhanced, 
  colframe=green!75!black, 
  colback=gray!10, 
  coltitle=white, 
  colbacktitle=green!75!black, 
  width=\linewidth, 
  arc=2mm, 
  auto outer arc, 
  boxrule=0.5pt, 
  left=10pt, 
  right=10pt, 
  drop shadow={black!50!white},
  top=10pt, 
  bottom=10pt, 
  title=\textbf{Prompt Template for \tsearch{}}, 
  fontupper=\small,      
  fonttitle=\footnotesize\bfseries, 
  title code={\node[rounded corners, fill=blue!75!black, draw=none, text=white] at (frame.title) {\textbf{xxx}};}, 
  attach boxed title to top center={yshift=-2mm}, 
  boxed title style={sharp corners, size=small}, 
]
  \textbf{System Prompt}

You are playing a game to find the global maximum value of a hidden continuous function f(x). Note that the function may have many local maximums which are not globally maximal. You have a limited number of queries.

The function is defined on the interval [0, 10].

Rules:

- In each turn, you can query the value of the function `f(x)` for any real number `x` in the interval `[0, 10]`.

- The number of remaining queries will be provided in each turn.

- Your goal is to find the global maximum value of `f(x)`.

How to play:

- Use the tool with a JSON payload like this: {"reason": "your strategy for querying x", "query": {"x": v}}, where 'v' is any real number between 0 and 10 that you want to query. The tool call MUST ONLY contain this JSON. No other text is allowed.

- After each query, you will be told the value of `f(x)` for the `x` you queried.

\textbf{User}

You have \textcolor{blue}{$q$} queries. Your goal is to find the maximum value.

\textbf{Assistant}

\{"reason": "...", "query": {"x": \textcolor{blue}{$x$}}\}

\textbf{User}

Value of f(\textcolor{blue}{$x$}) is \textcolor{blue}{$y$}. Remaining queries: \textcolor{blue}{\{$q-1$\}}.

\end{tcolorbox}
\caption{The prompt structure used for the \tsearch{} task. Blue text indicates dynamic variables specific to the problem instance or generated queries.}
\vspace{-5pt}
\label{fig:prompt_dataset_hill}
\end{figure}
\begin{figure}[ht]
\begin{tcolorbox}[
  enhanced, 
  colframe=green!75!black, 
  colback=gray!10, 
  coltitle=white, 
  colbacktitle=green!75!black, 
  width=\linewidth, 
  arc=2mm, 
  auto outer arc, 
  boxrule=0.5pt, 
  left=10pt, 
  right=10pt, 
  drop shadow={black!50!white},
  top=10pt, 
  bottom=10pt, 
  title=\textbf{Prompt Template for \ttree{}}, 
  fontupper=\small,      
  fonttitle=\footnotesize\bfseries, 
  title code={\node[rounded corners, fill=blue!75!black, draw=none, text=white] at (frame.title) {\textbf{xxx}};}, 
  attach boxed title to top center={yshift=-2mm}, 
  boxed title style={sharp corners, size=small}, 
]
  \textbf{System Prompt}

You are playing a graph game to find the maximum hidden node value. You have limited queries.

Rules:

- You are given the graph structure and the value of a starting node.

- The order of nodes in adjacency lists is randomized. The list of valid nodes to query is ordered such that neighbors of more recently queried nodes appear first, with ties broken randomly.

- Querying a node reveals its value.

- Do not query the same node twice.

- You can only query a node if it's a neighbor of an already known node.

- Values of connected nodes differ by at most 4.

How to play:

- Use the tool with a JSON payload like this: \{"reason": "your reasoning here", "query": \{"node": i\}\}, where 'i' is the index of the node you want to query. The tool call MUST ONLY contain this JSON. No other text is allowed.

- After each query, you will be told the value of the queried node.

\textbf{User}

n = \textcolor{blue}{$n$}

Graph (by adjacency list): The order of nodes in the following adjacency list and the order of neighbors for each node are both randomized.

\textcolor{blue}{$v_1$}: \textcolor{blue}{$v_2$}, \textcolor{blue}{$v_3$}, \textcolor{blue}{$\cdots$}

\textcolor{blue}{$v_4$}: \textcolor{blue}{$v_5$}, \textcolor{blue}{$\cdots$}

\textcolor{blue}{$\cdot$}

\textcolor{blue}{$\cdot$}

\textcolor{blue}{$\cdot$}

Start node: \textcolor{blue}{$v_6$}. Value: 0. You have \textcolor{blue}{$q$} queries remaining. Your goal is to find the maximum value.

Valid nodes to query next: \textcolor{blue}{$v_7$ $v_8$ $v_9$ $\cdots$}.

\textbf{Assistant}

\{"reason": "...", "query": \{"node": \textcolor{blue}{$v$}\}\}

\textbf{User}

 VALUE of \textcolor{blue}{$v$} is \textcolor{blue}{$x$}. Remaining queries: \textcolor{blue}{\{$q-1$\}}. This query has made new neighbors available to query: \textcolor{blue}{$v_{10}$ $v_{11}$ $v_{12}$ $\cdots$}. You can still choose to query from previously available neighbors. 

\end{tcolorbox}
\caption{The prompt structure used for the \ttree{} task. Blue text indicates dynamic variables specific to the problem instance or generated queries.}
\vspace{-5pt}
\label{fig:prompt_dataset_tree}
\end{figure}

\begin{figure}[ht]
\begin{tcolorbox}[
  enhanced, 
  colframe=green!75!black, 
  colback=gray!10, 
  coltitle=white, 
  colbacktitle=green!75!black, 
  width=\linewidth, 
  arc=2mm, 
  auto outer arc, 
  boxrule=0.5pt, 
  left=10pt, 
  right=10pt, 
  drop shadow={black!50!white},
  top=10pt, 
  bottom=10pt, 
  title=\textbf{Prompt Template for \tmaxsat{}}, 
  fontupper=\small,      
  fonttitle=\footnotesize\bfseries, 
  title code={\node[rounded corners, fill=blue!75!black, draw=none, text=white] at (frame.title) {\textbf{xxx}};}, 
  attach boxed title to top center={yshift=-2mm}, 
  boxed title style={sharp corners, size=small}, 
]
  \textbf{System Prompt}

  You are playing a game to find a Boolean assignment that maximizes the weighted sum of satisfied constraints of a hidden boolean formula.

  IMPORTANT: The formula is NOT in Conjunctive Normal Form (CNF).

  Rules:
  
  - There are \textcolor{blue}{$b$} Boolean variables x0..x\textcolor{blue}{\{$b-1$\}}.
  
  - There are \textcolor{blue}{$c$} weighted clauses.
  
  - Each clause is formed by 1 to \textcolor{blue}{$l$} literals (variables or their negations) combined with logical AND.
  
  - A clause is satisfied if and only if ALL its literals are True (unlike standard CNF which uses OR).
  
  - In each turn, you may query ANY full assignment (a list of 0/1 values of length \textcolor{blue}{$b$}).
  
  - You will receive the sum of weights of satisfied clauses for that assignment.
  
  - Your goal is to find the maximum weighted sum within the query budget.

  How to play:
  
  - Use the tool with a JSON payload like this: {"reason":"your strategy","query":{"x0":0, "x1":1, "x2":0, ...}}.
  
  - The tool call MUST ONLY contain this JSON. No other text is allowed.

\textbf{User}

You have \textcolor{blue}{$q$} queries. Maximize the number of satisfied clauses (m=\textcolor{blue}{$c$}).

\textbf{Assistant}

\{"reason":"...","query":{"x0":\textcolor{blue}{$v_1$},"x1":\textcolor{blue}{$v_2$},...,"x\textcolor{blue}{\{$b-1$\}}":\textcolor{blue}{$v_b$}}\}

\textbf{User}

Satisfied clauses: \textcolor{blue}{$c_1$}. Remaining queries: \textcolor{blue}{\{$q-1$\}}.

\end{tcolorbox}
\caption{The prompt structure used for the \tmaxsat{} task. Blue text indicates dynamic variables specific to the problem instance or generated queries.}
\vspace{-5pt}
\label{fig:prompt_dataset_sat}
\end{figure}

\clearpage

\section{Evaluation Setup Details}
\label{app:evaluation_setup_details}

In this section, we provide additional evaluation details used throughout the paper.
In Appendix~\ref{app:generation_params}, we describe the decoding parameters used for token generation.
In Appendix~\ref{app:evaluation_setup_details:reasoning}, we provide details on how we evaluate reasoning-capable models.
In Appendix~\ref{app:evaluation_setup_details:output}, we specify the required model output format and our error-handling protocol.
In Appendix~\ref{app:evaluation_setup_details:n_runs}, we report the number of runs per task instance and model used to obtain robust, statistically meaningful results.

\subsection{Generation parameters}
\label{app:generation_params}
Unless otherwise specified, we decode queries with temperature $T=0.7$ and nucleus sampling top-$p=0.95$ for all models that expose these controls (both local inference and API models). For models/APIs that do not support a parameter (e.g., \texttt{gpt-5} family's temperature), we use the provider default and keep the remaining decoding parameters identical when possible. All other decoding settings (e.g., max output tokens, top-$k$, repetition penalties) are left at their framework/provider defaults.

\subsection{Thinking Budget for Reasoning} \label{app:evaluation_setup_details:reasoning}
With the exception of \texttt{Qwen2.5-7B-Instruct} and \texttt{Qwen3-4B-Instruct}, all evaluated models support reasoning capabilities. For the \texttt{gpt-5} family, we set the \texttt{reasoning\_effort} parameter to \texttt{minimal} to prevent them from reasoning. We also evaluate \texttt{gpt-5-nano} and \texttt{gpt-5-mini} with \texttt{reasoning\_effort} set to \texttt{medium}, denoting these variations as \texttt{gpt-5-nano-medium} and \texttt{gpt-5-mini-medium}, respectively. For \texttt{Qwen3-8B}, we examine one variation with reasoning disabled and another, denoted as \texttt{Qwen3-8B-medium}, with reasoning limited to 2048 tokens per message. Following the official Qwen documentation \cite{yang2025qwen3} for the latter, we terminate the model's reasoning generation after 2048 tokens and append ``\textit{Considering the limited time by the user, I have to give the solution based on the thinking directly now}'' to the reasoning output. Finally, for \texttt{gemini-2.5-flash-lite}, we utilize the default \texttt{reasoning\_effort} setting of \texttt{none}.

\subsection{Model Output}
\label{app:evaluation_setup_details:output}
We use structured outputs for all models. Supported by \texttt{vLLM}~\cite{kwon2023efficient} for local inference and by the OpenAI and Gemini APIs, this method enforces the JSON structure defined in our prompts (see Figure~\ref{fig:prompt_dataset_hill}). We include a ``reason'' or ``think'' field in the response schema. This enables models to execute a chain of thought before generating a query, even if their internal reasoning features are disabled or unavailable.

Despite format enforcement, models may still rarely make semantic errors. For example, in the \ttree{} task, a model might request a node lacking a previously queried neighbor. In such cases, we first retry the generation silently, allowing up to 20 retries per episode. If errors persist, we provide an informative error message for up to 5 additional attempts. We find silent retries more effective than feedback, as models often persist in the invalid query when explicitly shown their errors. Invalid queries do not count toward the budget; however, the episode ends if the error count exceeds these thresholds.

\subsection{Number of runs}
\label{app:evaluation_setup_details:n_runs}

To ensure statistical robustness and account for variance, we evaluate all models and baselines using multiple independent runs. Table~\ref{tab:absolute_p1_runs} details the specific sample sizes used for the main results reported in Table~\ref{tab:absolute_p1}.

For the complementary analyses, we adopted the following sample sizes: For the scaling analysis (Figure~\ref{fig:scaling_law_curves}), we performed 100 runs for \texttt{Qwen2.5-7B-Instruct} and 50 runs for \texttt{gpt-5-mini} for each instance-budget pair, while the \texttt{explore-exploit} baseline was evaluated over 300 runs for \tsearch, 200 for \ttree, and 100 for \tmaxsat. Similarly, for both the summary method evaluation (Table~\ref{tab:all_tasks_n_48_summary}) and the difficulty variation experiments (Figure~\ref{fig:task_difficulty}), we conducted 200 baseline runs and 50 \texttt{Qwen2.5-7B-Instruct} runs per setting.

Regarding the parallel method (\S\ref{subsec:parllel}), we simulated performance combinatorially using the available pool of $r$ independent runs. We evaluated all $\binom{r}{p}$ subsets of size $p$, defined the subset reward as the maximum reward achieved by any run within that subset, and reported the average reward in Table~\ref{tab:all_tasks_n_48}. This same combinatorial pooling strategy was applied to the difficulty variation experiments.

\begin{table*}[t]
\centering
\caption{
Number of times we run each model on each instance and budget.
}
\label{tab:absolute_p1_runs}
\begin{adjustbox}{width=0.8\linewidth}
\begin{tabular}{l|cc||ccc||cc}
\toprule
 & \multicolumn{2}{c||}{\tsearch} & \multicolumn{3}{c||}{\ttree} & \multicolumn{2}{c}{\tmaxsat} \\
Model & $N=36$ & $N=48$ & $N=36$ & $N=48$ & $N=60$ & $N=36$ & $N=48$ \\
\midrule
\midrule
{\texttt{Qwen2.5-7B-Instruct}} & 50 & 50 & 100 & 100 & 100 & 60 & 60 \\
{\texttt{Qwen3-4B-Instruct}} & 50 & 50 & 100 & 100 & 100 & 60 & 60 \\
{\texttt{Qwen3-8B}} & 50 & 50 & 100 & 100 & 100 & 60 & 60 \\
{\texttt{gemini-2.5-flash-lite}} & 50 & 50 & 100 & 100 & 100 & 60 & 60 \\
{\texttt{gpt-5-nano}} & 50 & 50 & 100 & 100 & 100 & 60 & 60 \\
{\texttt{gpt-5-mini}} & 50 & 50 & 100 & 100 & 100 & 60 & 60 \\
{\texttt{gpt-5}} & 50 & 50 & 50 & 50 & 50 & 40 & 40 \\ \midrule
{\texttt{Qwen3-8B-medium}} & 50 & 50 & 50 & 50 & 50 & 40 & 40 \\
{\texttt{gpt-5-nano-medium}} & 50 & 50 & 50 & 50 & 50 & 40 & 40 \\
{\texttt{gpt-5-mini-medium}} & 50 & 50 & 50 & 50 & 50 & 40 & 40 \\ \midrule
{\texttt{explore-exploit baseline} } & 500 & 500 & 200 & 200 & 200 & 200 & 200 \\
\bottomrule
\end{tabular}
\end{adjustbox}
\end{table*}

\section{Complementary Results}
\label{app:experiment_full}

In this section, we present additional experimental results that support and extend the findings discussed in the main text. We begin by analyzing the specific exploration trajectories of the models to provide further interpretability regarding their greedy behavior (Appendix~\ref{app:interpretability}). Subsequently, we evaluate performance across different interaction budgets (Appendix~\ref{app:experiment_full:diff-budgets}) and validate the robustness of our conclusions using a set of additional problem instances (Appendix~\ref{app:experiment_full:add-instances}). We also verify the consistency of our results by scaling our evaluation to a larger, diverse dataset consisting of 50 randomly generated instances for each task (Appendix~\ref{app:fifty_instance}). Finally, we present the comprehensive results of our robustness analysis across varying levels of task difficulty (Appendix~\ref{app:task_scale_difficulty}).

In this section, we present additional experimental results that support and extend the findings discussed in the main text. We begin by analyzing the specific exploration trajectories of the models to provide further interpretability regarding their greedy behavior (\S\ref{app:interpretability}). Subsequently, we evaluate performance across different interaction budgets (\S\ref{app:experiment_full:diff-budgets}) and validate the robustness of our conclusions using a set of additional problem instances (\S\ref{app:experiment_full:add-instances}). 

\subsection{Interpreting Models' Exploration Patterns}
\label{app:interpretability}
We analyze how the model allocates its budget in Figure~\ref{fig:interp_supp} for the instance described in \S\ref{subsec:eval}. These results complement Figure~\ref{fig:interp} and illustrate the greedy behavior of the model. With \tsearch{}, the model spends a small fraction of its budget on exploration before switching to consistent exploitation. With \ttree{}, the model persists along the branch selected during the initial queries.

\begin{figure*}[t]
  \centering
  \begin{subfigure}[t]{0.5\textwidth}
    \centering
    \includegraphics[width=\linewidth]{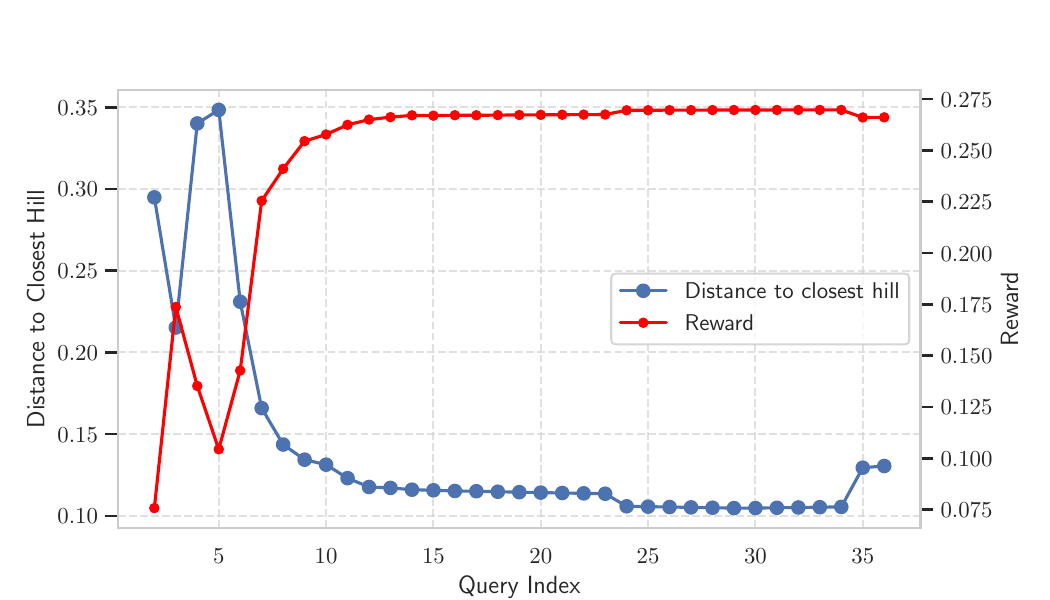}
    \caption{\tsearch}
  \end{subfigure}\hfill
  \begin{subfigure}[t]{0.5\textwidth}
    \centering
    \includegraphics[width=\linewidth]{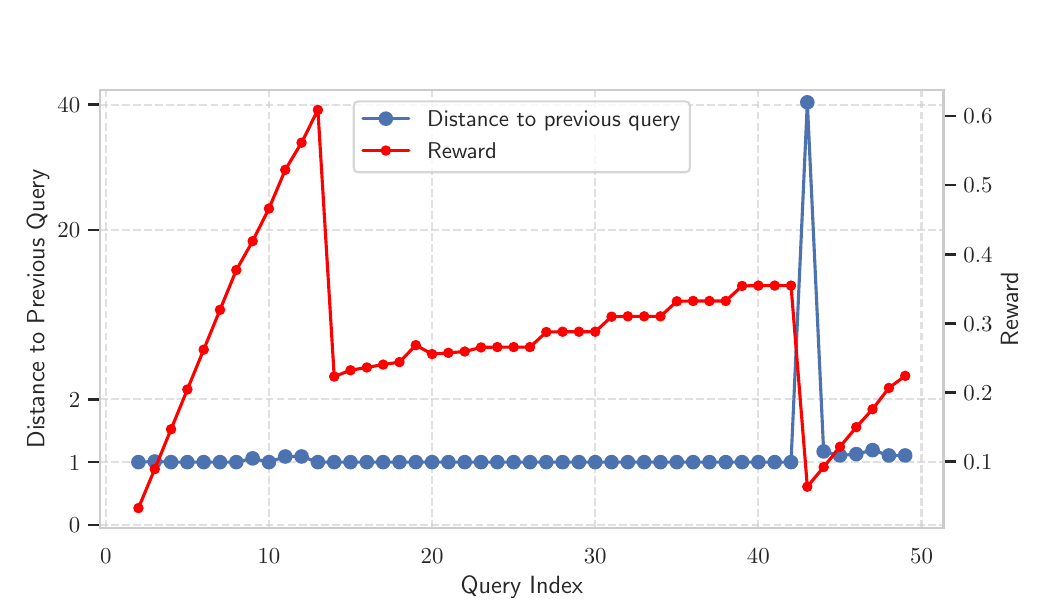}
    \caption{\ttree}
  \end{subfigure}\hfill
  
  \caption{
    Query patterns of \texttt{Qwen2.5-7B-Instruct} on \tsearch{} and \ttree{}. 
    \textbf{Left (\tsearch, 36 queries):} We report the average distance to the closest hill and the average reward at each step across episodes. The model starts exploiting in fewer than 10 steps and results in a suboptimal final reward, as shown in Table~\ref{tab:absolute_p1}. \textbf{Right (\ttree, 48 queries):} We report the average distance to the previously queried node and the average reward at each query. The distance remains almost constant at 1 until the model reaches the end of the trap gateway branches at height 42. The reward rises sharply during the first 13 queries (good gateway branches) and drops after these branches end, though it continues to increase slowly along the trap branch. Together with Figure~\ref{fig:interp}, this indicates that the model exploits early and reaches the end of any branch it enters.
  }
  \label{fig:interp_supp}
\end{figure*}

We also include more episodes similar to \Cref{fig:interp} in Figure~\ref{fig:interp:all:hill}, \ref{fig:interp:all:tree}, and \ref{fig:interp:all:cnf}.

\begin{figure*}[p] 
  \centering

  \begin{subfigure}[t]{0.33\textwidth}
    \centering
    \includegraphics[width=\linewidth]{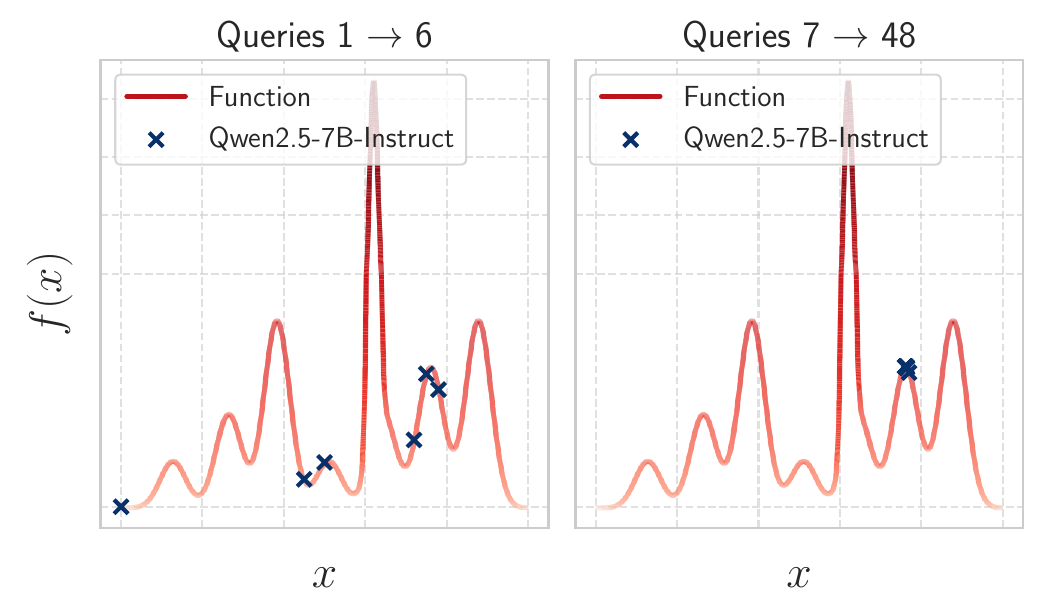}
  \end{subfigure}
  \hfill
  \begin{subfigure}[t]{0.33\textwidth}
    \centering
    \includegraphics[width=\linewidth]{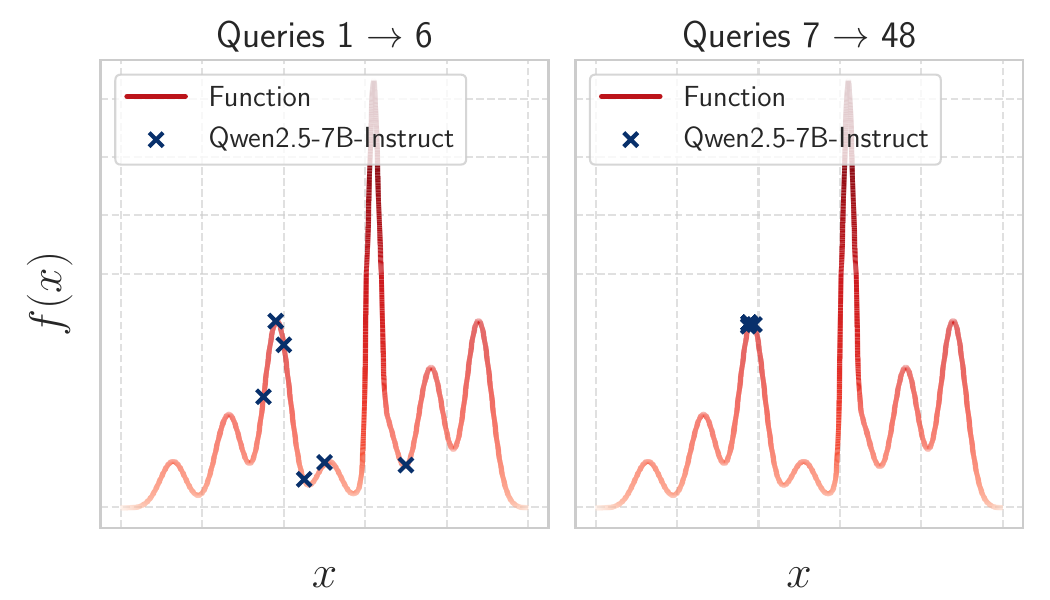}
  \end{subfigure}
  \hfill
  \begin{subfigure}[t]{0.33\textwidth}
    \centering
    \includegraphics[width=\linewidth]{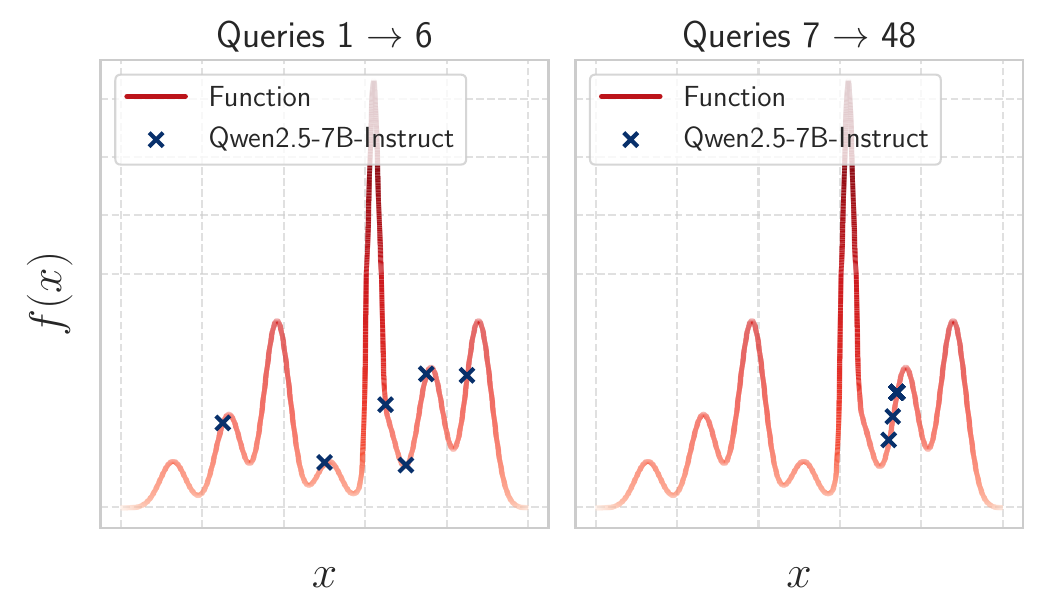}
  \end{subfigure}
  \hfill
  \begin{subfigure}[t]{0.33\textwidth}
    \centering
    \includegraphics[width=\linewidth]{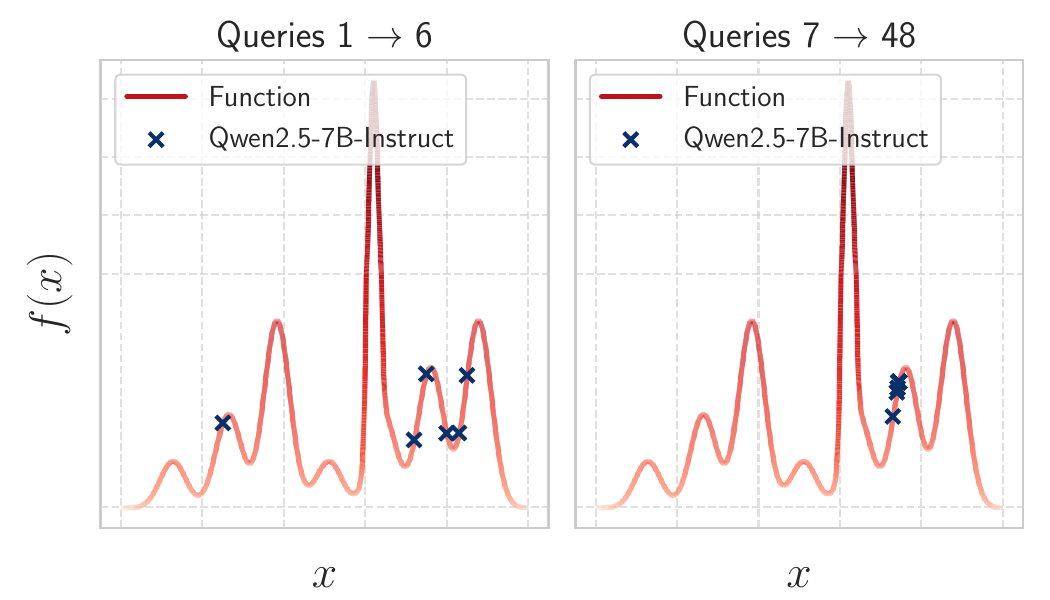}
  \end{subfigure}
  \hfill
  \begin{subfigure}[t]{0.33\textwidth}
    \centering
    \includegraphics[width=\linewidth]{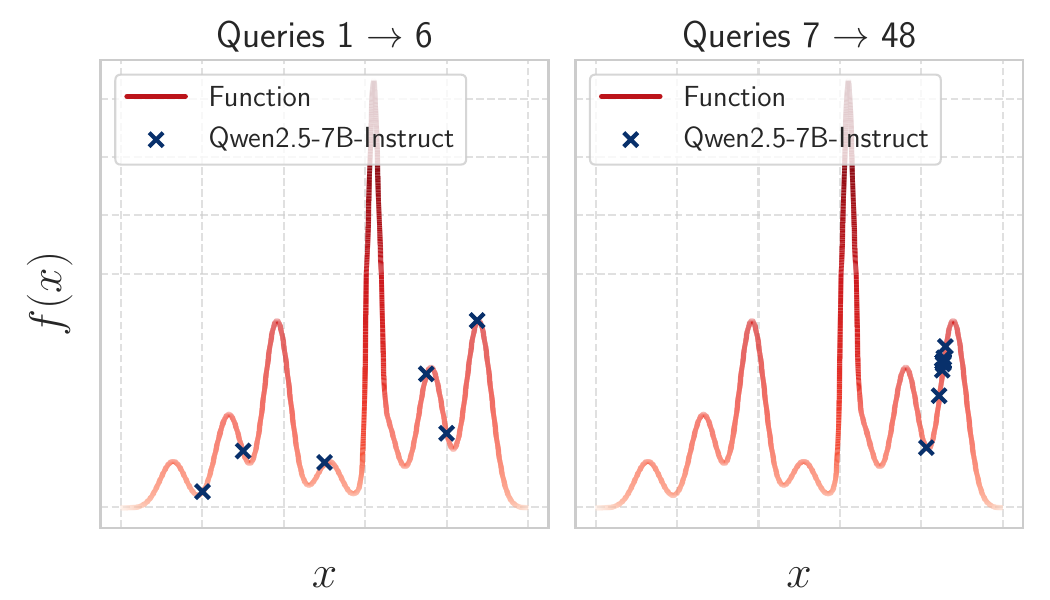}
  \end{subfigure}
  \hfill
  \begin{subfigure}[t]{0.33\textwidth}
    \centering
    \includegraphics[width=\linewidth]{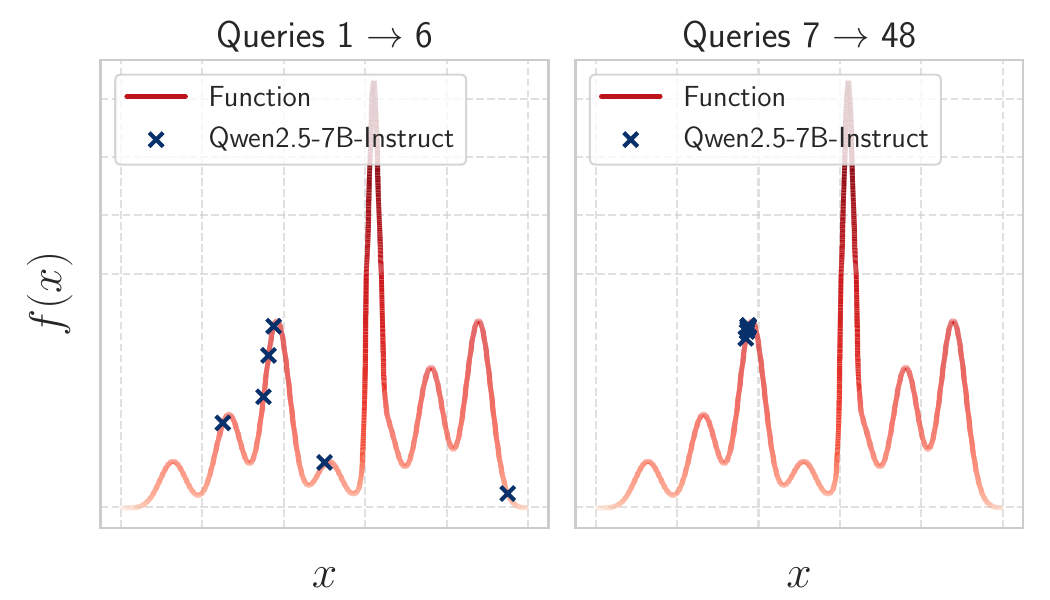}
  \end{subfigure}
  \hfill
  \begin{subfigure}[t]{0.33\textwidth}
    \centering
    \includegraphics[width=\linewidth]{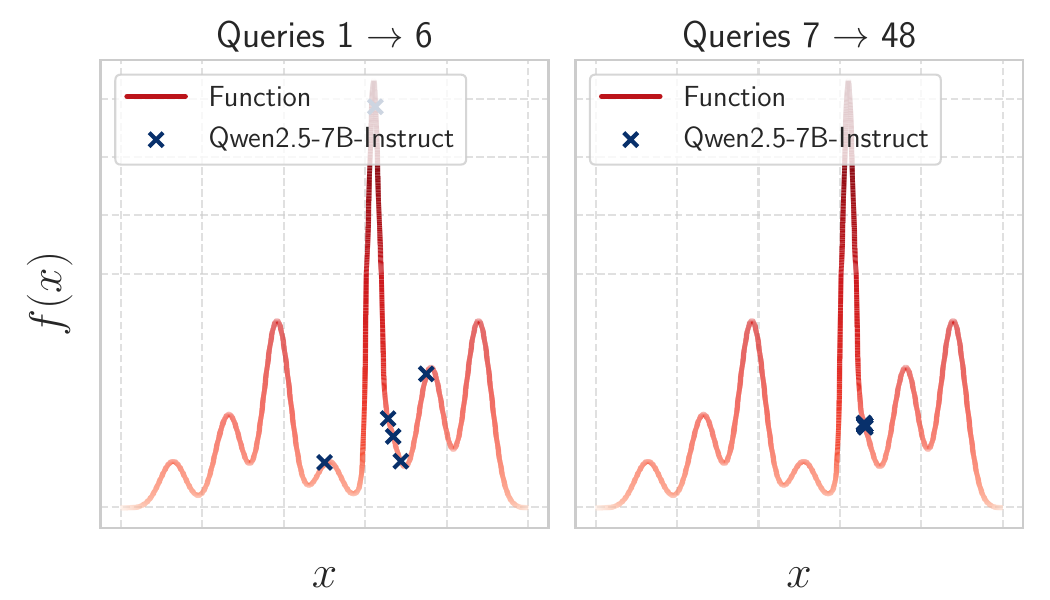}
  \end{subfigure}
  \hfill
  \begin{subfigure}[t]{0.33\textwidth}
    \centering
    \includegraphics[width=\linewidth]{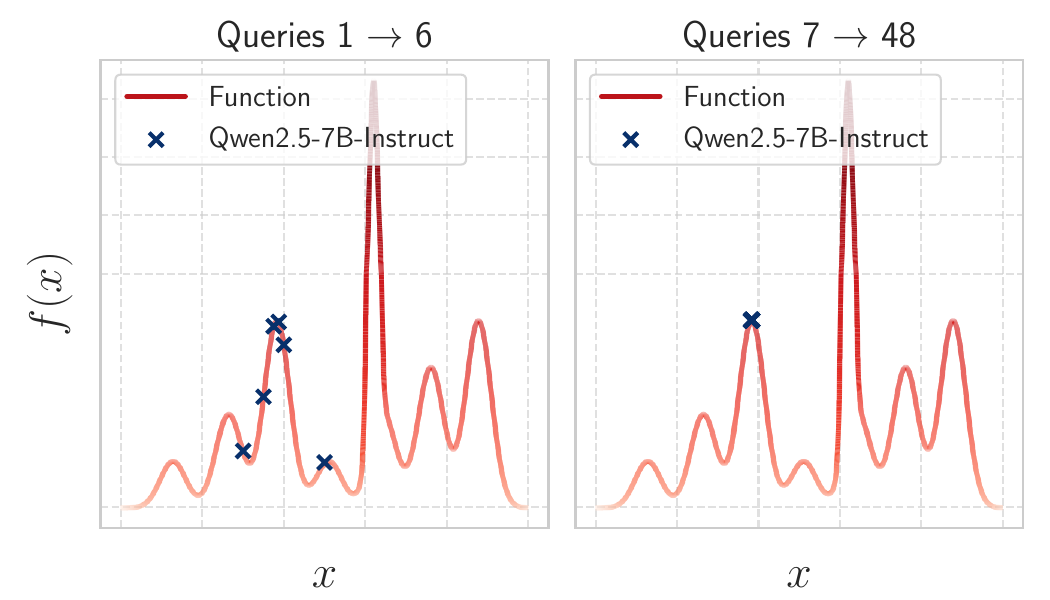}
  \end{subfigure}
  \hfill
  \begin{subfigure}[t]{0.33\textwidth}
    \centering
    \includegraphics[width=\linewidth]{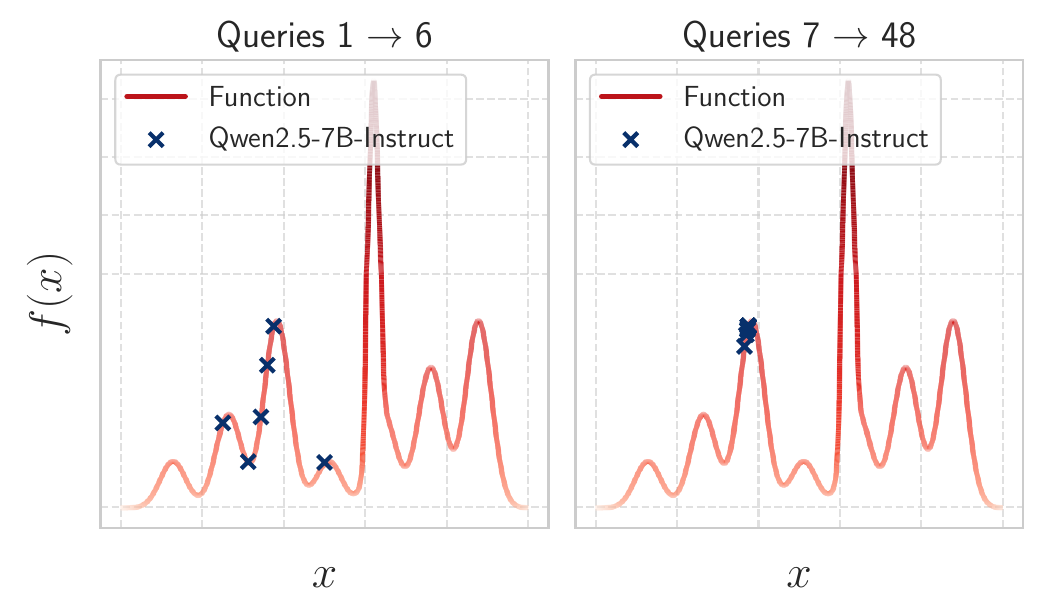}
  \end{subfigure}
  \hfill
  \begin{subfigure}[t]{0.33\textwidth}
    \centering
    \includegraphics[width=\linewidth]{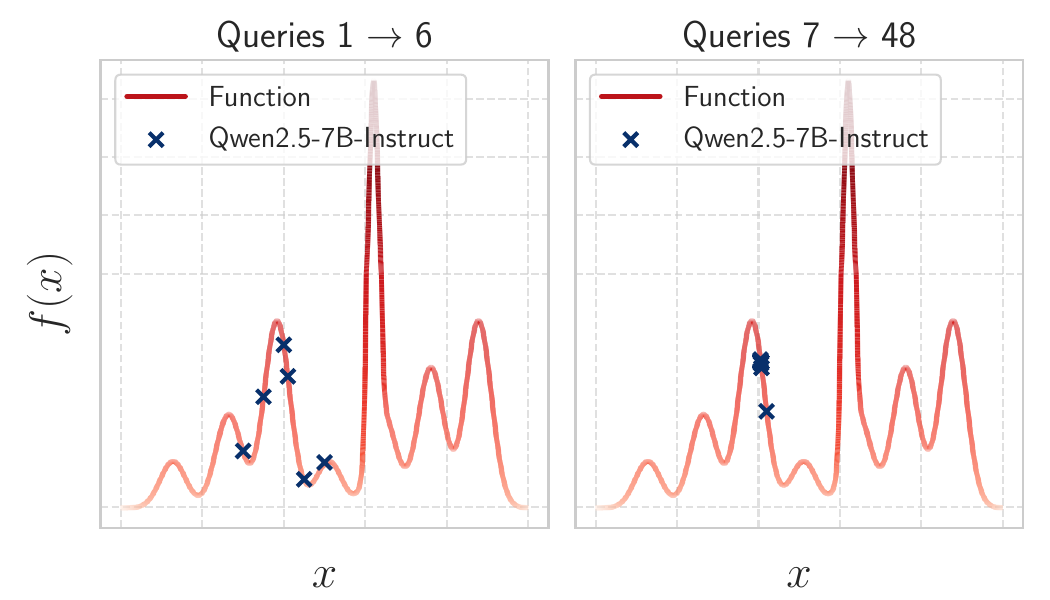}
  \end{subfigure}
  \hfill
  \begin{subfigure}[t]{0.33\textwidth}
    \centering
    \includegraphics[width=\linewidth]{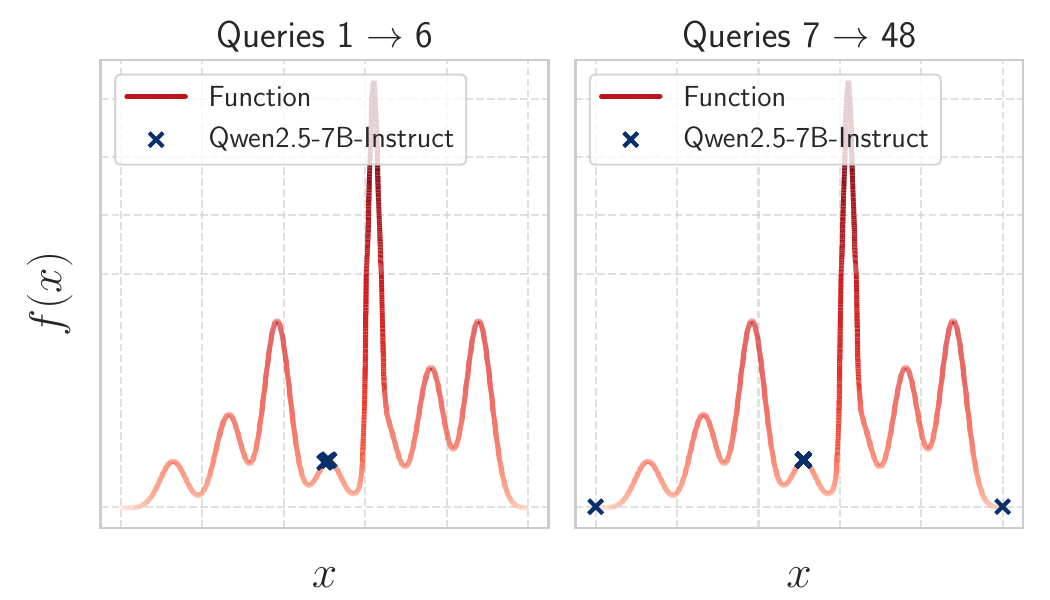}
  \end{subfigure}
  \hfill
  \begin{subfigure}[t]{0.33\textwidth}
    \centering
    \includegraphics[width=\linewidth]{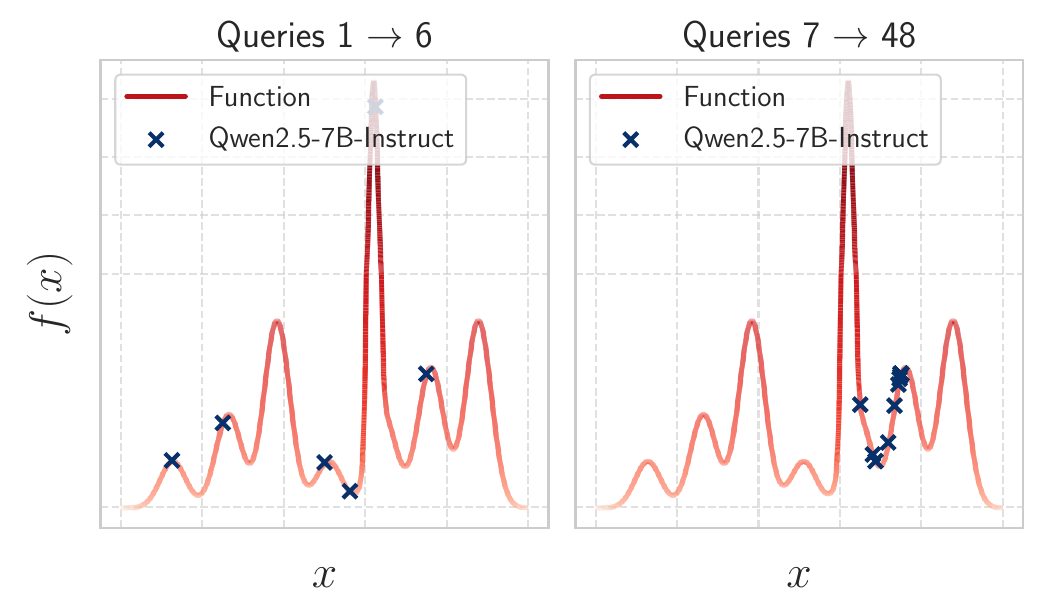}
  \end{subfigure}
  \hfill
  \begin{subfigure}[t]{0.33\textwidth}
    \centering
    \includegraphics[width=\linewidth]{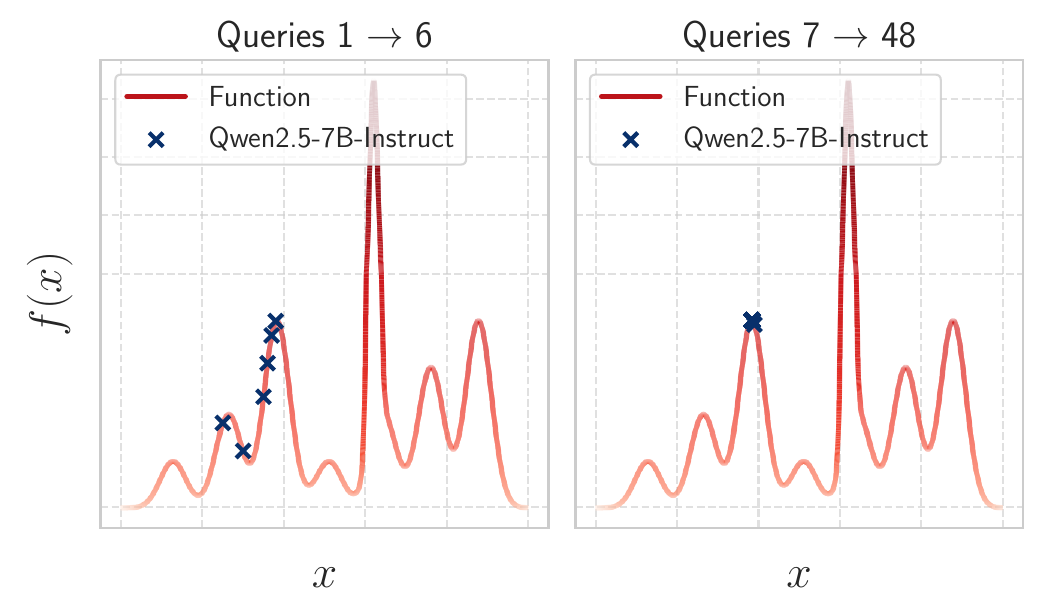}
  \end{subfigure}
  \hfill
  \begin{subfigure}[t]{0.33\textwidth}
    \centering
    \includegraphics[width=\linewidth]{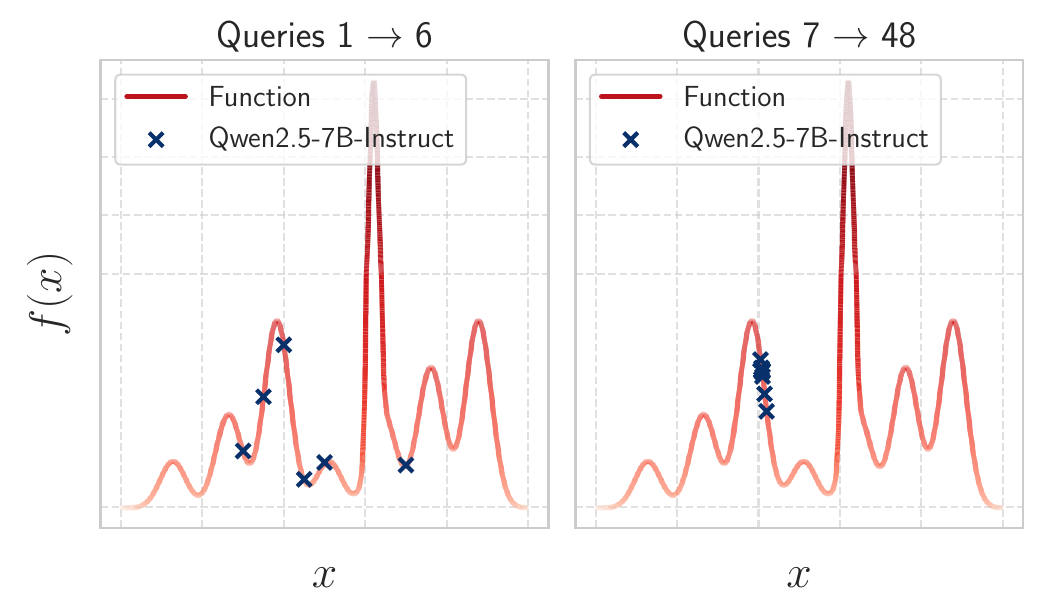}
  \end{subfigure}
  \hfill
  \begin{subfigure}[t]{0.33\textwidth}
    \centering
    \includegraphics[width=\linewidth]{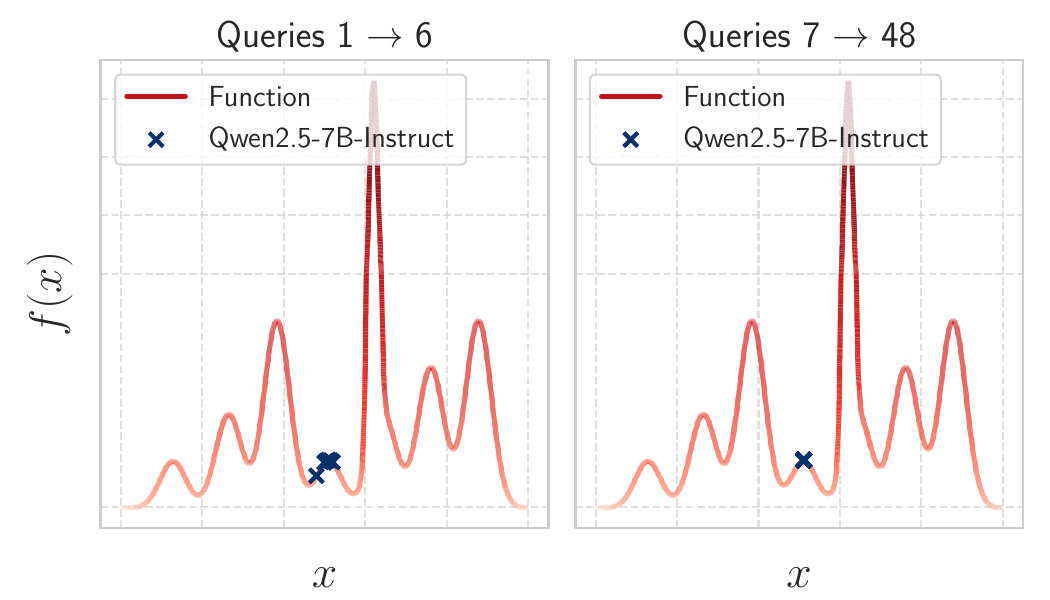}
  \end{subfigure}
  \hfill
  \begin{subfigure}[t]{0.33\textwidth}
    \centering
    \includegraphics[width=\linewidth]{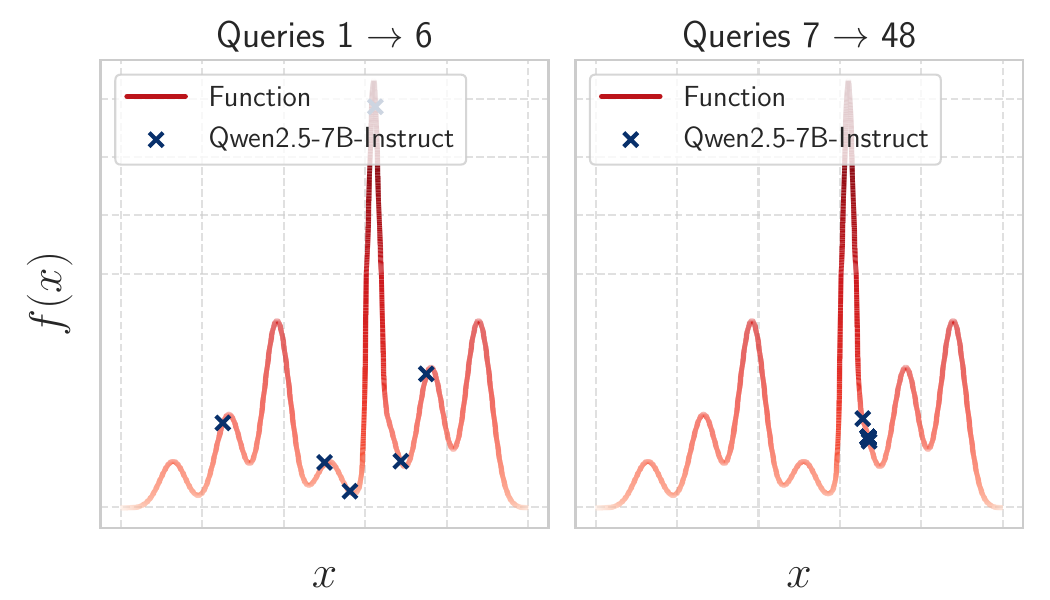}
  \end{subfigure}
  \hfill
  \begin{subfigure}[t]{0.33\textwidth}
    \centering
    \includegraphics[width=\linewidth]{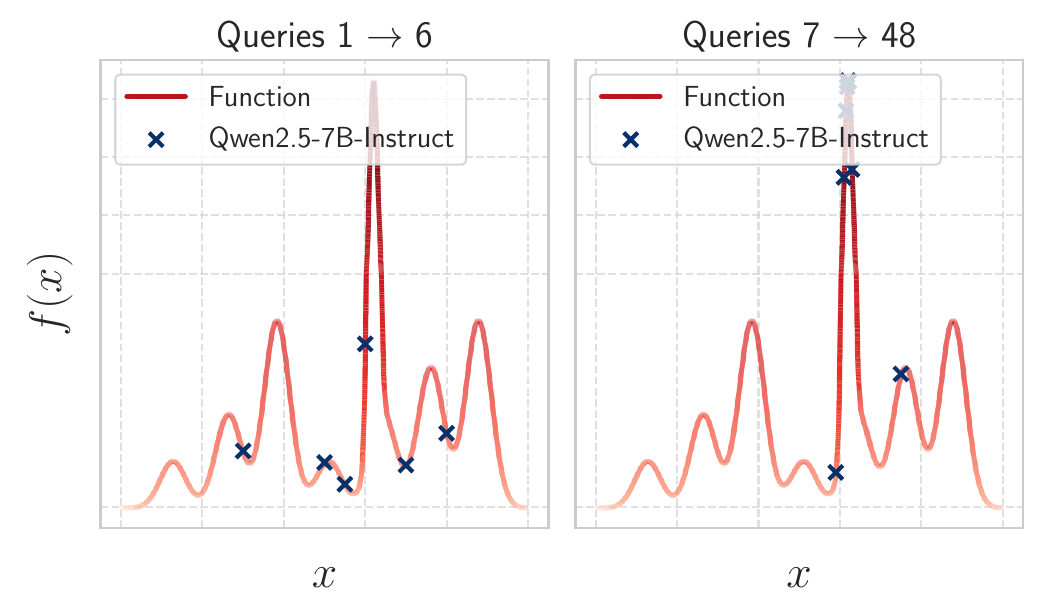}
  \end{subfigure}
  \hfill
  \begin{subfigure}[t]{0.33\textwidth}
    \centering
    \includegraphics[width=\linewidth]{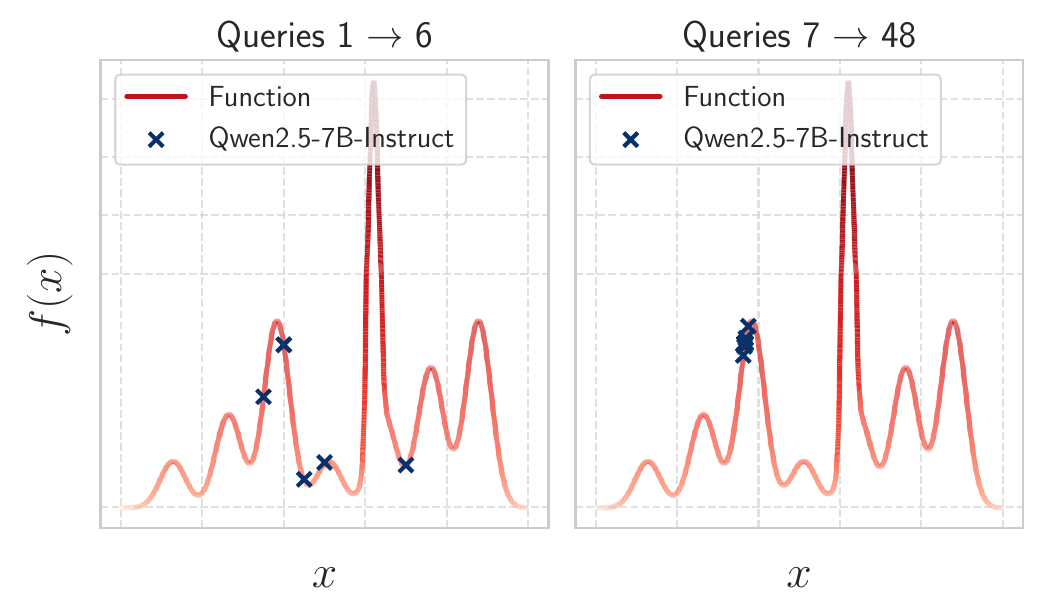}
  \end{subfigure}
  \hfill

  \caption{Visualization of 18 different episodes of \texttt{Qwen2.5-7B-Instruct} on \tsearch{} with budget $N=48$. Queries are shown as crosses on the hidden function. Early steps ($1$--$6$) explore the space, while later steps ($7$--$48$) cluster near local maxima.}
  \label{fig:interp:all:hill}
\end{figure*}
\begin{figure*}[p] 
  \centering

  \begin{subfigure}[t]{0.33\textwidth}
    \centering
    \includegraphics[width=\linewidth]{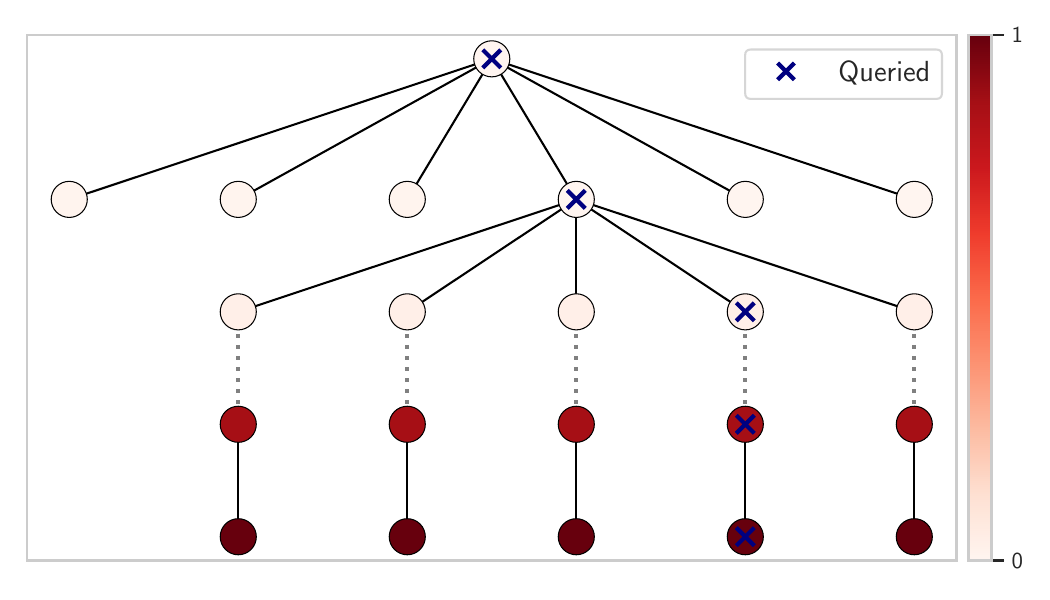}
  \end{subfigure}
  \hfill
  \begin{subfigure}[t]{0.33\textwidth}
    \centering
    \includegraphics[width=\linewidth]{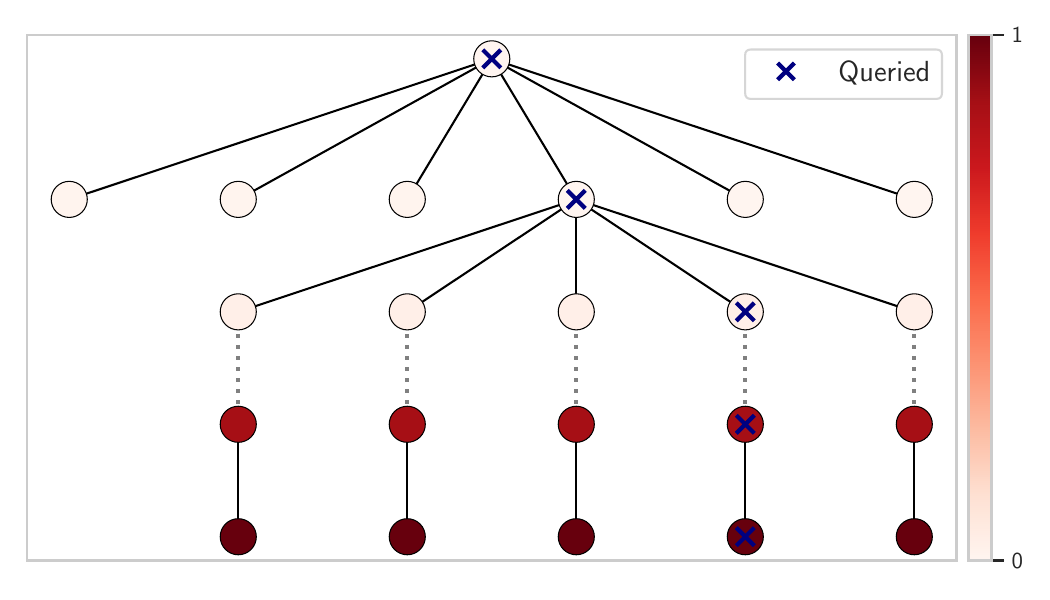}
  \end{subfigure}
  \hfill
  \begin{subfigure}[t]{0.33\textwidth}
    \centering
    \includegraphics[width=\linewidth]{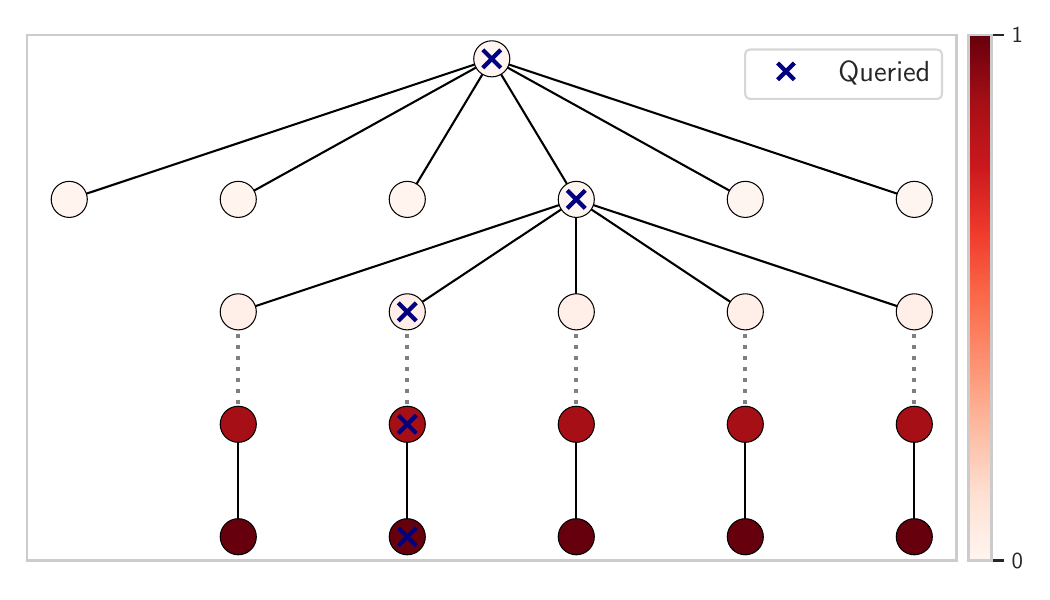}
  \end{subfigure}
  \hfill
  \begin{subfigure}[t]{0.33\textwidth}
    \centering
    \includegraphics[width=\linewidth]{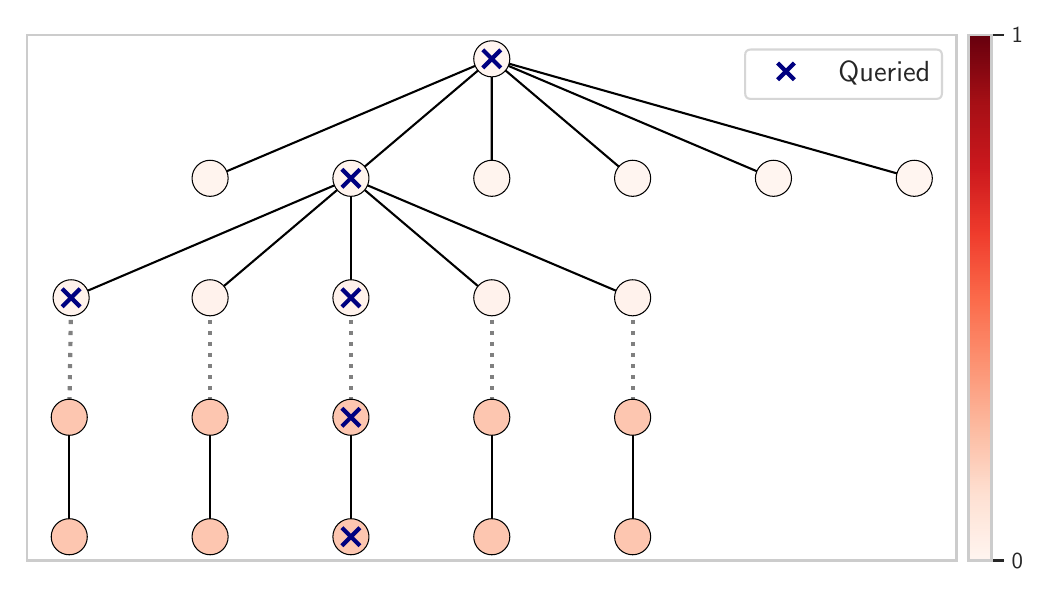}
  \end{subfigure}
  \hfill
  \begin{subfigure}[t]{0.33\textwidth}
    \centering
    \includegraphics[width=\linewidth]{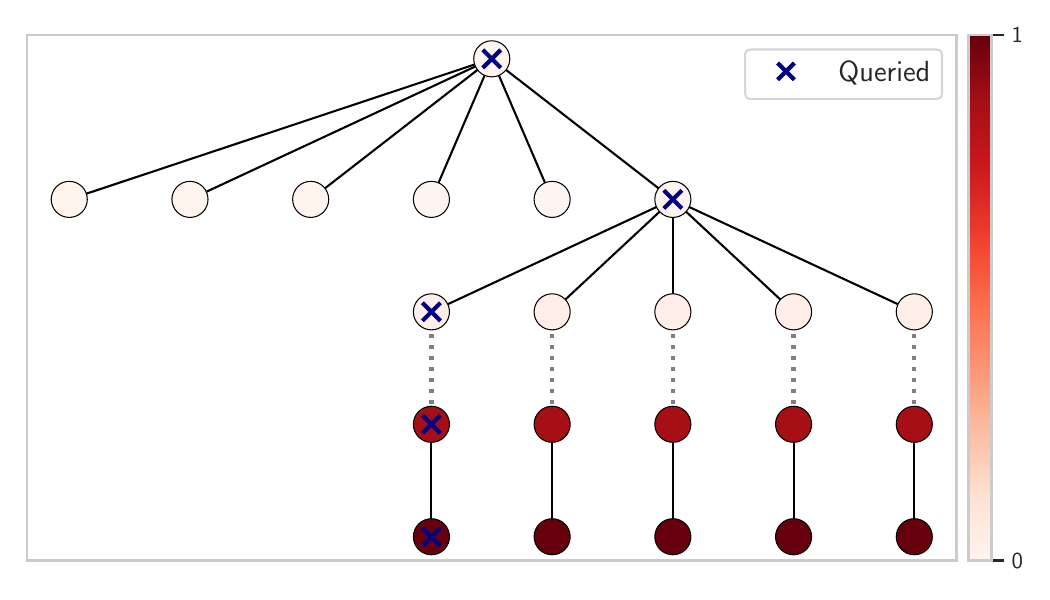}
  \end{subfigure}
  \hfill
  \begin{subfigure}[t]{0.33\textwidth}
    \centering
    \includegraphics[width=\linewidth]{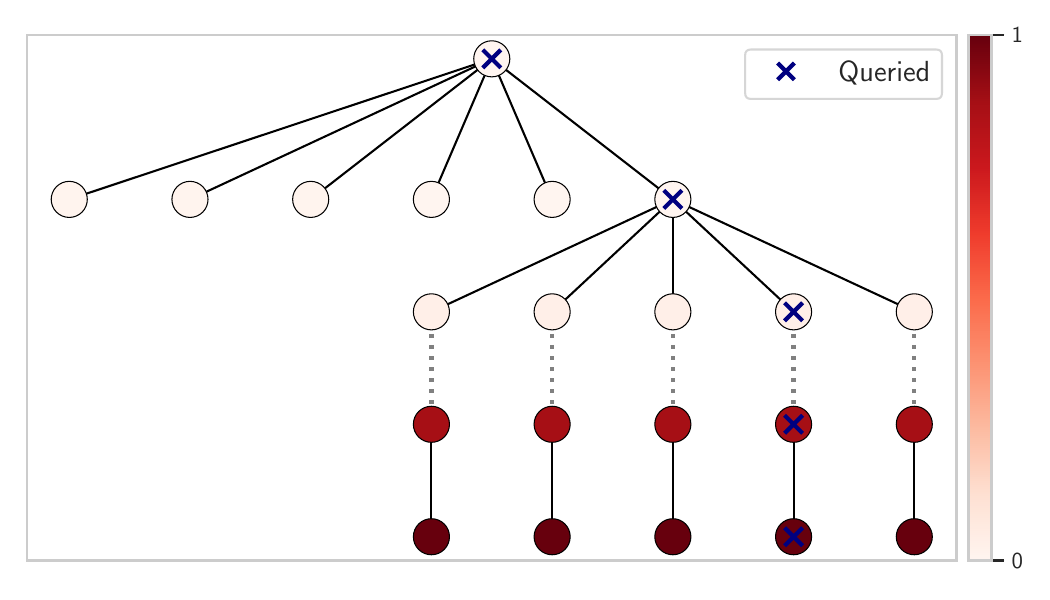}
  \end{subfigure}
  \hfill
  \begin{subfigure}[t]{0.33\textwidth}
    \centering
    \includegraphics[width=\linewidth]{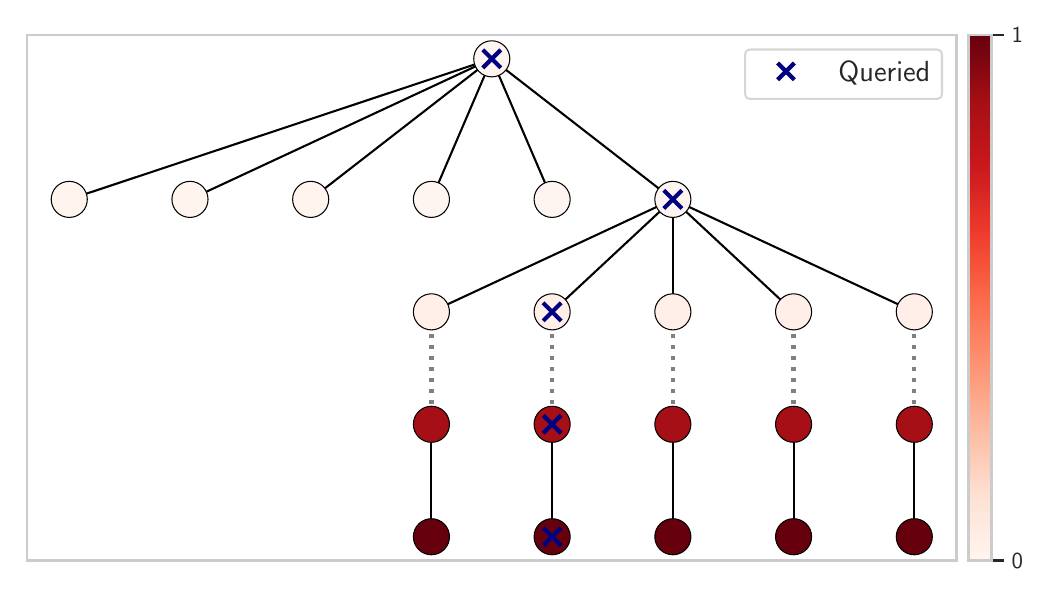}
  \end{subfigure}
  \hfill
  \begin{subfigure}[t]{0.33\textwidth}
    \centering
    \includegraphics[width=\linewidth]{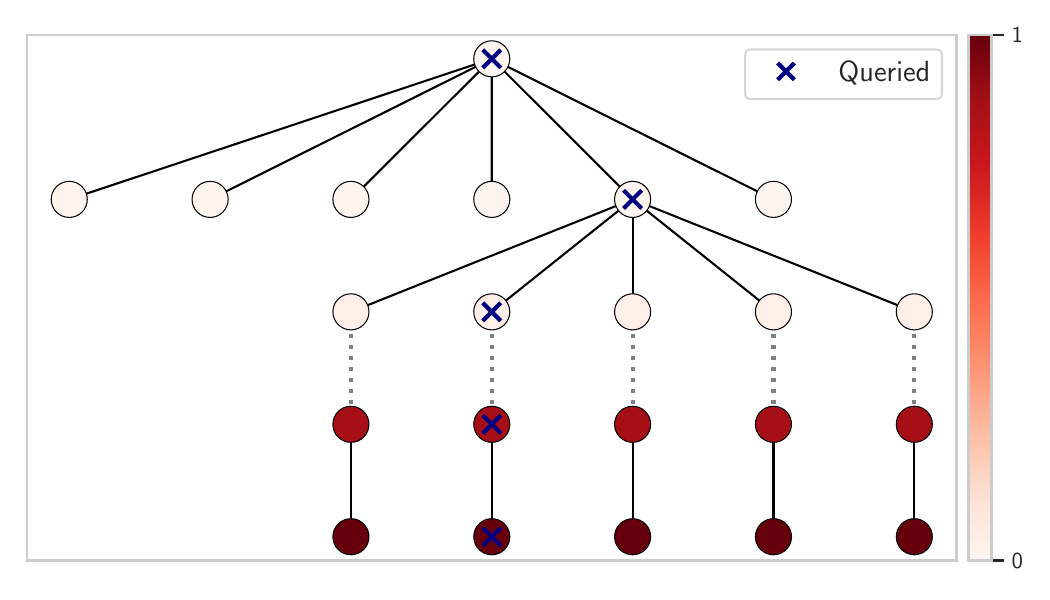}
  \end{subfigure}
  \hfill
  \begin{subfigure}[t]{0.33\textwidth}
    \centering
    \includegraphics[width=\linewidth]{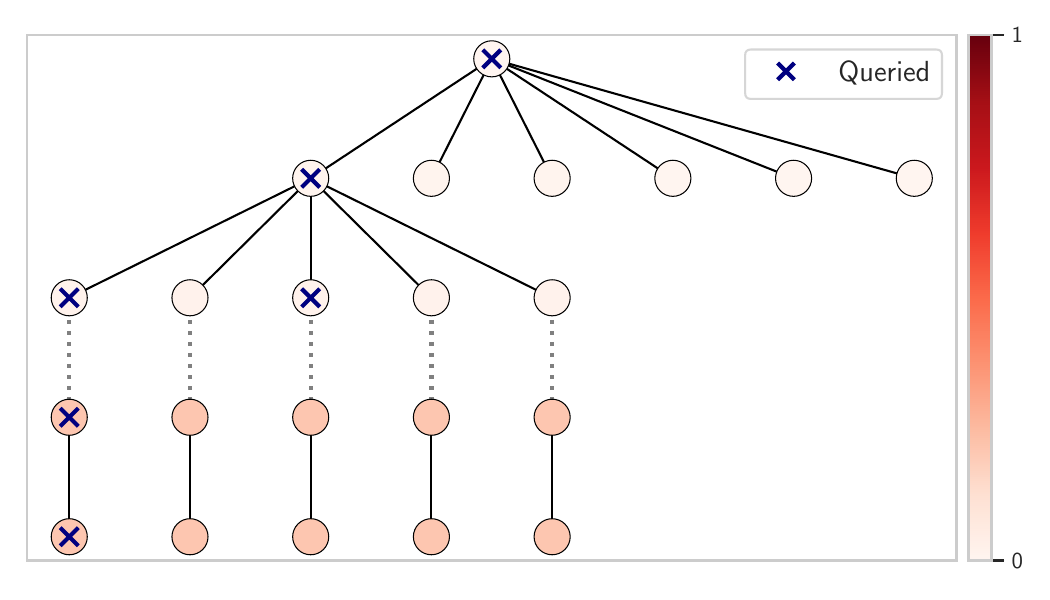}
  \end{subfigure}
  \hfill
  \begin{subfigure}[t]{0.33\textwidth}
    \centering
    \includegraphics[width=\linewidth]{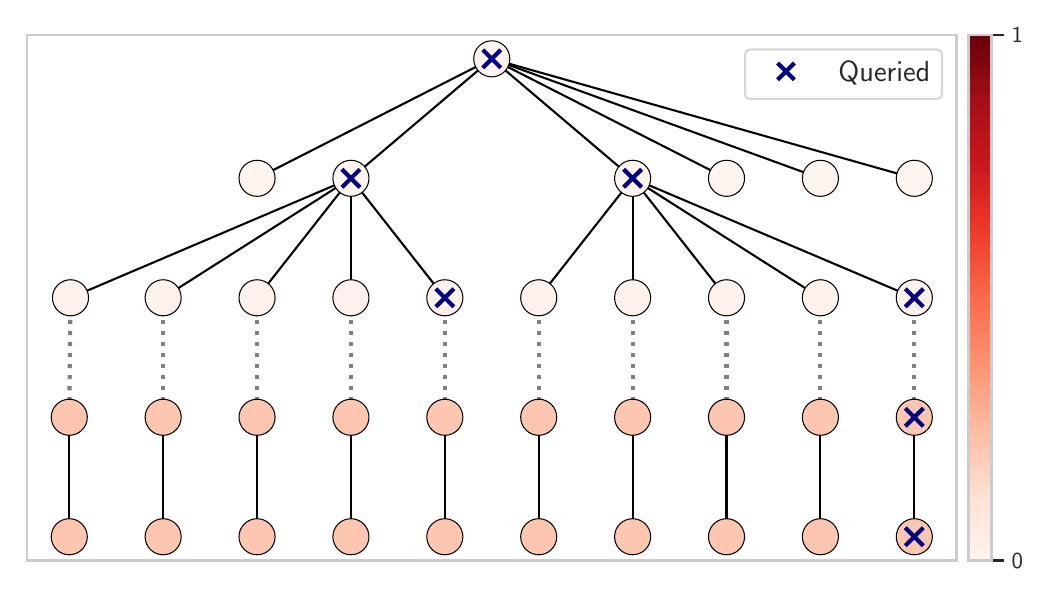}
  \end{subfigure}
  \hfill
  \begin{subfigure}[t]{0.33\textwidth}
    \centering
    \includegraphics[width=\linewidth]{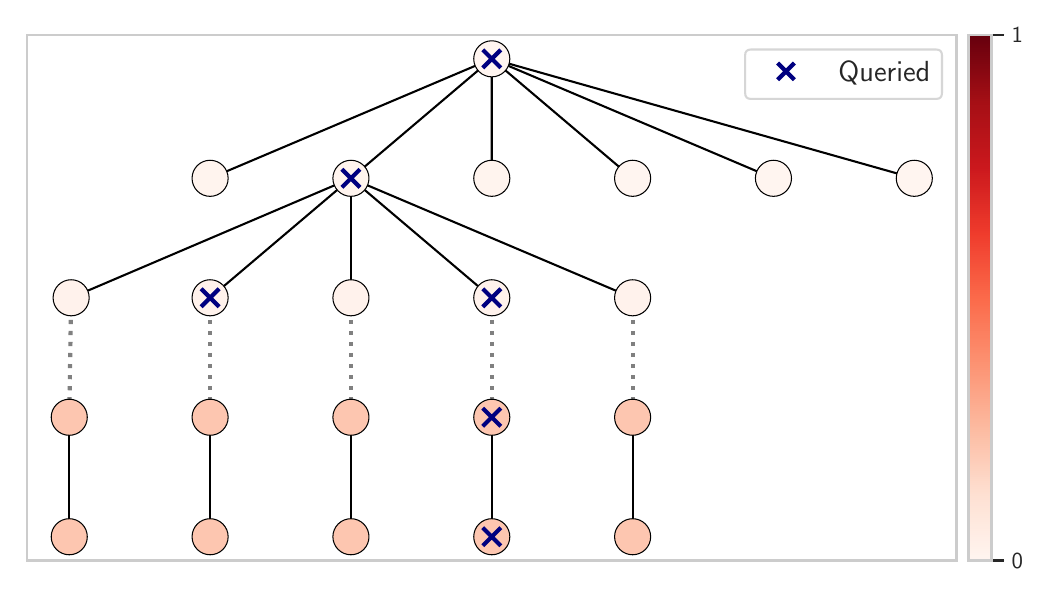}
  \end{subfigure}
  \hfill
  \begin{subfigure}[t]{0.33\textwidth}
    \centering
    \includegraphics[width=\linewidth]{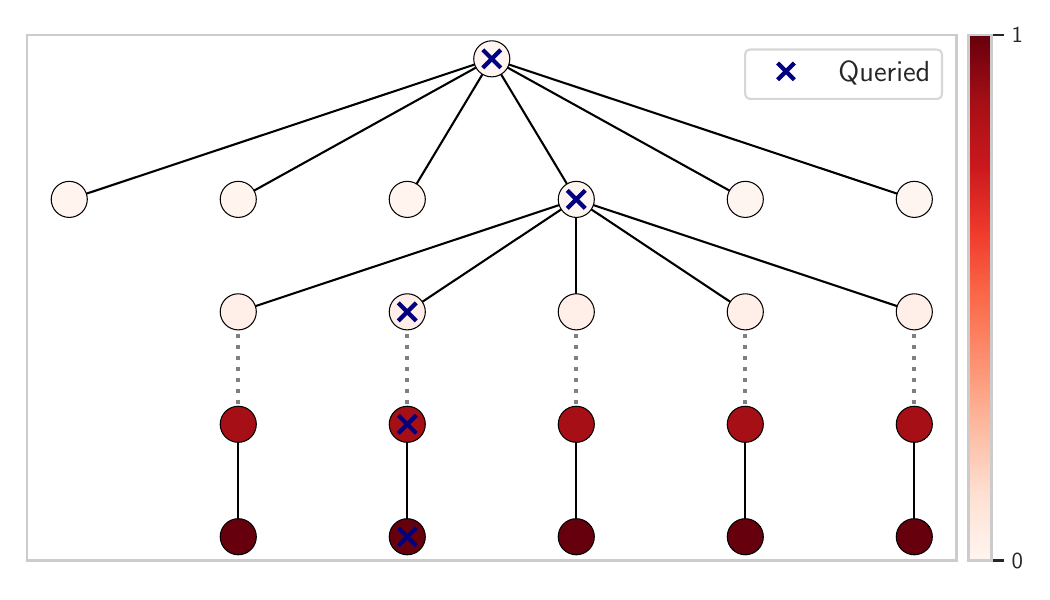}
  \end{subfigure}
  \hfill
  \begin{subfigure}[t]{0.33\textwidth}
    \centering
    \includegraphics[width=\linewidth]{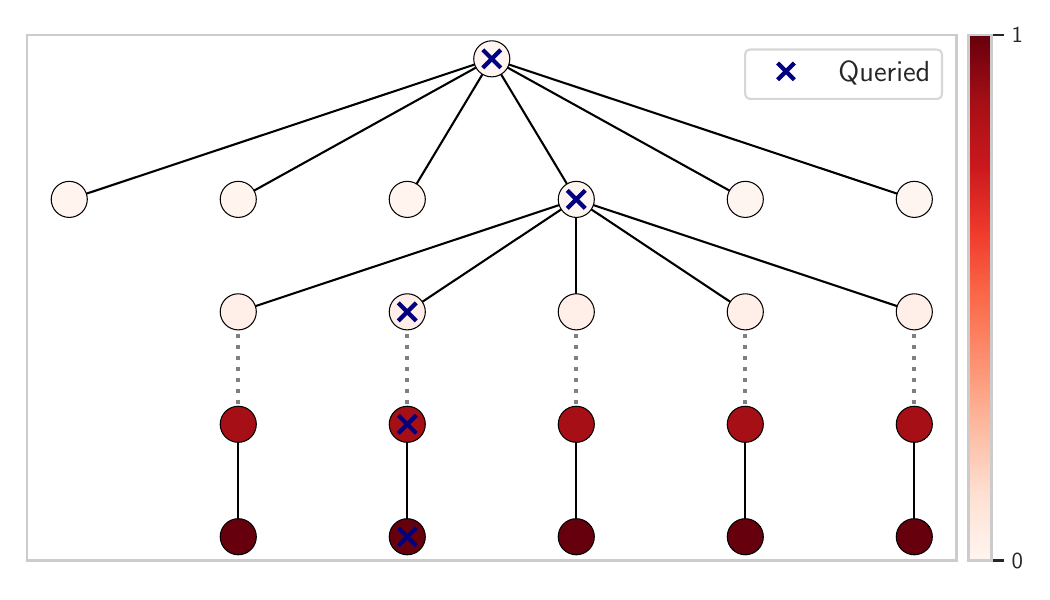}
  \end{subfigure}
  \hfill
  \begin{subfigure}[t]{0.33\textwidth}
    \centering
    \includegraphics[width=\linewidth]{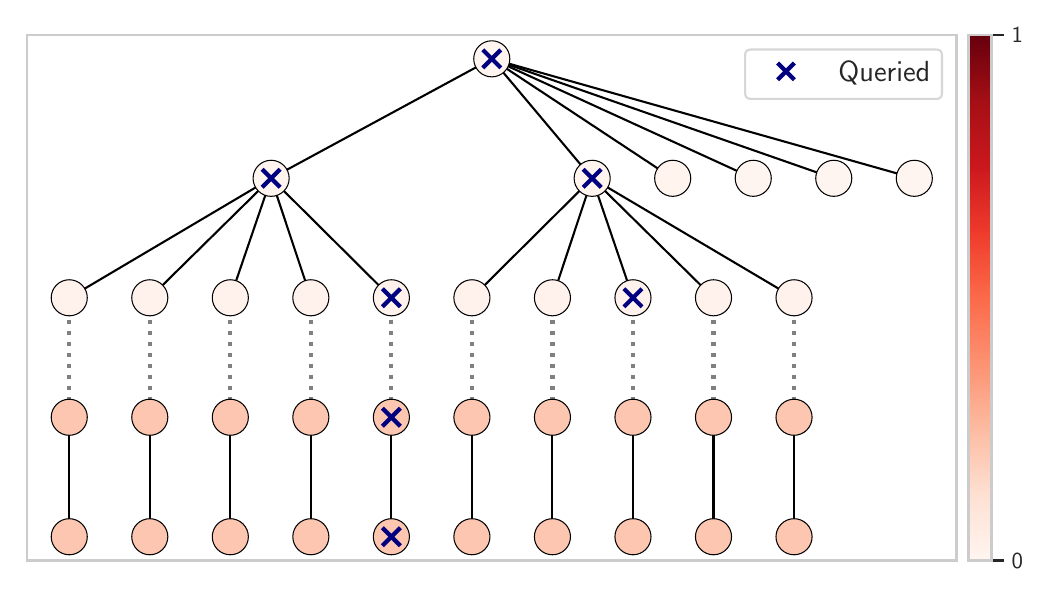}
  \end{subfigure}
  \hfill
  \begin{subfigure}[t]{0.33\textwidth}
    \centering
    \includegraphics[width=\linewidth]{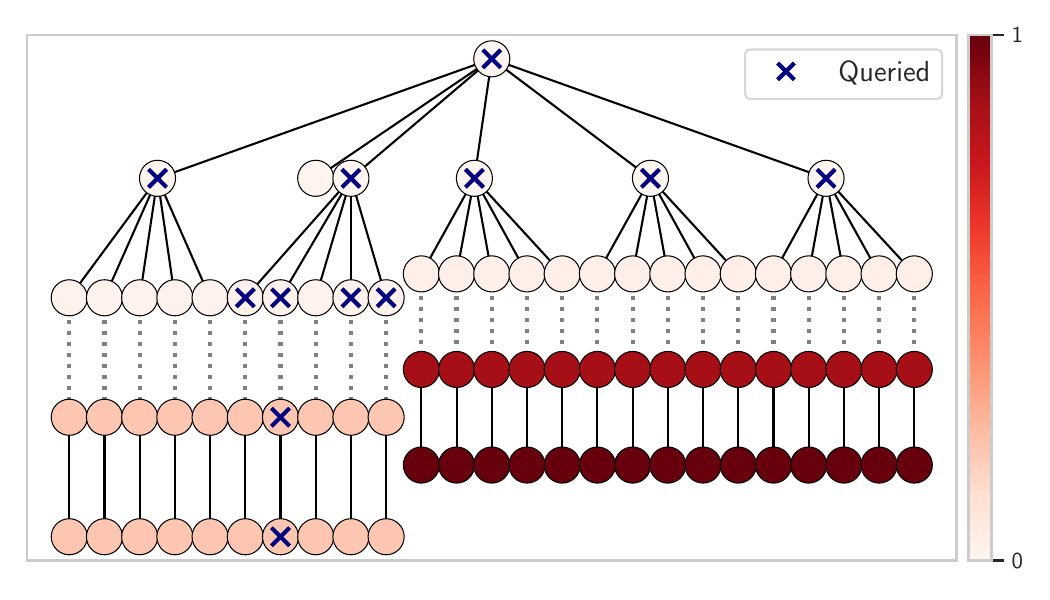}
  \end{subfigure}
  \hfill
  \begin{subfigure}[t]{0.33\textwidth}
    \centering
    \includegraphics[width=\linewidth]{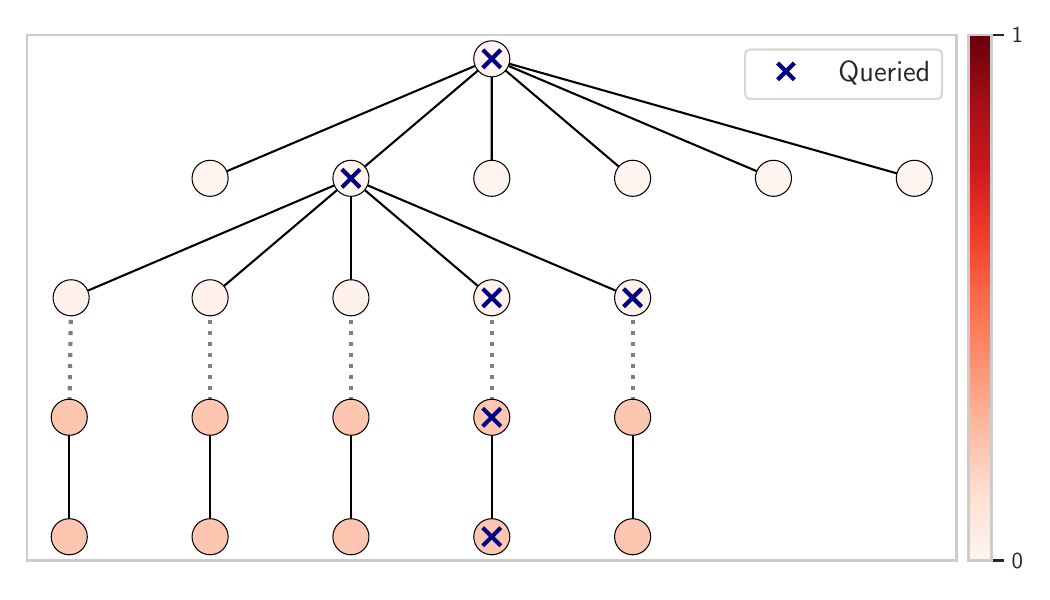}
  \end{subfigure}
  \hfill
  \begin{subfigure}[t]{0.33\textwidth}
    \centering
    \includegraphics[width=\linewidth]{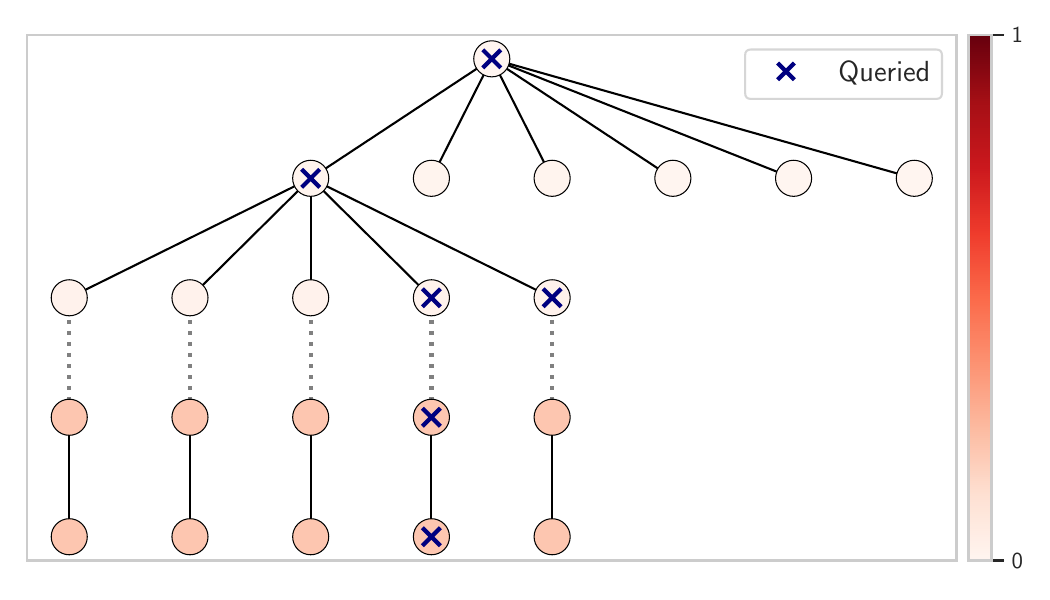}
  \end{subfigure}
  \hfill
  \begin{subfigure}[t]{0.33\textwidth}
    \centering
    \includegraphics[width=\linewidth]{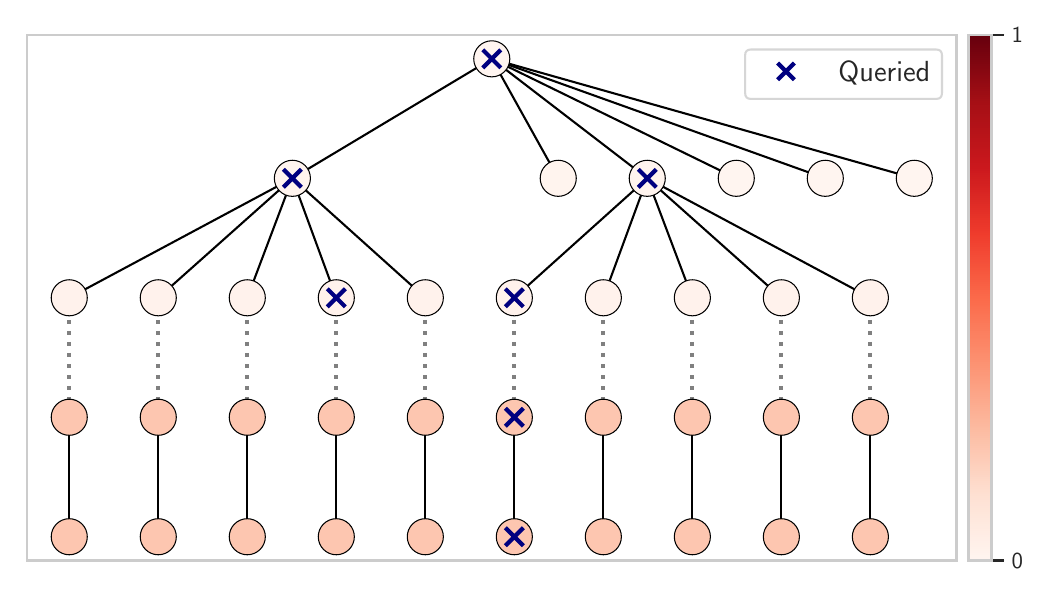}
  \end{subfigure}
  \hfill

  \caption{Visualization of 18 different episodes of \texttt{Qwen2.5-7B-Instruct} on \ttree{} with budget $N=48$. Crosses indicate queried nodes. Darker node color indicate higher reward. Other unqueried nodes are omitted. In almost all episodes, the model greedily descends any branch it enters.}
  \label{fig:interp:all:tree}
\end{figure*}
\begin{figure*}[p] 
  \centering

  \begin{subfigure}[t]{0.33\textwidth}
    \centering
    \includegraphics[width=\linewidth]{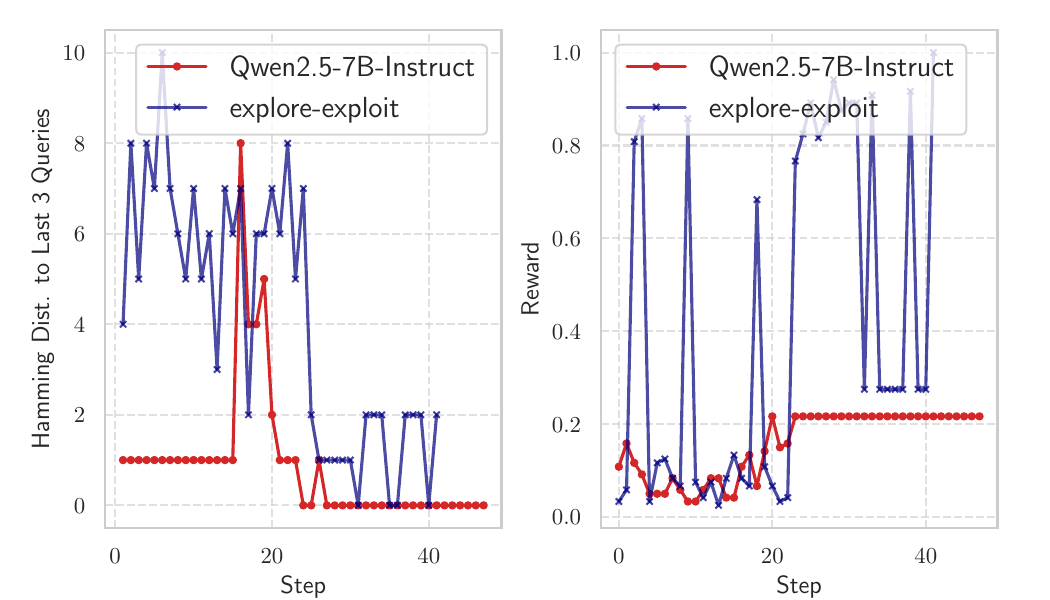}
  \end{subfigure}
  \hfill
  \begin{subfigure}[t]{0.33\textwidth}
    \centering
    \includegraphics[width=\linewidth]{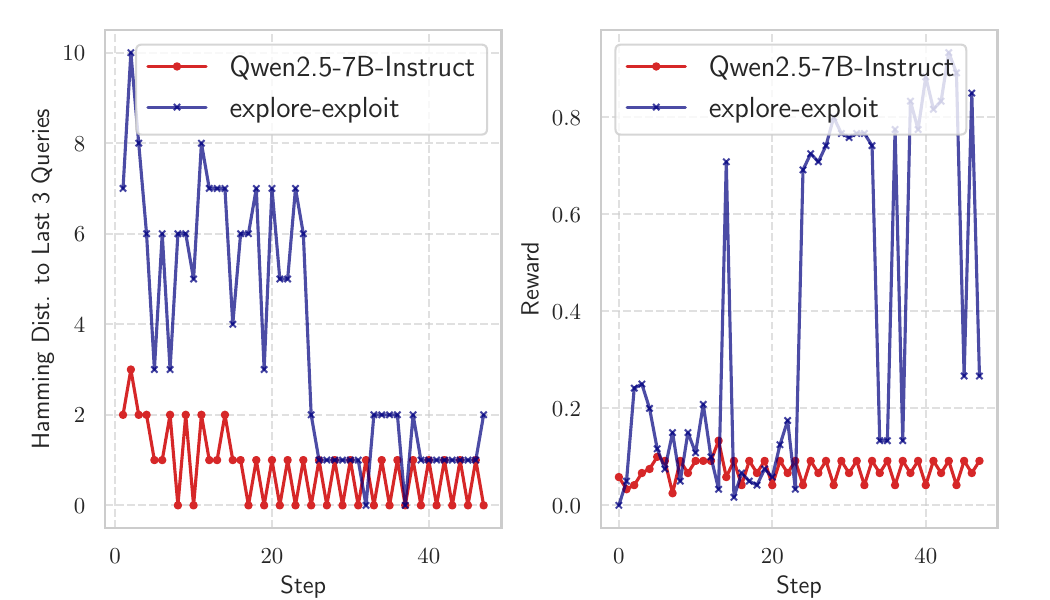}
  \end{subfigure}
  \hfill
  \begin{subfigure}[t]{0.33\textwidth}
    \centering
    \includegraphics[width=\linewidth]{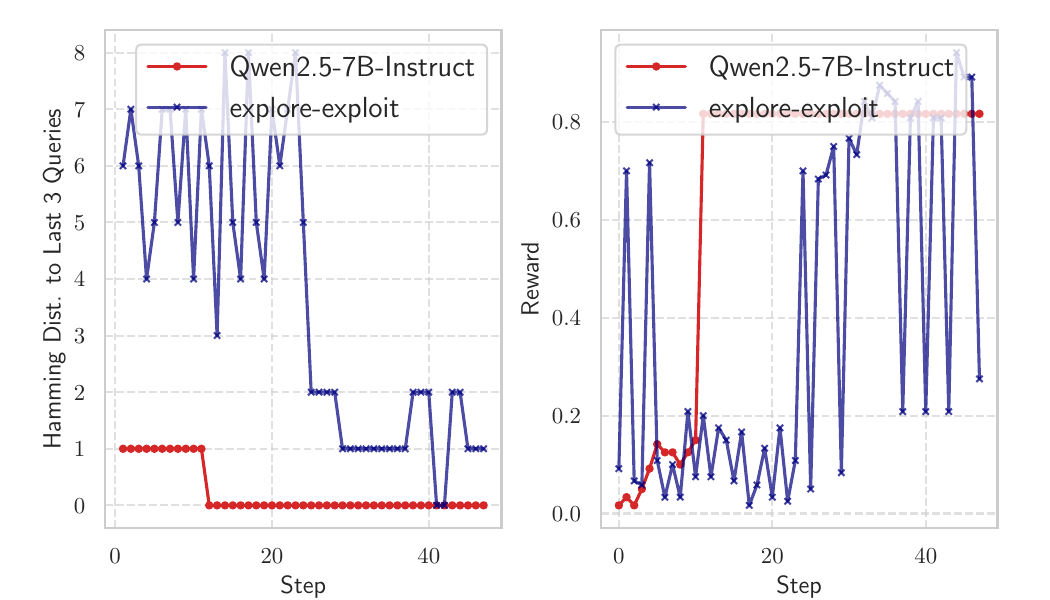}
  \end{subfigure}
  \hfill
  \begin{subfigure}[t]{0.33\textwidth}
    \centering
    \includegraphics[width=\linewidth]{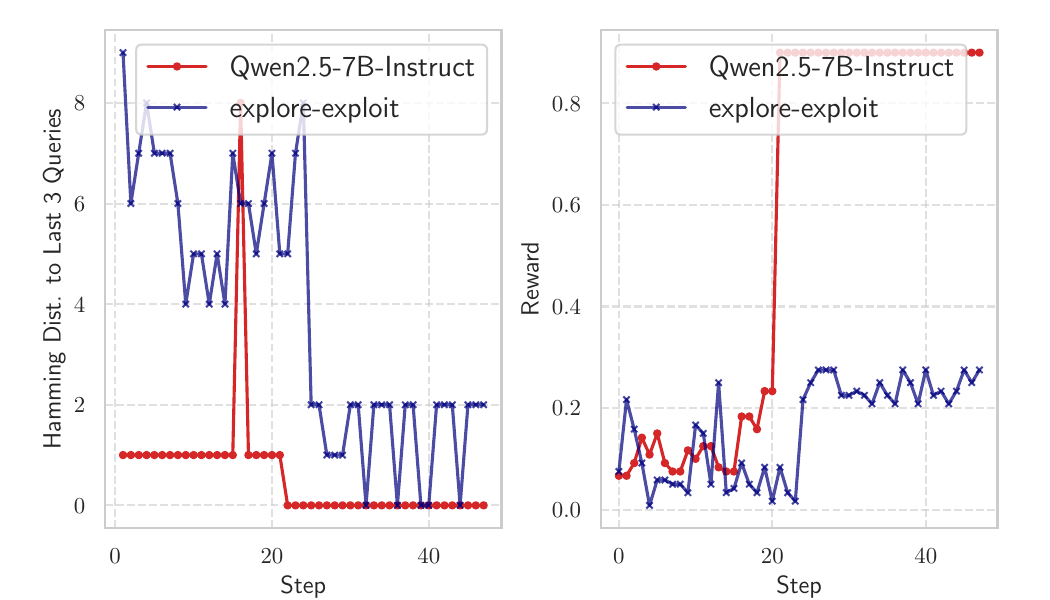}
  \end{subfigure}
  \hfill
  \begin{subfigure}[t]{0.33\textwidth}
    \centering
    \includegraphics[width=\linewidth]{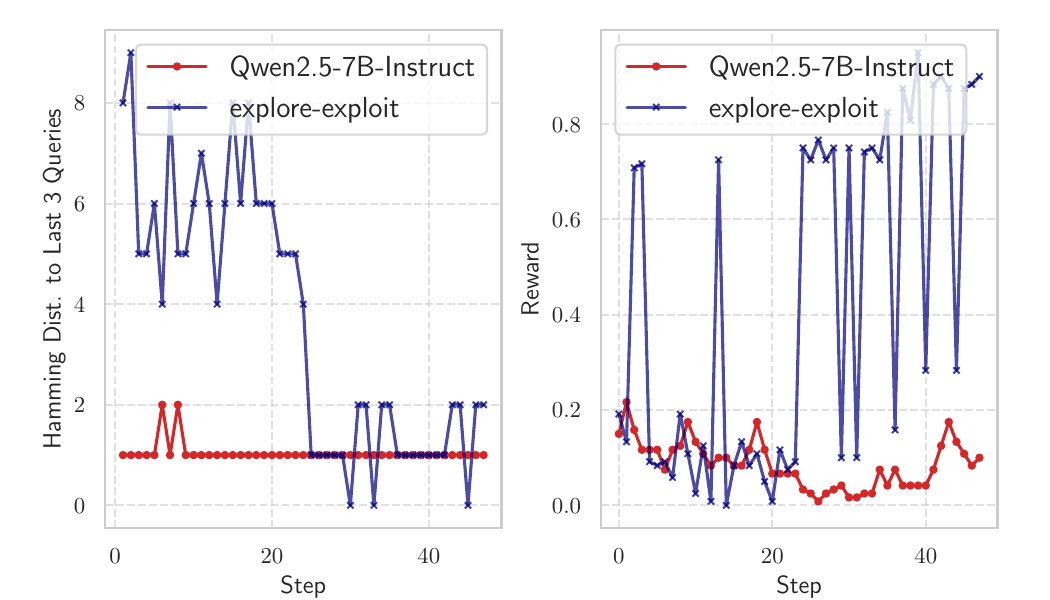}
  \end{subfigure}
  \hfill
  \begin{subfigure}[t]{0.33\textwidth}
    \centering
    \includegraphics[width=\linewidth]{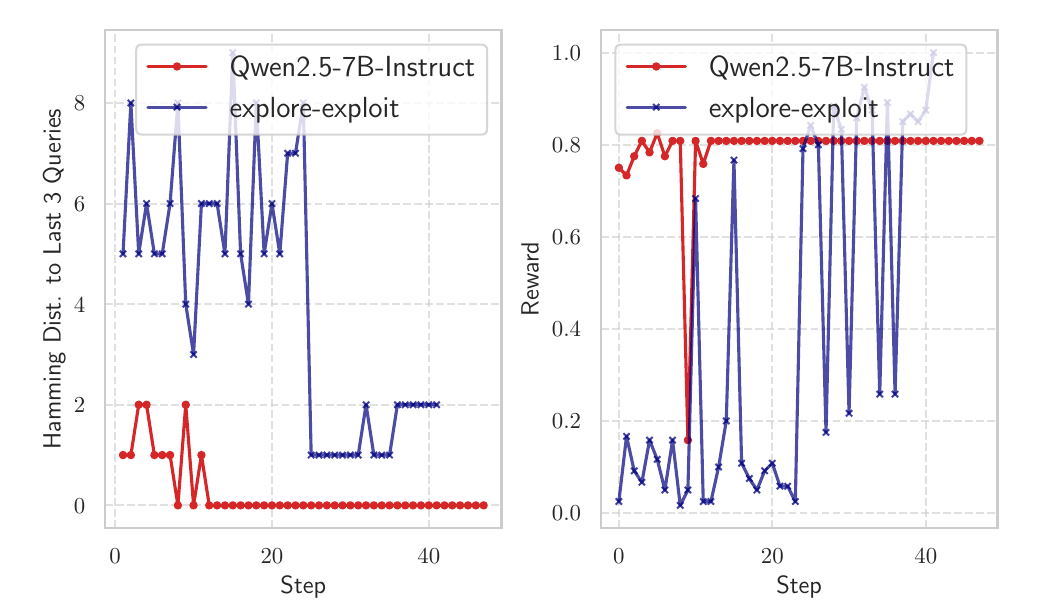}
  \end{subfigure}
  \hfill
  \begin{subfigure}[t]{0.33\textwidth}
    \centering
    \includegraphics[width=\linewidth]{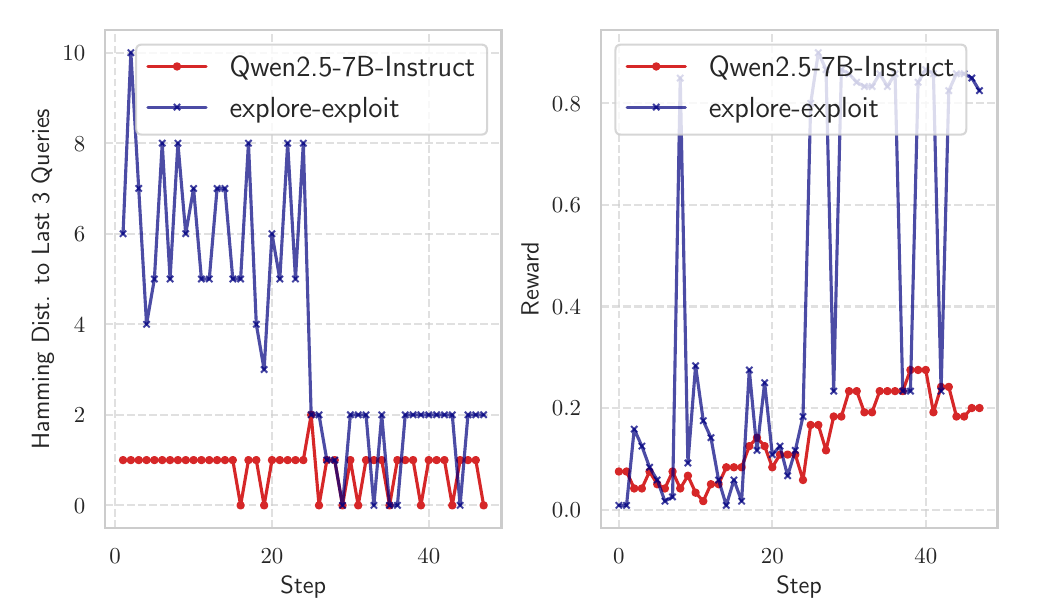}
  \end{subfigure}
  \hfill
  \begin{subfigure}[t]{0.33\textwidth}
    \centering
    \includegraphics[width=\linewidth]{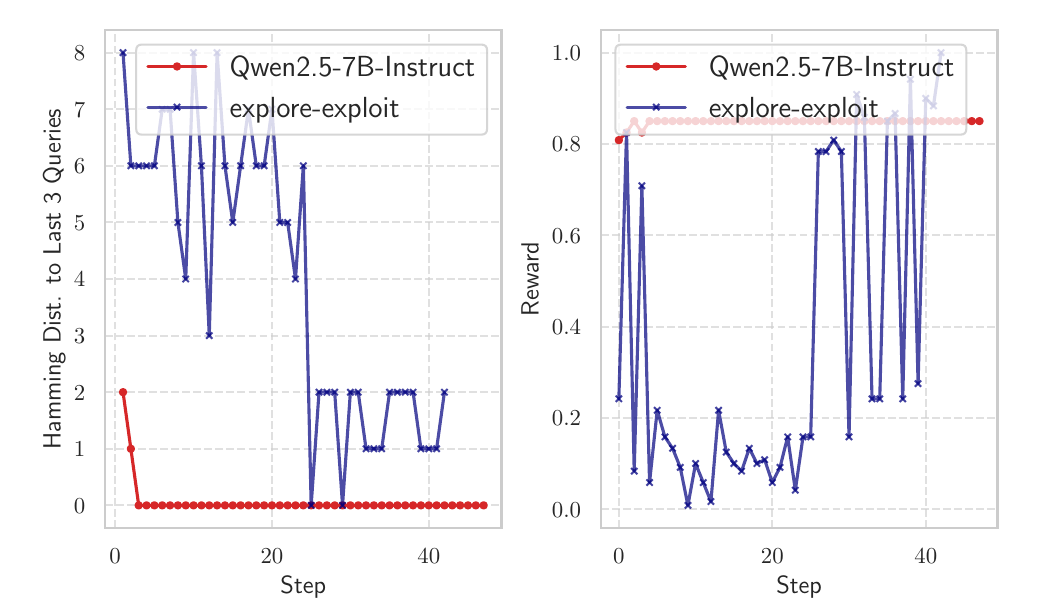}
  \end{subfigure}
  \hfill
  \begin{subfigure}[t]{0.33\textwidth}
    \centering
    \includegraphics[width=\linewidth]{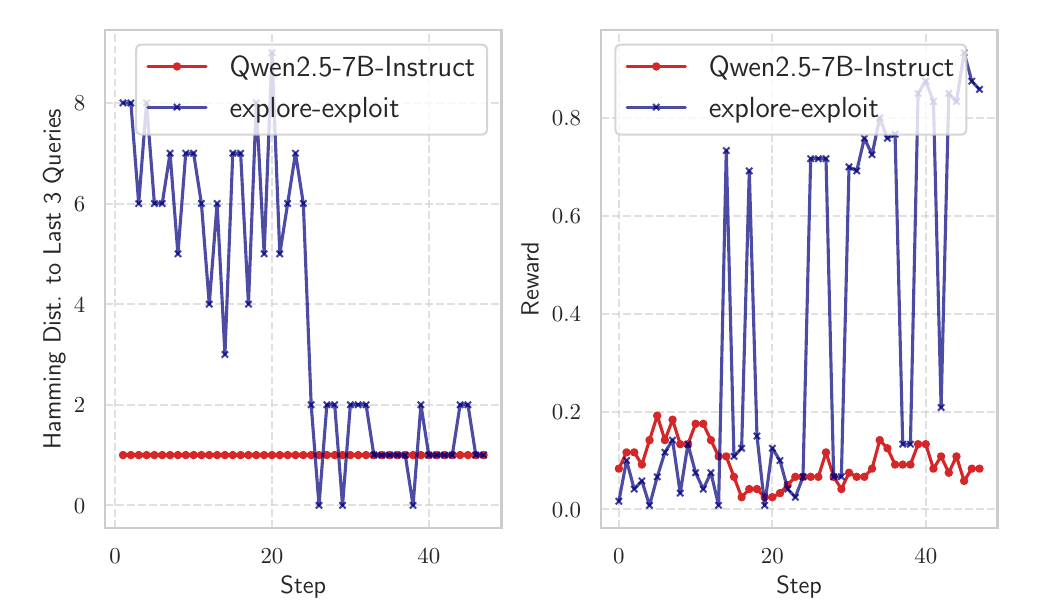}
  \end{subfigure}
  \hfill
  \begin{subfigure}[t]{0.33\textwidth}
    \centering
    \includegraphics[width=\linewidth]{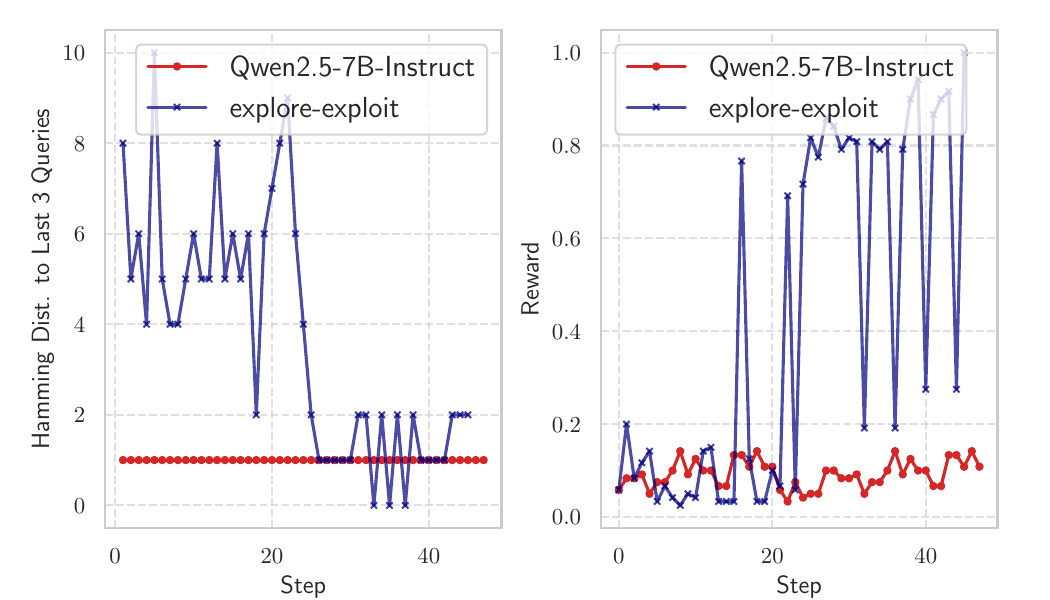}
  \end{subfigure}
  \hfill
  \begin{subfigure}[t]{0.33\textwidth}
    \centering
    \includegraphics[width=\linewidth]{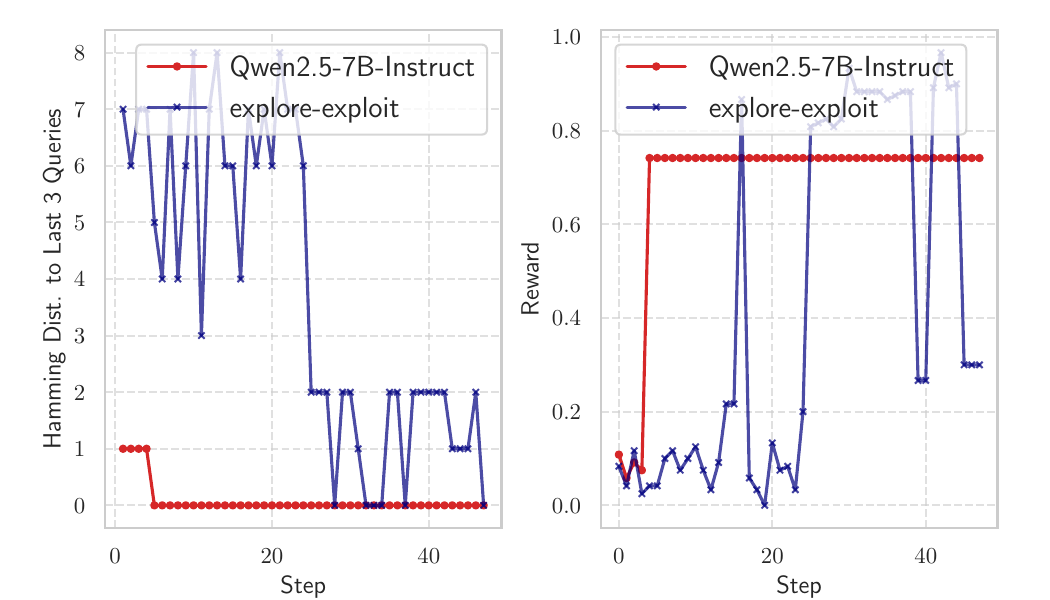}
  \end{subfigure}
  \hfill
  \begin{subfigure}[t]{0.33\textwidth}
    \centering
    \includegraphics[width=\linewidth]{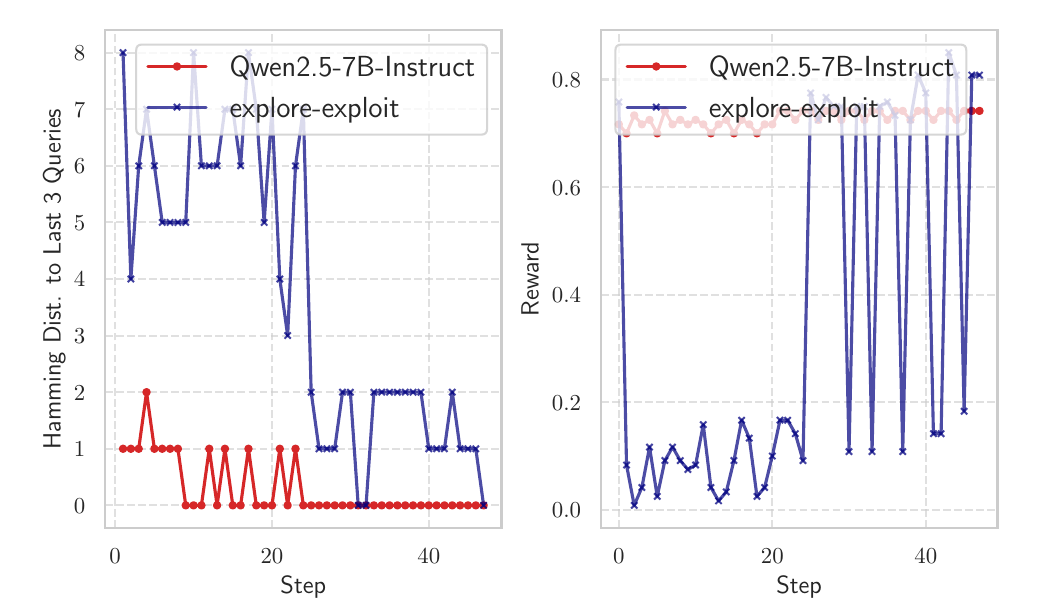}
  \end{subfigure}
  \hfill
  \begin{subfigure}[t]{0.33\textwidth}
    \centering
    \includegraphics[width=\linewidth]{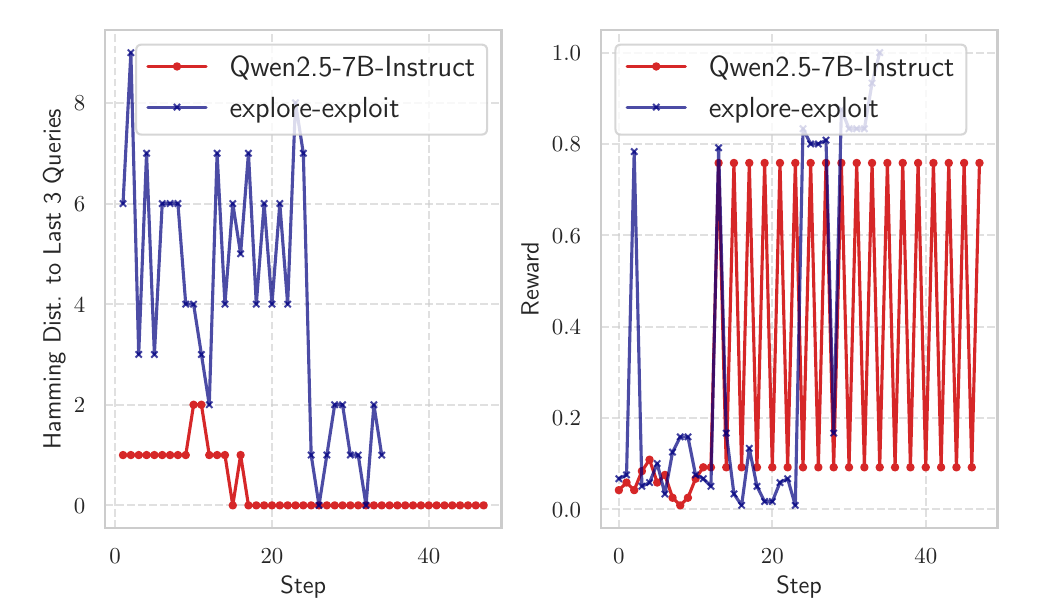}
  \end{subfigure}
  \hfill
  \begin{subfigure}[t]{0.33\textwidth}
    \centering
    \includegraphics[width=\linewidth]{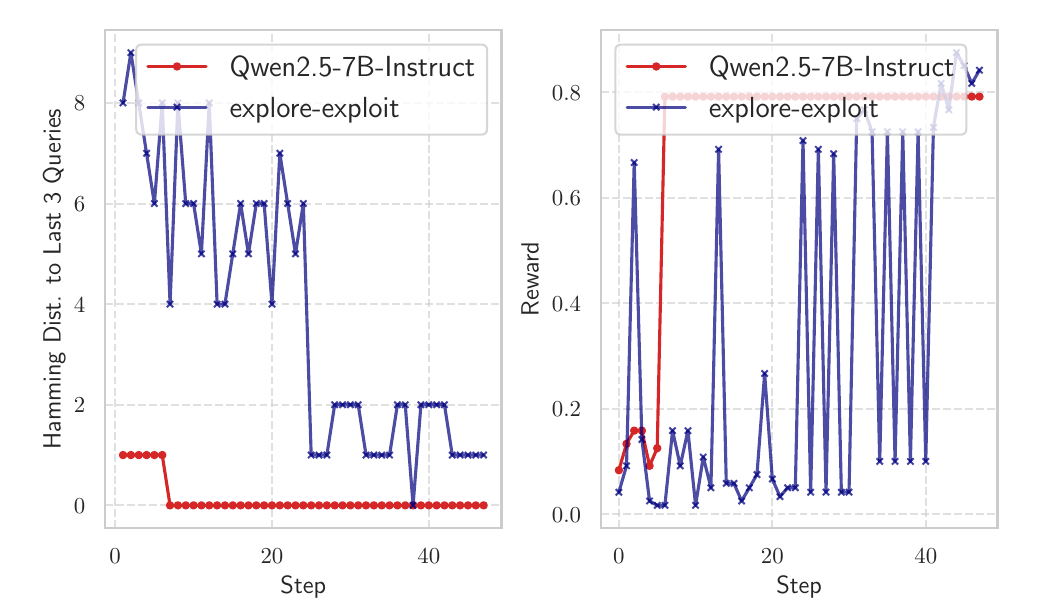}
  \end{subfigure}
  \hfill
  \begin{subfigure}[t]{0.33\textwidth}
    \centering
    \includegraphics[width=\linewidth]{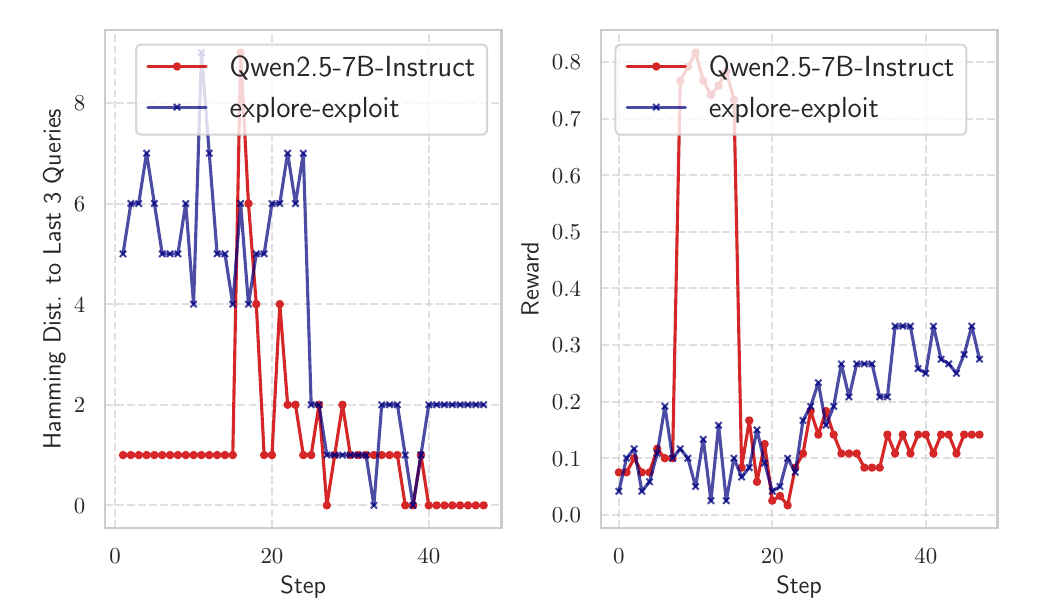}
  \end{subfigure}
  \hfill
  \begin{subfigure}[t]{0.33\textwidth}
    \centering
    \includegraphics[width=\linewidth]{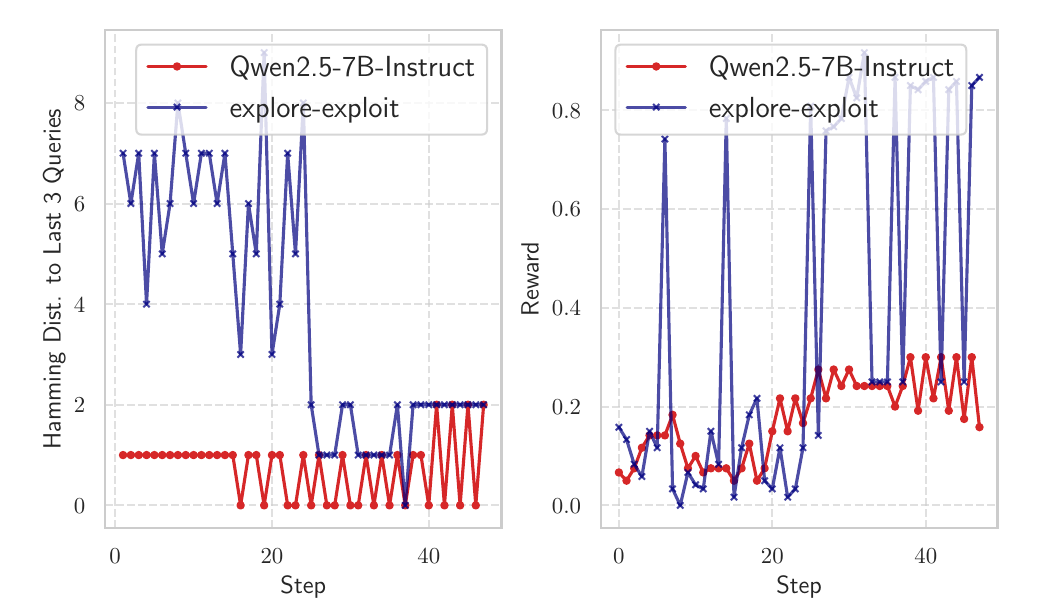}
  \end{subfigure}
  \hfill
  \begin{subfigure}[t]{0.33\textwidth}
    \centering
    \includegraphics[width=\linewidth]{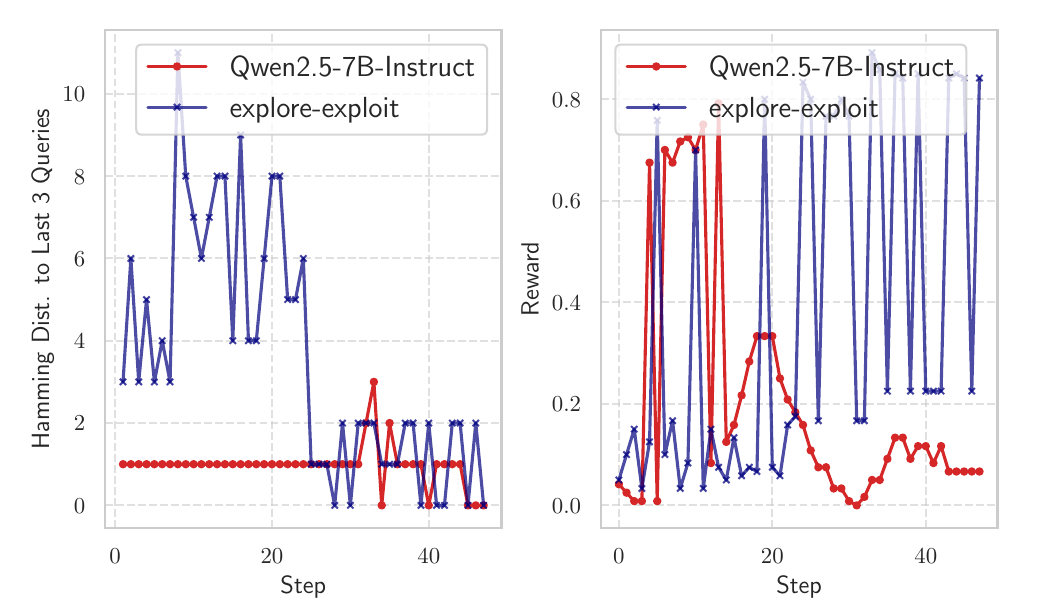}
  \end{subfigure}
  \hfill
  \begin{subfigure}[t]{0.33\textwidth}
    \centering
    \includegraphics[width=\linewidth]{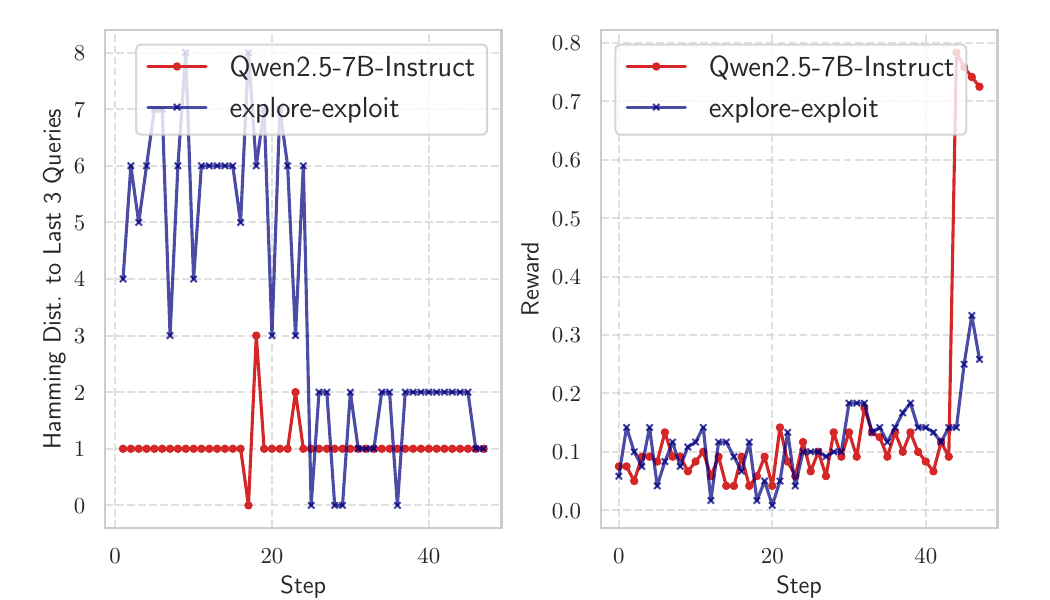}
  \end{subfigure}
  \hfill

  \caption{Visualization of 18 different episodes of \texttt{Qwen2.5-7B-Instruct} and the \texttt{explore-exploit} baseline on \tmaxsat{} with budget $N=48$. The LM often makes local changes with a single flip. In some episodes, the model refuses to make further queries when it finds a large reward, and keeps querying the same assignment. The baseline often finds a good assignment during the first half of its queries and applies local search for the rest of its queries, slightly increasing its reward.}
  \label{fig:interp:all:cnf}
\end{figure*}

\subsection{Evaluation With Different Budgets}
\label{app:experiment_full:diff-budgets}
We extend the evaluation from \S\ref{subsec:parllel} to different budgets. Table~\ref{tab:all_tasks_n_36} presents results for all tasks with $N=36$, and Table~\ref{tab:all_tasks_n_60} presents results for the \ttree{} task with $N=60$. These configurations cover the complete set of budgets listed in Table~\ref{tab:absolute_p1}. We use the same number of episodes that can be found in Table~\ref{tab:absolute_p1_runs}.

\begin{table*}[t]
\caption{Comparing parallel execution with various values of $p$ under total budget of $N=36$ for all tasks, on the same instance as Table~\ref{tab:all_tasks_n_48}.}
\label{tab:all_tasks_n_36}
\begin{adjustbox}{width=\linewidth}
\begin{tabular}{l|rrrr||rrrr||rrrr}
\toprule
 & \multicolumn{4}{c||}{\tsearch} & \multicolumn{4}{c||}{\ttree} & \multicolumn{4}{c}{\tmaxsat} \\
Model & $p=1$ & $p=2$ & $p=3$ & $p=4$ & $p=1$ & $p=2$ & $p=3$ & $p=4$ & $p=1$ & $p=2$ & $p=3$ & $p=4$ \\
\midrule
\midrule
{\small\texttt{Qwen2.5-7B-Instruct}} & 0.34 & 0.35 \updelta{3} & 0.61 \updelta{78} & \textbf{0.67 \updelta{95}} & 0.71 & \textbf{0.88 \updelta{25}} & 0.85 \updelta{21} & 0.71 \downdelta{0} & 0.38 & 0.48 \updelta{25} & \textbf{0.55 \updelta{43}} & 0.50 \updelta{31} \\
{\small\texttt{Qwen3-4B-Instruct}} & 0.60 & 0.77 \updelta{29} & 0.86 \updelta{44} & \textbf{0.89 \updelta{48}} & 0.75 & 0.79 \updelta{5} & \textbf{0.91 \updelta{21}} & 0.67 \downdelta{11} & 0.41 & 0.60 \updelta{45} & 0.60 \updelta{46} & \textbf{0.71 \updelta{72}} \\
{\small\texttt{Qwen3-8B}} & 0.52 & 0.94 \updelta{81} & \textbf{0.97 \updelta{87}} & 0.94 \updelta{80} & 0.64 & 0.77 \updelta{21} & \textbf{0.90 \updelta{40}} & 0.65 \updelta{2} & 0.52 & 0.63 \updelta{22} & 0.64 \updelta{23} & \textbf{0.67 \updelta{28}} \\
{\small\texttt{gemini-2.5-flash-lite}} & 0.83 & 0.96 \updelta{16} & \textbf{0.98 \updelta{18}} & 0.98 \updelta{18} & 0.64 & 0.77 \updelta{20} & \textbf{0.80 \updelta{26}} & 0.68 \updelta{6} & 0.54 & 0.62 \updelta{15} & 0.64 \updelta{19} & \textbf{0.67 \updelta{23}} \\
{\small\texttt{gpt-5-nano}} & 0.49 & 0.67 \updelta{37} & 0.71 \updelta{46} & \textbf{0.85 \updelta{74}} & 0.69 & 0.85 \updelta{23} & \textbf{0.85 \updelta{23}} & 0.70 \updelta{1} & 0.56 & 0.61 \updelta{8} & 0.54 \downdelta{3} & \textbf{0.68 \updelta{22}} \\
{\small\texttt{gpt-5-mini}} & 0.69 & 0.88 \updelta{28} & 0.84 \updelta{23} & \textbf{0.97 \updelta{42}} & 0.67 & 0.81 \updelta{21} & \textbf{0.90 \updelta{35}} & 0.68 \updelta{2} & 0.61 & 0.65 \updelta{7} & 0.65 \updelta{8} & \textbf{0.77 \updelta{26}} \\
{\small\texttt{gpt-5}} & 0.41 & \textbf{0.67 \updelta{65}} & 0.48 \updelta{18} & 0.64 \updelta{59} & 0.58 & 0.73 \updelta{25} & \textbf{0.93 \updelta{59}} & 0.65 \updelta{11} & 0.68 & 0.57 \downdelta{16} & \textbf{0.70 \updelta{3}} & 0.63 \downdelta{8} \\
\midrule
{\small\texttt{Qwen3-8B-medium}} & 0.74 & 0.86 \updelta{16} & \textbf{0.90 \updelta{22}} & 0.90 \updelta{22} & 0.72 & 0.87 \updelta{21} & \textbf{0.93 \updelta{29}} & 0.70 \downdelta{3} & 0.48 & 0.57 \updelta{21} & 0.52 \updelta{10} & \textbf{0.58 \updelta{21}} \\
{\small\texttt{gemini-2.5-flash}} & 0.58 & 0.68 \updelta{16} & 0.53 \downdelta{10} & \textbf{0.94 \updelta{61}} & 0.64 & \textbf{0.85 \updelta{33}} & 0.82 \updelta{28} & 0.71 \updelta{10} &  &  &  &  \\
{\small\texttt{gpt-5-nano-medium}} & 0.33 & 0.31 \downdelta{5} & 0.23 \downdelta{31} & \textbf{0.61 \updelta{87}} & 0.65 & 0.80 \updelta{23} & \textbf{0.89 \updelta{36}} & 0.69 \updelta{5} & 0.54 & \textbf{0.70 \updelta{30}} & 0.63 \updelta{17} & 0.45 \downdelta{16} \\
{\small\texttt{gpt-5-mini-medium}} & 0.57 & 0.96 \updelta{68} & 0.94 \updelta{64} & \textbf{0.98 \updelta{71}} & 0.53 & 0.61 \updelta{14} & \textbf{0.72 \updelta{36}} & 0.60 \updelta{13} & 0.59 & \textbf{0.69 \updelta{16}} & 0.55 \downdelta{8} & 0.56 \downdelta{6} \\
\midrule
{\small\texttt{explore-exploit}} & \multicolumn{4}{c||}{0.88} & \multicolumn{4}{c||}{0.94} & \multicolumn{4}{c}{0.77} \\
\bottomrule
\end{tabular}
\end{adjustbox}
\end{table*}

\begin{table}[t]
\centering
\begin{minipage}{0.8\linewidth}
\centering
\caption{Comparing parallel execution with various values of $p$ under the same total budget ($N=60$) for task \ttree{}, on the same instance as Table~\ref{tab:all_tasks_n_48}.}
\label{tab:all_tasks_n_60}
\begin{adjustbox}{width=0.7\linewidth}
\begin{tabular}{l|rrrr}
\toprule
 & \multicolumn{4}{c}{\ttree} \\
Model & $p=1$ & $p=2$ & $p=3$ & $p=4$ \\
\midrule
\midrule
{\small\texttt{Qwen2.5-7B-Instruct}} & 0.72 & 0.77 \updelta{7} & 0.88 \updelta{23} & \textbf{0.95 \updelta{31}} \\
{\small\texttt{Qwen3-4B-Instruct}} & 0.95 & 0.86 \downdelta{9} & 0.92 \downdelta{3} & \textbf{0.97 \updelta{2}} \\
{\small\texttt{Qwen3-8B}} & 0.68 & 0.87 \updelta{27} & 0.92 \updelta{35} & \textbf{0.94 \updelta{37}} \\
{\small\texttt{gemini-2.5-flash-lite}} & 0.76 & 0.81 \updelta{6} & 0.77 \updelta{1} & \textbf{0.89 \updelta{17}} \\
{\small\texttt{gpt-5-nano}} & 0.79 & 0.85 \updelta{7} & 0.88 \updelta{11} & \textbf{0.95 \updelta{20}} \\
{\small\texttt{gpt-5-mini}} & 0.73 & 0.85 \updelta{17} & 0.95 \updelta{30} & \textbf{0.96 \updelta{31}} \\
{\small\texttt{gpt-5}} & 0.77 & 0.76 \downdelta{1} & 0.79 \updelta{2} & \textbf{0.95 \updelta{23}} \\
\midrule
{\small\texttt{Qwen3-8B-medium}} & 0.85 & 0.91 \updelta{7} & 0.89 \updelta{4} & \textbf{0.97 \updelta{14}} \\
{\small\texttt{gemini-2.5-flash}} & 0.68 & 0.85 \updelta{24} & 0.91 \updelta{33} & \textbf{0.98 \updelta{43}} \\
{\small\texttt{gpt-5-nano-medium}} & 0.82 & 0.85 \updelta{4} & 0.86 \updelta{4} & \textbf{0.94 \updelta{15}} \\
{\small\texttt{gpt-5-mini-medium}} & 0.68 & 0.68 \updelta{0} & 0.86 \updelta{27} & \textbf{0.86 \updelta{27}} \\
\midrule
{\small\texttt{explore-exploit}} & \multicolumn{4}{c}{0.97} \\
\bottomrule
\end{tabular}
\end{adjustbox}
\end{minipage}
\end{table}

\subsection{Additional Instances}
\label{app:experiment_full:add-instances}
We validate the findings from \S\ref{subsec:eval} by repeating the experiments on additional problem instances. We use the same number of episodes that can be found in Table~\ref{tab:absolute_p1_runs}.

\paragraph{\tsearch{}}
We evaluate a second instance containing decoy hills centered at 
$[1.28, 2.65, 3.83, 5.11, 6.44, 7.6, 8.77]$, 
with widths 
$[0.1, 0.1, 0.1, 0.1, 0.1, 0.1, 0.1]$ 
and heights 
$[1, 2, 4, 1, 2, 3, 4]$. 
The needle hill is centered at $6.2$ with height $20$ and width $0.01$. 
Table~\ref{tab:hill_instance_3_merged} presents the results for this instance.

\begin{table*}[t]
\centering
\caption{Task \tsearch. Additional Instance 1. Results across $N$.}
\label{tab:hill_instance_3_merged}
\begin{adjustbox}{width=0.8\linewidth}
\begin{tabular}{l|rrrr|rrrr}
\toprule
 & \multicolumn{4}{c|}{$N=36$} & \multicolumn{4}{c}{$N=48$} \\
 & $p=1$ & $p=2$ & $p=3$ & $p=4$ & $p=1$ & $p=2$ & $p=3$ & $p=4$ \\
\midrule
\midrule
{\small\texttt{Qwen2.5-7B-Instruct}} & 0.34 & 0.35 \updelta{3} & 0.61 \updelta{78} & \textbf{0.67 \updelta{95}} & 0.30 & 0.46 \updelta{57} & 0.61 \updelta{107} & \textbf{0.65 \updelta{121}} \\
{\small\texttt{Qwen3-4B-Instruct}} & 0.60 & 0.77 \updelta{29} & 0.86 \updelta{44} & \textbf{0.89 \updelta{48}} & 0.59 & 0.83 \updelta{40} & \textbf{0.90 \updelta{51}} & 0.88 \updelta{48} \\
{\small\texttt{Qwen3-8B}} & 0.52 & 0.94 \updelta{81} & \textbf{0.97 \updelta{87}} & 0.94 \updelta{80} & 0.58 & 0.71 \updelta{24} & 0.95 \updelta{65} & \textbf{0.96 \updelta{66}} \\
{\small\texttt{gemini-2.5-flash-lite}} & 0.83 & 0.96 \updelta{16} & \textbf{0.98 \updelta{18}} & 0.98 \updelta{18} & 0.85 & 0.93 \updelta{10} & \textbf{0.99 \updelta{17}} & 0.99 \updelta{17} \\
{\small\texttt{gpt-5-nano}} & 0.49 & 0.67 \updelta{37} & 0.71 \updelta{46} & \textbf{0.85 \updelta{74}} & 0.49 & 0.59 \updelta{20} & \textbf{0.77 \updelta{58}} & 0.63 \updelta{29} \\
{\small\texttt{gpt-5-mini}} & 0.69 & 0.88 \updelta{28} & 0.84 \updelta{23} & \textbf{0.97 \updelta{42}} & 0.56 & 0.80 \updelta{42} & \textbf{0.97 \updelta{73}} & 0.83 \updelta{47} \\
{\small\texttt{gpt-5}} & 0.41 & \textbf{0.67 \updelta{65}} & 0.48 \updelta{18} & 0.64 \updelta{59} & 0.63 & 0.69 \updelta{9} & \textbf{0.78 \updelta{23}} & 0.54 \downdelta{14} \\
\midrule
{\small\texttt{Qwen3-8B-medium}} & 0.74 & 0.86 \updelta{16} & \textbf{0.90 \updelta{22}} & 0.90 \updelta{22} & 0.76 & 0.86 \updelta{13} & 0.89 \updelta{17} & \textbf{0.93 \updelta{22}} \\
{\small\texttt{gpt-5-nano-medium}} & 0.33 & 0.31 \downdelta{5} & 0.23 \downdelta{31} & \textbf{0.61 \updelta{87}} & 0.34 & 0.48 \updelta{43} & \textbf{0.57 \updelta{69}} & 0.30 \downdelta{13} \\
{\small\texttt{gpt-5-mini-medium}} & 0.57 & 0.96 \updelta{68} & 0.94 \updelta{64} & \textbf{0.98 \updelta{71}} & 0.60 & 0.88 \updelta{46} & \textbf{0.96 \updelta{59}} & 0.94 \updelta{55} \\
\midrule
{\small\texttt{explore-exploit}} & \multicolumn{4}{c|}{0.88} & \multicolumn{4}{c}{0.89} \\
\bottomrule
\end{tabular}
\end{adjustbox}
\end{table*}

\paragraph{\ttree}
We generate two additional instances following the procedure described in Appendix~\ref{app:subsec:task}, using the following parameters:
\begin{enumerate}
    \item $r_\trap = r_\good = 2$, $b = 3$, $d_\trap = 40$, and $d_\good = 14$. This yields a tree with 323 nodes and a maximum value of 53. Results are shown in Table~\ref{tab:graph_instance_1_merged}.
    \item $r_\trap = r_\good = 4$, $b = 4$, $d_\trap = 40$, and $d_\good = 16$. This yields a tree with 889 nodes and a maximum value of 61. Results are shown in Table~\ref{tab:graph_instance_2_merged}.
\end{enumerate}

\begin{table*}[t]
\caption{Task \ttree. Additional Instance 1. Results across $N$.}
\label{tab:graph_instance_1_merged}
\begin{adjustbox}{width=\linewidth}
\begin{tabular}{l|rrrr|rrrr|rrrr}
\toprule
 & \multicolumn{4}{c|}{$N=36$} & \multicolumn{4}{c|}{$N=48$} & \multicolumn{4}{c}{$N=60$} \\
 & $p=1$ & $p=2$ & $p=3$ & $p=4$ & $p=1$ & $p=2$ & $p=3$ & $p=4$ & $p=1$ & $p=2$ & $p=3$ & $p=4$ \\
\midrule
\midrule
{\small\texttt{Qwen2.5-7B-Instruct}} & 0.65 & \textbf{0.87 \updelta{34}} & 0.76 \updelta{17} & 0.61 \downdelta{6} & 0.65 & 0.83 \updelta{28} & \textbf{0.92 \updelta{42}} & 0.79 \updelta{22} & 0.69 & 0.85 \updelta{22} & 0.96 \updelta{38} & \textbf{0.96 \updelta{39}} \\
{\small\texttt{Qwen3-4B-Instruct}} & 0.71 & \textbf{0.85 \updelta{20}} & 0.76 \updelta{8} & 0.59 \downdelta{16} & 0.81 & 0.81 \updelta{0} & \textbf{0.91 \updelta{13}} & 0.83 \updelta{3} & 0.85 & 0.79 \downdelta{8} & 0.89 \updelta{5} & \textbf{0.91 \updelta{7}} \\
{\small\texttt{Qwen3-8B}} & 0.68 & \textbf{0.84 \updelta{23}} & 0.78 \updelta{13} & 0.59 \downdelta{14} & 0.71 & 0.81 \updelta{14} & \textbf{0.91 \updelta{28}} & 0.79 \updelta{11} & 0.72 & 0.80 \updelta{11} & 0.84 \updelta{17} & \textbf{0.97 \updelta{36}} \\
{\small\texttt{gemini-2.5-flash-lite}} & 0.62 & \textbf{0.80 \updelta{29}} & 0.72 \updelta{15} & 0.56 \downdelta{9} & 0.65 & 0.73 \updelta{12} & \textbf{0.90 \updelta{38}} & 0.76 \updelta{17} & 0.69 & 0.88 \updelta{28} & 0.88 \updelta{28} & \textbf{0.93 \updelta{35}} \\
{\small\texttt{gpt-5-nano}} & 0.70 & \textbf{0.80 \updelta{15}} & 0.77 \updelta{10} & 0.59 \downdelta{15} & 0.66 & 0.84 \updelta{27} & \textbf{0.89 \updelta{35}} & 0.82 \updelta{24} & 0.73 & 0.87 \updelta{19} & 0.90 \updelta{23} & \textbf{0.97 \updelta{33}} \\
{\small\texttt{gpt-5-mini}} & 0.60 & \textbf{0.80 \updelta{32}} & 0.78 \updelta{29} & 0.60 \downdelta{1} & 0.71 & 0.74 \updelta{4} & \textbf{0.89 \updelta{25}} & 0.80 \updelta{13} & 0.79 & 0.81 \updelta{4} & 0.90 \updelta{15} & \textbf{0.93 \updelta{19}} \\
{\small\texttt{gpt-5}} & \textbf{0.67} & 0.66 \downdelta{2} & 0.64 \downdelta{4} & 0.54 \downdelta{19} & 0.60 & 0.53 \downdelta{11} & \textbf{0.83 \updelta{39}} & 0.78 \updelta{30} & 0.64 & 0.77 \updelta{20} & 0.88 \updelta{37} & \textbf{0.90 \updelta{40}} \\
\midrule
{\small\texttt{Qwen3-8B-medium}} & 0.58 & \textbf{0.81 \updelta{38}} & 0.78 \updelta{33} & 0.60 \updelta{2} & 0.67 & 0.74 \updelta{11} & \textbf{0.81 \updelta{21}} & 0.80 \updelta{20} & 0.70 & 0.80 \updelta{15} & 0.92 \updelta{31} & \textbf{0.93 \updelta{33}} \\
{\small\texttt{gpt-5-nano-medium}} & 0.65 & \textbf{0.89 \updelta{37}} & 0.75 \updelta{16} & 0.57 \downdelta{13} & 0.63 & 0.78 \updelta{24} & \textbf{0.89 \updelta{42}} & 0.80 \updelta{28} & 0.78 & 0.71 \downdelta{9} & 0.90 \updelta{16} & \textbf{0.93 \updelta{19}} \\
{\small\texttt{gpt-5-mini-medium}} & 0.51 & \textbf{0.68 \updelta{33}} & 0.68 \updelta{32} & 0.50 \downdelta{3} & 0.51 & 0.71 \updelta{38} & \textbf{0.75 \updelta{47}} & 0.67 \updelta{31} & 0.69 & 0.58 \downdelta{16} & 0.71 \updelta{3} & \textbf{0.74 \updelta{7}} \\
\midrule
{\small\texttt{explore-exploit}} & \multicolumn{4}{c|}{0.93} & \multicolumn{4}{c|}{0.96} & \multicolumn{4}{c}{0.98} \\
\bottomrule
\end{tabular}
\end{adjustbox}
\end{table*}

\begin{table*}[t]
\caption{Task \ttree. Additional Instance 2. Results across $N$.}
\label{tab:graph_instance_2_merged}
\begin{adjustbox}{width=\linewidth}
\begin{tabular}{l|rrrr|rrrr|rrrr}
\toprule
 & \multicolumn{4}{c|}{$N=36$} & \multicolumn{4}{c|}{$N=48$} & \multicolumn{4}{c}{$N=60$} \\
 & $p=1$ & $p=2$ & $p=3$ & $p=4$ & $p=1$ & $p=2$ & $p=3$ & $p=4$ & $p=1$ & $p=2$ & $p=3$ & $p=4$ \\
\midrule
\midrule
{\small\texttt{Qwen2.5-7B-Instruct}} & 0.62 & \textbf{0.83 \updelta{34}} & 0.67 \updelta{7} & 0.48 \downdelta{22} & 0.67 & 0.85 \updelta{26} & \textbf{0.92 \updelta{37}} & 0.69 \updelta{2} & 0.65 & 0.84 \updelta{29} & \textbf{0.93 \updelta{42}} & 0.84 \updelta{29} \\
{\small\texttt{Qwen3-4B-Instruct}} & \textbf{0.82} & 0.80 \downdelta{3} & 0.63 \downdelta{23} & 0.53 \downdelta{36} & \textbf{0.89} & 0.81 \downdelta{8} & 0.87 \downdelta{2} & 0.72 \downdelta{19} & \textbf{0.93} & 0.85 \downdelta{8} & 0.91 \downdelta{2} & 0.88 \downdelta{5} \\
{\small\texttt{Qwen3-8B}} & 0.74 & \textbf{0.74 \updelta{0}} & 0.69 \downdelta{7} & 0.52 \downdelta{30} & 0.66 & 0.85 \updelta{30} & \textbf{0.86 \updelta{31}} & 0.71 \updelta{8} & 0.73 & 0.74 \updelta{1} & 0.89 \updelta{21} & \textbf{0.90 \updelta{23}} \\
{\small\texttt{gemini-2.5-flash-lite}} & 0.57 & \textbf{0.82 \updelta{43}} & 0.65 \updelta{14} & 0.49 \downdelta{14} & 0.63 & 0.80 \updelta{26} & \textbf{0.83 \updelta{32}} & 0.67 \updelta{6} & 0.67 & \textbf{0.86 \updelta{27}} & 0.82 \updelta{22} & 0.83 \updelta{24} \\
{\small\texttt{gpt-5-nano}} & 0.67 & \textbf{0.81 \updelta{22}} & 0.68 \updelta{3} & 0.52 \downdelta{22} & 0.63 & 0.76 \updelta{20} & \textbf{0.84 \updelta{33}} & 0.69 \updelta{9} & 0.78 & 0.79 \updelta{1} & 0.84 \updelta{8} & \textbf{0.84 \updelta{8}} \\
{\small\texttt{gpt-5-mini}} & 0.66 & \textbf{0.79 \updelta{19}} & 0.69 \updelta{4} & 0.51 \downdelta{23} & 0.65 & \textbf{0.87 \updelta{34}} & 0.86 \updelta{33} & 0.68 \updelta{5} & 0.75 & 0.82 \updelta{11} & 0.84 \updelta{12} & \textbf{0.86 \updelta{15}} \\
{\small\texttt{gpt-5}} & 0.61 & \textbf{0.75 \updelta{23}} & 0.53 \downdelta{13} & 0.51 \downdelta{16} & 0.57 & 0.67 \updelta{19} & \textbf{0.74 \updelta{31}} & 0.64 \updelta{13} & 0.74 & 0.71 \downdelta{4} & \textbf{0.89 \updelta{20}} & 0.84 \updelta{13} \\
\midrule
{\small\texttt{Qwen3-8B-medium}} & 0.52 & \textbf{0.77 \updelta{48}} & 0.70 \updelta{35} & 0.50 \downdelta{2} & 0.62 & 0.80 \updelta{29} & \textbf{0.89 \updelta{43}} & 0.69 \updelta{11} & 0.74 & 0.83 \updelta{12} & \textbf{0.94 \updelta{27}} & 0.89 \updelta{21} \\
{\small\texttt{gpt-5-nano-medium}} & 0.56 & \textbf{0.79 \updelta{40}} & 0.67 \updelta{20} & 0.51 \downdelta{9} & 0.62 & 0.86 \updelta{40} & \textbf{0.91 \updelta{47}} & 0.69 \updelta{13} & 0.71 & 0.77 \updelta{9} & \textbf{0.88 \updelta{25}} & 0.87 \updelta{23} \\
{\small\texttt{gpt-5-mini-medium}} & 0.49 & \textbf{0.55 \updelta{14}} & 0.46 \downdelta{5} & 0.44 \downdelta{9} & 0.41 & 0.65 \updelta{60} & \textbf{0.78 \updelta{91}} & 0.64 \updelta{57} & 0.76 & 0.67 \downdelta{12} & 0.62 \downdelta{19} & \textbf{0.85 \updelta{12}} \\
\midrule
{\small\texttt{explore-exploit}} & \multicolumn{4}{c|}{0.89} & \multicolumn{4}{c|}{0.98} & \multicolumn{4}{c}{0.99} \\
\bottomrule
\end{tabular}
\end{adjustbox}
\end{table*}

\paragraph{\tmaxsat}
We evaluate three further instances generated via the method in Appendix~\ref{app:subsec:task}:
\begin{enumerate}
    \item $n=15$, $m=135$, $k_\gold = 4$, $k_\other=2$, and $w_\gold = 90$. See Table~\ref{tab:max_sat_instance_1_merged} for results.
    \item $n=15$, $m=150$, $k_\gold = 4$, $k_\other=2$, and $w_\gold = 100$. See Table~\ref{tab:max_sat_instance_2_merged} for results.
    \item $n=15$, $m=165$, $k_\gold = 4$, $k_\other=2$, and $w_\gold = 110$. See Table~\ref{tab:max_sat_instance_3_merged} for results.
\end{enumerate}

\begin{table*}[t]
\centering
\caption{Task \tmaxsat. Additional Instance 1. Results across $N$.}
\label{tab:max_sat_instance_1_merged}
\begin{adjustbox}{width=0.8\linewidth}
\begin{tabular}{l|rrrr|rrrr}
\toprule
 & \multicolumn{4}{c|}{$N=36$} & \multicolumn{4}{c}{$N=48$} \\
 & $p=1$ & $p=2$ & $p=3$ & $p=4$ & $p=1$ & $p=2$ & $p=3$ & $p=4$ \\
\midrule
\midrule
{\small\texttt{Qwen2.5-7B-Instruct}} & 0.42 & \textbf{0.58 \updelta{40}} & 0.55 \updelta{33} & 0.56 \updelta{35} & 0.45 & 0.60 \updelta{33} & \textbf{0.62 \updelta{37}} & 0.61 \updelta{37} \\
{\small\texttt{Qwen3-4B-Instruct}} & 0.42 & 0.50 \updelta{19} & 0.58 \updelta{37} & \textbf{0.74 \updelta{75}} & 0.36 & 0.53 \updelta{47} & \textbf{0.69 \updelta{91}} & 0.68 \updelta{88} \\
{\small\texttt{Qwen3-8B}} & 0.60 & 0.61 \updelta{2} & 0.72 \updelta{21} & \textbf{0.73 \updelta{22}} & 0.53 & \textbf{0.73 \updelta{37}} & 0.69 \updelta{30} & 0.68 \updelta{28} \\
{\small\texttt{gemini-2.5-flash-lite}} & 0.52 & \textbf{0.66 \updelta{28}} & 0.63 \updelta{21} & 0.65 \updelta{26} & 0.58 & 0.66 \updelta{14} & 0.65 \updelta{11} & \textbf{0.70 \updelta{20}} \\
{\small\texttt{gpt-5-nano}} & 0.56 & 0.62 \updelta{10} & \textbf{0.63 \updelta{12}} & 0.60 \updelta{7} & 0.57 & 0.69 \updelta{22} & 0.68 \updelta{19} & \textbf{0.70 \updelta{23}} \\
{\small\texttt{gpt-5-mini}} & 0.61 & \textbf{0.72 \updelta{17}} & 0.65 \updelta{7} & 0.65 \updelta{5} & 0.72 & 0.73 \updelta{1} & \textbf{0.76 \updelta{6}} & 0.74 \updelta{2} \\
{\small\texttt{gpt-5}} & \textbf{0.83} & 0.75 \downdelta{9} & 0.69 \downdelta{17} & 0.64 \downdelta{23} & 0.76 & \textbf{0.90 \updelta{18}} & 0.82 \updelta{7} & 0.66 \downdelta{14} \\
\midrule
{\small\texttt{Qwen3-8B-medium}} & 0.48 & 0.66 \updelta{38} & \textbf{0.71 \updelta{49}} & 0.64 \updelta{34} & 0.53 & 0.62 \updelta{16} & \textbf{0.77 \updelta{45}} & 0.71 \updelta{33} \\
{\small\texttt{gpt-5-nano-medium}} & 0.47 & 0.59 \updelta{26} & 0.56 \updelta{19} & \textbf{0.65 \updelta{39}} & 0.52 & 0.58 \updelta{11} & \textbf{0.65 \updelta{25}} & 0.61 \updelta{17} \\
{\small\texttt{gpt-5-mini-medium}} & \textbf{0.64} & 0.59 \downdelta{8} & 0.55 \downdelta{14} & 0.63 \downdelta{3} & 0.70 & \textbf{0.81 \updelta{16}} & 0.80 \updelta{15} & 0.75 \updelta{8} \\
\midrule
{\small\texttt{explore-exploit}} & \multicolumn{4}{c|}{0.74} & \multicolumn{4}{c}{0.83} \\
\bottomrule
\end{tabular}
\end{adjustbox}
\end{table*}

\begin{table*}[t]
\centering
\caption{Task \tmaxsat. Additional Instance 2. Results across $N$.}
\label{tab:max_sat_instance_2_merged}
\begin{adjustbox}{width=0.8\linewidth}
\begin{tabular}{l|rrrr|rrrr}
\toprule
 & \multicolumn{4}{c|}{$N=36$} & \multicolumn{4}{c}{$N=48$} \\
 & $p=1$ & $p=2$ & $p=3$ & $p=4$ & $p=1$ & $p=2$ & $p=3$ & $p=4$ \\
\midrule
\midrule
{\small\texttt{Qwen2.5-7B-Instruct}} & 0.44 & 0.47 \updelta{7} & 0.56 \updelta{27} & \textbf{0.57 \updelta{31}} & 0.50 & 0.58 \updelta{17} & \textbf{0.67 \updelta{35}} & 0.66 \updelta{33} \\
{\small\texttt{Qwen3-4B-Instruct}} & 0.33 & 0.53 \updelta{61} & 0.57 \updelta{72} & \textbf{0.64 \updelta{94}} & 0.43 & 0.58 \updelta{34} & 0.63 \updelta{45} & \textbf{0.69 \updelta{58}} \\
{\small\texttt{Qwen3-8B}} & 0.56 & 0.64 \updelta{13} & 0.55 \downdelta{3} & \textbf{0.67 \updelta{18}} & 0.42 & 0.68 \updelta{61} & \textbf{0.71 \updelta{69}} & 0.70 \updelta{68} \\
{\small\texttt{gemini-2.5-flash-lite}} & 0.57 & 0.59 \updelta{3} & \textbf{0.67 \updelta{17}} & 0.66 \updelta{17} & 0.54 & 0.66 \updelta{21} & \textbf{0.75 \updelta{38}} & 0.67 \updelta{24} \\
{\small\texttt{gpt-5-nano}} & 0.53 & 0.61 \updelta{14} & \textbf{0.61 \updelta{16}} & 0.60 \updelta{14} & 0.51 & 0.59 \updelta{16} & 0.64 \updelta{27} & \textbf{0.72 \updelta{42}} \\
{\small\texttt{gpt-5-mini}} & 0.64 & 0.64 \updelta{0} & \textbf{0.79 \updelta{23}} & 0.73 \updelta{15} & 0.72 & 0.75 \updelta{3} & \textbf{0.80 \updelta{10}} & 0.71 \downdelta{2} \\
{\small\texttt{gpt-5}} & 0.65 & \textbf{0.73 \updelta{12}} & 0.72 \updelta{12} & 0.72 \updelta{11} & 0.73 & 0.77 \updelta{6} & \textbf{0.83 \updelta{14}} & 0.70 \downdelta{4} \\
\midrule
{\small\texttt{Qwen3-8B-medium}} & 0.53 & 0.57 \updelta{8} & 0.56 \updelta{6} & \textbf{0.64 \updelta{23}} & 0.48 & \textbf{0.76 \updelta{58}} & 0.68 \updelta{41} & 0.66 \updelta{36} \\
{\small\texttt{gpt-5-nano-medium}} & 0.49 & 0.48 \downdelta{3} & 0.36 \downdelta{28} & \textbf{0.65 \updelta{31}} & \textbf{0.59} & 0.56 \downdelta{5} & 0.56 \downdelta{4} & 0.47 \downdelta{21} \\
{\small\texttt{gpt-5-mini-medium}} & 0.69 & \textbf{0.71 \updelta{3}} & 0.65 \downdelta{6} & 0.59 \downdelta{14} & 0.71 & 0.58 \downdelta{19} & \textbf{0.74 \updelta{4}} & 0.59 \downdelta{17} \\
\midrule
{\small\texttt{explore-exploit}} & \multicolumn{4}{c|}{0.73} & \multicolumn{4}{c}{0.84} \\
\bottomrule
\end{tabular}
\end{adjustbox}
\end{table*}

\begin{table*}[t]
\centering
\caption{Task \tmaxsat. Additional Instance 3. Results across $N$.}
\label{tab:max_sat_instance_3_merged}
\begin{adjustbox}{width=0.8\linewidth}
\begin{tabular}{l|rrrr|rrrr}
\toprule
 & \multicolumn{4}{c|}{$N=36$} & \multicolumn{4}{c}{$N=48$} \\
 & $p=1$ & $p=2$ & $p=3$ & $p=4$ & $p=1$ & $p=2$ & $p=3$ & $p=4$ \\
\midrule
\midrule
{\small\texttt{Qwen2.5-7B-Instruct}} & 0.48 & 0.51 \updelta{6} & \textbf{0.61 \updelta{26}} & 0.61 \updelta{25} & 0.49 & 0.54 \updelta{9} & \textbf{0.63 \updelta{28}} & 0.61 \updelta{24} \\
{\small\texttt{Qwen3-4B-Instruct}} & 0.30 & 0.55 \updelta{80} & 0.56 \updelta{84} & \textbf{0.68 \updelta{123}} & 0.35 & 0.57 \updelta{61} & 0.60 \updelta{71} & \textbf{0.64 \updelta{83}} \\
{\small\texttt{Qwen3-8B}} & 0.45 & 0.67 \updelta{48} & 0.68 \updelta{50} & \textbf{0.70 \updelta{55}} & 0.52 & 0.70 \updelta{36} & \textbf{0.71 \updelta{38}} & 0.67 \updelta{30} \\
{\small\texttt{gemini-2.5-flash-lite}} & 0.60 & 0.62 \updelta{4} & 0.56 \downdelta{6} & \textbf{0.62 \updelta{5}} & 0.59 & 0.57 \downdelta{3} & \textbf{0.68 \updelta{16}} & 0.62 \updelta{4} \\
{\small\texttt{gpt-5-nano}} & 0.50 & 0.55 \updelta{11} & \textbf{0.62 \updelta{24}} & 0.59 \updelta{18} & 0.53 & 0.69 \updelta{29} & 0.66 \updelta{25} & \textbf{0.69 \updelta{30}} \\
{\small\texttt{gpt-5-mini}} & 0.67 & 0.68 \updelta{0} & 0.64 \downdelta{4} & \textbf{0.70 \updelta{4}} & 0.63 & 0.65 \updelta{3} & \textbf{0.76 \updelta{19}} & 0.72 \updelta{14} \\
{\small\texttt{gpt-5}} & 0.74 & \textbf{0.74 \updelta{0}} & 0.64 \downdelta{13} & 0.51 \downdelta{32} & 0.71 & 0.70 \downdelta{2} & \textbf{0.77 \updelta{8}} & 0.63 \downdelta{11} \\
\midrule
{\small\texttt{Qwen3-8B-medium}} & 0.51 & 0.58 \updelta{13} & 0.59 \updelta{15} & \textbf{0.68 \updelta{33}} & 0.56 & 0.67 \updelta{19} & \textbf{0.68 \updelta{21}} & 0.52 \downdelta{7} \\
{\small\texttt{gpt-5-nano-medium}} & 0.50 & \textbf{0.61 \updelta{23}} & 0.60 \updelta{21} & 0.59 \updelta{20} & 0.53 & 0.51 \downdelta{4} & 0.50 \downdelta{5} & \textbf{0.59 \updelta{13}} \\
{\small\texttt{gpt-5-mini-medium}} & 0.62 & 0.62 \downdelta{0} & 0.53 \downdelta{16} & \textbf{0.63 \updelta{1}} & \textbf{0.70} & 0.66 \downdelta{6} & 0.68 \downdelta{2} & 0.68 \downdelta{2} \\
\midrule
{\small\texttt{explore-exploit}} & \multicolumn{4}{c|}{0.76} & \multicolumn{4}{c}{0.81} \\
\bottomrule
\end{tabular}
\end{adjustbox}
\end{table*}

\subsection{Evaluating on 50 More Instances}
\label{app:fifty_instance}
We extend our analysis from specific cases to a larger dataset to check the consistency of our results. We evaluate \texttt{Qwen2.5-7B-Instruct} and present the outcomes in Table~\ref{tab:ds-full:all_tasks_n48}. The model generally underperforms compared to the \texttt{explore-exploit} baselines, although summarization and parallelization improve performance.

\begin{table}[t]
\centering
\begin{minipage}{0.8\linewidth}
\caption{Performance comparison on the general dataset with budget $N=48$. We report the average reward using summarization ($s$) and parallelization ($p$) strategies. Relative improvements over the base model are shown in parentheses.}
\label{tab:ds-full:all_tasks_n48}
\centering
\begin{adjustbox}{width=\linewidth}
\begin{tabular}{cc|c||c||c}
\toprule
\multicolumn{2}{c|}{Method} & \tsearch & \ttree & \tmaxsat \\
\midrule
\multicolumn{2}{c|}{\texttt{Qwen2.5-7B-Instruct}} & 0.23 & 0.66 & 0.42 \\
\midrule
\multirow{3}{*}{summary} & $s=2$ & 0.30 \updelta{26} & 0.76 \updelta{15} & 0.50 \updelta{20} \\
 & $s=3$ & 0.32 \updelta{38} & 0.78 \updelta{18} & 0.54 \updelta{29} \\
 & $s=4$ & \textbf{0.36 \updelta{55}} & \textbf{0.81 \updelta{22}} & \textbf{0.57 \updelta{35}} \\
\midrule
\multirow{3}{*}{parallel} & $p=2$ & 0.30 \updelta{29} & 0.78 \updelta{17} & 0.54 \updelta{28} \\
 & $p=3$ & 0.37 \updelta{57} & \textbf{0.86 \updelta{30}} & 0.57 \updelta{36} \\
 & $p=4$ & \textbf{0.37 \updelta{59}} & 0.84 \updelta{27} & \textbf{0.59 \updelta{41}} \\
\midrule
\multicolumn{2}{c|}{\texttt{explore-exploit baseline}} & 0.75 & 0.92 & 0.77 \\
\bottomrule
\end{tabular}
\end{adjustbox}
\end{minipage}
\end{table}

For \tsearch{}, we generate 50 instances using the procedure in Appendix~\ref{app:subsec:task}. We sample parameters uniformly at random, with $k \in \{2,3,4\}$ and $k' \in \{k + 1, k + 2\}$. We fix the remaining parameters: $\alpha_\dec = 0.01$, $\alpha_\need = 0.008$, $j_\dec = 0.1$, and $j_\need = 0.2$.

For \ttree{}, we generate 50 instances uniformly at random following Appendix~\ref{app:subsec:task}. We use the parameters $r_\trap \in [1,4]$, $r_\good \in [1,4]$, $b \in [3,5]$, $d_\trap \in [20,40]$, and $d_\good \in [10,16]$.

For \tmaxsat{}, we generate 50 instances as defined in Appendix~\ref{app:subsec:task}. We sample parameters uniformly: $n \in [12, 24]$, $k_\gold \in [3, 4]$, $k_\other \in [2,6]$, and $w_\gold \in [60, 160]$. We set $m = w_\gold + m_\other$, where $m_\other \in [30, 80]$.

For all tasks, we run the baseline for 200 episodes per instance. We run the model for 50 episodes per instance for each setting of $s$ and $p$.

\subsection{Full Results of Difficulty Variation}
\label{app:task_scale_difficulty}

We generate task instances with varying difficulty levels based on the parameters defined in Appendix~\ref{app:subsec:task}.

For \tsearch{}, we vary $k'$ from $2$ to $8$, with $k = k'-1$. We set the remaining parameters as follows: $\alpha_\dec = 0.01$, $\alpha_\need = 0.008$, $j_\dec = 0.1$, and $j_\need = 0.2$.

For \ttree{}, we vary $r_\good$ from $1$ to $7$. We fix the other parameters at $r_\trap = 4$, $b = 3$, $d_\trap = 40$, and $d_\good = 12$.

For \tmaxsat{}, we vary $k_\gold$ from $1$ to $7$. We fix the number of non-gold variables at $11$ (by setting $n = 11 + k_\gold$) and the number of non-gold clauses at $50$ (by setting $m = 50 + w_\gold$, where $w_\gold = 130 + 10k_\gold$). We also fix $k_\other = 2$.

Figure~\ref{fig:task_difficulty_n36_details} presents the complete evaluation results with budget $N=36$, covering both intervention methods: parallel ($p \in \{2,3,4\}$) and summary ($s \in \{2,3,4,6\}$).

\begin{figure*}[t]
  \centering
  \begin{subfigure}[t]{0.32\textwidth}
    \centering
    \includegraphics[width=\linewidth]{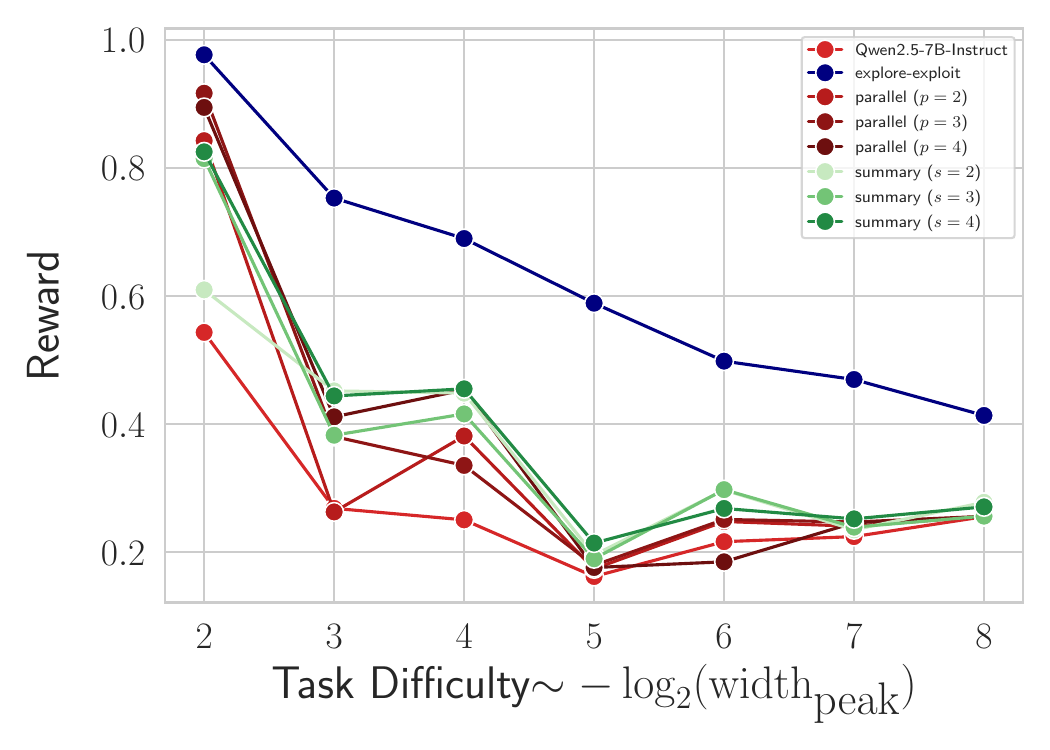}
    \caption{\tsearch}
  \end{subfigure}\hfill
  \begin{subfigure}[t]{0.32\textwidth}
    \centering
    \includegraphics[width=\linewidth]{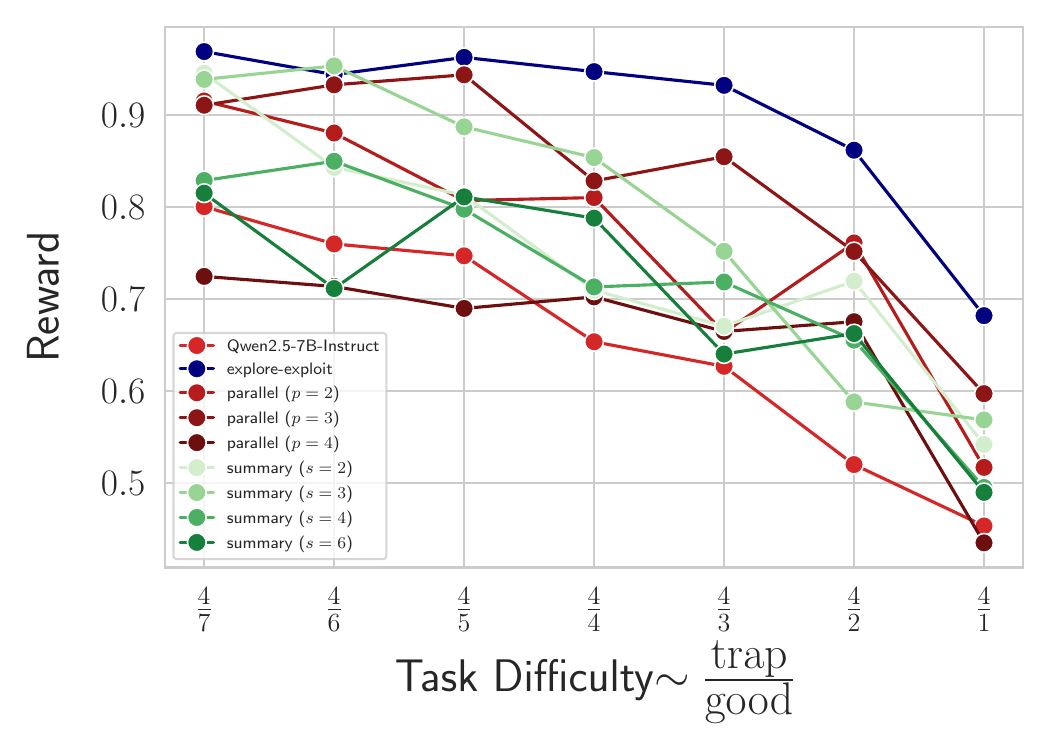}
    \caption{\ttree} 
  \end{subfigure}\hfill
  \begin{subfigure}[t]{0.32\textwidth}
    \centering
    \includegraphics[width=\linewidth]{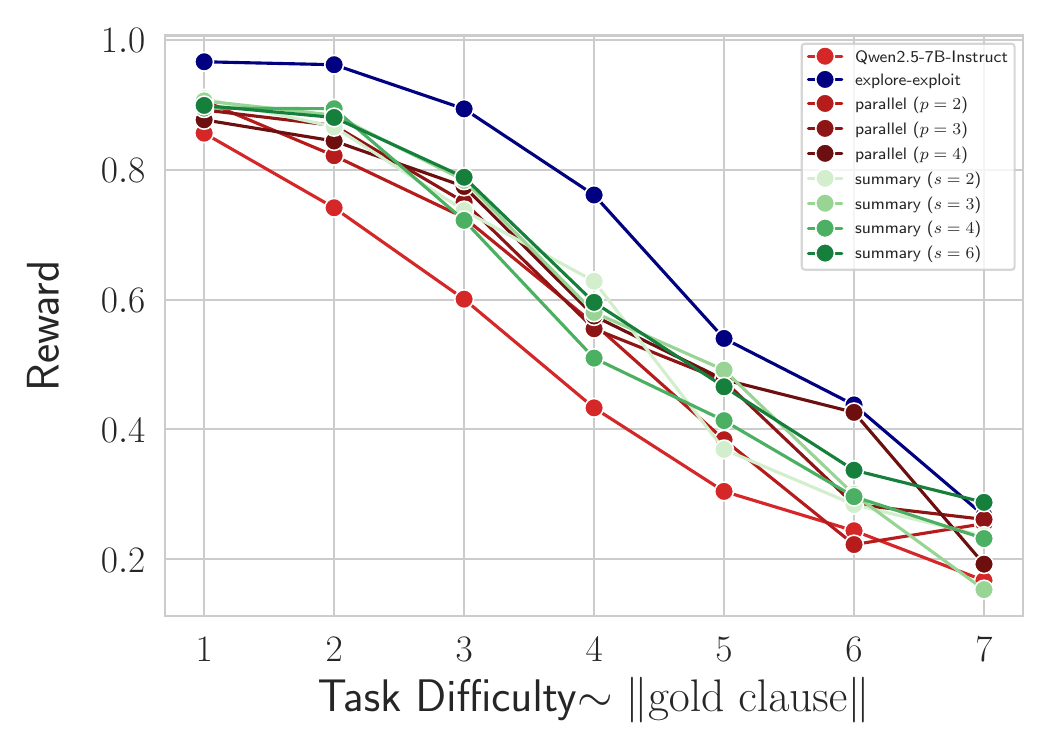}
    \caption{\tmaxsat}
  \end{subfigure}\hfill
  
  \caption{
Task difficulty variations for \tsearch{} (a), \ttree{} (b), and \tmaxsat{} (c).
We generate multiple task instances by varying key difficulty-controlling parameters---the peak width for \tsearch{}, the ratio of good to trap gateways for \ttree{}, and the size of the gold clause for \tmaxsat{}---with harder settings yielding lower baseline performance.
Across all tasks and difficulty levels, the parallel method (same total budget, split across independent threads) and the summary method consistently improve over single-thread execution of \texttt{Qwen2.5-7B-Instruct}, 
narrowing the gap to simple explore--exploit baselines.
Budget $N=36$ is considered for these episodes.
}
\label{fig:task_difficulty_n36_details}
\end{figure*}

\subsection{Standard Errors}
\label{app:subsec:standard-errors}
Table~\ref{tab:all_tasks_n_48_ci} reports the standard errors for the results in Table~\ref{tab:all_tasks_n_48}, calculated using bootstrap resampling. The non-overlapping intervals indicate a clear separation in performance between the parallel method and the single-thread baseline.
Similarly, Table~\ref{tab:all_tasks_n_48_summary_ci} details the standard errors for the summary method experiments (Table~\ref{tab:all_tasks_n_48_summary}), also derived via bootstrap resampling.

\begin{table*}[t]
\caption{Results of Table~\ref{tab:all_tasks_n_48} along with their standard errors in parentheses. Text color denotes performance relative to the single-thread execution ($p=1$): \textbf{\color{blue}blue} indicates an improvement with non-overlapping error intervals; \textbf{black} indicates an improvement where intervals overlap; and \textbf{\color{red}red} indicates performance below the single-thread execution. Bold values represent the best result for each model across $p$ values.}
\label{tab:all_tasks_n_48_ci}
\begin{adjustbox}{width=\linewidth}
\begin{tabular}{l|rrrr||rrrr||rrrr}
\toprule
 & \multicolumn{4}{c||}{\tsearch} & \multicolumn{4}{c||}{\ttree} & \multicolumn{4}{c}{\tmaxsat} \\
Model & $p=1$ & $p=2$ & $p=3$ & $p=4$ & $p=1$ & $p=2$ & $p=3$ & $p=4$ & $p=1$ & $p=2$ & $p=3$ & $p=4$ \\
\midrule
\midrule
{\small\texttt{Qwen2.5-7B-Instruct}} & {\color{black} 0.33} {\tiny \color{black} ($5$\%)} & {\color{black} 0.52} {\tiny \color{blue} ($7$\%)} & \textbf{{\color{black} 0.74} {\tiny \color{blue} ($7$\%)}} & {\color{black} 0.66} {\tiny \color{blue} ($9$\%)} & {\color{black} 0.71} {\tiny \color{black} ($3$\%)} & {\color{black} 0.83} {\tiny \color{blue} ($4$\%)} & {\color{black} 0.93} {\tiny \color{blue} ($2$\%)} & \textbf{{\color{black} 0.96} {\tiny \color{blue} ($2$\%)}} & {\color{black} 0.45} {\tiny \color{black} ($4$\%)} & {\color{black} 0.53} {\tiny \color{black} ($5$\%)} & {\color{black} 0.58} {\tiny \color{blue} ($6$\%)} & \textbf{{\color{black} 0.63} {\tiny \color{blue} ($6$\%)}} \\
{\small\texttt{Qwen3-4B-Instruct}} & {\color{black} 0.27} {\tiny \color{black} ($4$\%)} & \textbf{{\color{black} 0.52} {\tiny \color{blue} ($7$\%)}} & {\color{black} 0.45} {\tiny \color{blue} ($8$\%)} & {\color{black} 0.30} {\tiny \color{black} ($8$\%)} & {\color{black} 0.91} {\tiny \color{black} ($2$\%)} & {\color{black} 0.85} {\tiny \color{red} ($3$\%)} & {\color{black} 0.84} {\tiny \color{red} ($4$\%)} & \textbf{{\color{black} 0.91} {\tiny \color{black} ($3$\%)}} & {\color{black} 0.39} {\tiny \color{black} ($4$\%)} & {\color{black} 0.58} {\tiny \color{blue} ($5$\%)} & \textbf{{\color{black} 0.66} {\tiny \color{blue} ($6$\%)}} & {\color{black} 0.57} {\tiny \color{blue} ($6$\%)} \\
{\small\texttt{Qwen3-8B}} & {\color{black} 0.59} {\tiny \color{black} ($6$\%)} & {\color{black} 0.71} {\tiny \color{blue} ($6$\%)} & {\color{black} 0.57} {\tiny \color{red} ($9$\%)} & \textbf{{\color{black} 0.72} {\tiny \color{blue} ($7$\%)}} & {\color{black} 0.68} {\tiny \color{black} ($3$\%)} & {\color{black} 0.76} {\tiny \color{blue} ($4$\%)} & {\color{black} 0.88} {\tiny \color{blue} ($3$\%)} & \textbf{{\color{black} 0.93} {\tiny \color{blue} ($3$\%)}} & {\color{black} 0.48} {\tiny \color{black} ($4$\%)} & {\color{black} 0.70} {\tiny \color{blue} ($4$\%)} & {\color{black} 0.69} {\tiny \color{blue} ($4$\%)} & \textbf{{\color{black} 0.72} {\tiny \color{blue} ($4$\%)}} \\
{\small\texttt{gemini-2.5-flash-lite}} & {\color{black} 0.41} {\tiny \color{black} ($5$\%)} & \textbf{{\color{black} 0.58} {\tiny \color{blue} ($7$\%)}} & {\color{black} 0.57} {\tiny \color{blue} ($8$\%)} & {\color{black} 0.55} {\tiny \color{black} ($10$\%)} & {\color{black} 0.63} {\tiny \color{black} ($3$\%)} & {\color{black} 0.75} {\tiny \color{blue} ($4$\%)} & {\color{black} 0.80} {\tiny \color{blue} ($5$\%)} & \textbf{{\color{black} 0.85} {\tiny \color{blue} ($4$\%)}} & {\color{black} 0.62} {\tiny \color{black} ($4$\%)} & \textbf{{\color{black} 0.71} {\tiny \color{black} ($5$\%)}} & {\color{black} 0.70} {\tiny \color{black} ($5$\%)} & {\color{black} 0.68} {\tiny \color{black} ($6$\%)} \\
{\small\texttt{gpt-5-nano}} & {\color{black} 0.31} {\tiny \color{black} ($5$\%)} & {\color{black} 0.47} {\tiny \color{blue} ($7$\%)} & {\color{black} 0.55} {\tiny \color{blue} ($9$\%)} & \textbf{{\color{black} 0.74} {\tiny \color{blue} ($8$\%)}} & {\color{black} 0.69} {\tiny \color{black} ($3$\%)} & {\color{black} 0.83} {\tiny \color{blue} ($4$\%)} & {\color{black} 0.91} {\tiny \color{blue} ($3$\%)} & \textbf{{\color{black} 0.96} {\tiny \color{blue} ($2$\%)}} & {\color{black} 0.62} {\tiny \color{black} ($4$\%)} & {\color{black} 0.66} {\tiny \color{black} ($4$\%)} & {\color{black} 0.60} {\tiny \color{red} ($5$\%)} & \textbf{{\color{black} 0.72} {\tiny \color{blue} ($4$\%)}} \\
{\small\texttt{gpt-5-mini}} & {\color{black} 0.77} {\tiny \color{black} ($5$\%)} & {\color{black} 0.82} {\tiny \color{black} ($5$\%)} & {\color{black} 0.87} {\tiny \color{black} ($5$\%)} & \textbf{{\color{black} 0.91} {\tiny \color{blue} ($5$\%)}} & {\color{black} 0.64} {\tiny \color{black} ($3$\%)} & {\color{black} 0.84} {\tiny \color{blue} ($3$\%)} & {\color{black} 0.90} {\tiny \color{blue} ($3$\%)} & \textbf{{\color{black} 0.91} {\tiny \color{blue} ($3$\%)}} & {\color{black} 0.67} {\tiny \color{black} ($4$\%)} & {\color{black} 0.66} {\tiny \color{red} ($5$\%)} & \textbf{{\color{black} 0.82} {\tiny \color{blue} ($4$\%)}} & {\color{black} 0.67} {\tiny \color{black} ($6$\%)} \\
{\small\texttt{gpt-5}} & {\color{black} 0.72} {\tiny \color{black} ($5$\%)} & \textbf{{\color{black} 0.83} {\tiny \color{blue} ($5$\%)}} & {\color{black} 0.77} {\tiny \color{black} ($7$\%)} & {\color{black} 0.75} {\tiny \color{black} ($8$\%)} & {\color{black} 0.63} {\tiny \color{black} ($5$\%)} & {\color{black} 0.74} {\tiny \color{blue} ($6$\%)} & {\color{black} 0.78} {\tiny \color{blue} ($7$\%)} & \textbf{{\color{black} 0.94} {\tiny \color{blue} ($4$\%)}} & {\color{black} 0.78} {\tiny \color{black} ($5$\%)} & \textbf{{\color{black} 0.82} {\tiny \color{black} ($6$\%)}} & {\color{black} 0.74} {\tiny \color{red} ($7$\%)} & {\color{black} 0.76} {\tiny \color{red} ($7$\%)} \\
\midrule
{\small\texttt{Qwen3-8B-medium}} & {\color{black} 0.29} {\tiny \color{black} ($4$\%)} & {\color{black} 0.35} {\tiny \color{black} ($6$\%)} & \textbf{{\color{black} 0.38} {\tiny \color{black} ($7$\%)}} & {\color{black} 0.30} {\tiny \color{black} ($7$\%)} & {\color{black} 0.70} {\tiny \color{black} ($5$\%)} & {\color{black} 0.85} {\tiny \color{blue} ($5$\%)} & {\color{black} 0.94} {\tiny \color{blue} ($3$\%)} & \textbf{{\color{black} 0.94} {\tiny \color{blue} ($4$\%)}} & {\color{black} 0.54} {\tiny \color{black} ($5$\%)} & \textbf{{\color{black} 0.70} {\tiny \color{blue} ($6$\%)}} & {\color{black} 0.69} {\tiny \color{blue} ($7$\%)} & {\color{black} 0.64} {\tiny \color{black} ($8$\%)} \\
{\small\texttt{gpt-5-nano-medium}} & {\color{black} 0.48} {\tiny \color{black} ($6$\%)} & {\color{black} 0.75} {\tiny \color{blue} ($6$\%)} & {\color{black} 0.88} {\tiny \color{blue} ($5$\%)} & \textbf{{\color{black} 0.96} {\tiny \color{blue} ($3$\%)}} & {\color{black} 0.68} {\tiny \color{black} ($3$\%)} & {\color{black} 0.80} {\tiny \color{blue} ($4$\%)} & {\color{black} 0.89} {\tiny \color{blue} ($3$\%)} & \textbf{{\color{black} 0.95} {\tiny \color{blue} ($2$\%)}} & {\color{black} 0.54} {\tiny \color{black} ($5$\%)} & {\color{black} 0.59} {\tiny \color{black} ($6$\%)} & {\color{black} 0.61} {\tiny \color{black} ($6$\%)} & \textbf{{\color{black} 0.65} {\tiny \color{black} ($7$\%)}} \\
{\small\texttt{gpt-5-mini-medium}} & {\color{black} 0.58} {\tiny \color{black} ($6$\%)} & {\color{black} 0.84} {\tiny \color{blue} ($5$\%)} & \textbf{{\color{black} 0.86} {\tiny \color{blue} ($5$\%)}} & {\color{black} 0.79} {\tiny \color{blue} ($8$\%)} & {\color{black} 0.54} {\tiny \color{black} ($4$\%)} & {\color{black} 0.58} {\tiny \color{black} ($7$\%)} & \textbf{{\color{black} 0.67} {\tiny \color{blue} ($8$\%)}} & {\color{black} 0.66} {\tiny \color{black} ($9$\%)} & {\color{black} 0.66} {\tiny \color{black} ($5$\%)} & {\color{black} 0.65} {\tiny \color{red} ($7$\%)} & \textbf{{\color{black} 0.73} {\tiny \color{black} ($7$\%)}} & {\color{black} 0.71} {\tiny \color{black} ($7$\%)} \\
\midrule
{\small\texttt{explore-exploit}} & \multicolumn{4}{c||}{{\color{black} 0.97} {\tiny \color{black} ($0.2$\%)}} & \multicolumn{4}{c||}{{\color{black} 0.96} {\tiny \color{black} ($1$\%)}} & \multicolumn{4}{c}{{\color{black} 0.84} {\tiny \color{black} ($2$\%)}} \\
\bottomrule
\end{tabular}
\end{adjustbox}
\end{table*}
\begin{table}[t]
\centering
\begin{minipage}{0.8\linewidth}
\caption{Results of Table~\ref{tab:all_tasks_n_48_summary} along with their standard errors in parentheses. Text color denotes performance relative to no summary execution: \textbf{\color{blue}blue} indicates an improvement with non-overlapping error intervals; \textbf{black} indicates an improvement where intervals overlap; and \textbf{\color{red}red} indicates performance below the no summary execution.}
\label{tab:all_tasks_n_48_summary_ci}
\begin{adjustbox}{width=\linewidth}
\begin{tabular}{rc|c||c||c}
\toprule
\multicolumn{2}{c|}{Method}                            & \tsearch                                                   & \ttree                                                     & \tmaxsat                                                   \\ \midrule
\multicolumn{2}{c|}{\texttt{Qwen2.5-7B-Instruct}}      & {\color{black} 0.33} {\tiny \color{black} ($5$\%)}         & {\color{black} 0.71} {\tiny \color{black} ($3$\%)}         & {\color{black} 0.45} {\tiny \color{black} ($4$\%)}         \\ \midrule
\multirow{4}{*}{summary}             & $s=2$           & {\color{black} 0.45} {\tiny \color{blue} ($6$\%)}          & {\color{black} 0.78} {\tiny \color{black} ($4$\%)}         & {\color{black} 0.55} {\tiny \color{blue} ($4$\%)}          \\
                                     & $s=3$           & {\color{black} 0.43} {\tiny \color{black} ($5$\%)}         & {\color{black} 0.82} {\tiny \color{blue} ($4$\%)}          & {\color{black} 0.57} {\tiny \color{blue} ($4$\%)}          \\
                                     & $s=4$           & {\color{black} 0.52} {\tiny \color{blue} ($5$\%)}          & \textbf{{\color{black} 0.82} {\tiny \color{blue} ($4$\%)}} & \textbf{{\color{black} 0.66} {\tiny \color{blue} ($4$\%)}} \\
                                     & $s=6$           & \textbf{{\color{black} 0.62} {\tiny \color{blue} ($5$\%)}} & {\color{black} 0.80} {\tiny \color{blue} ($3$\%)}          & {\color{black} 0.60} {\tiny \color{blue} ($4$\%)}          \\ \hline
\multicolumn{2}{c|}{\texttt{explore-exploit baseline}} & {\color{black} 0.97} {\tiny \color{black} ($0.2$\%)}       & {\color{black} 0.96} {\tiny \color{black} ($1$\%)}         & {\color{black} 0.84} {\tiny \color{black} ($2$\%)}         \\ \bottomrule

\end{tabular}
\end{adjustbox}
\end{minipage}
\end{table}


\section{Details on Providing Models with Summary}
\label{app:summary}

We include our prompts for summary generation in Figures \ref{fig:prompt_summary_hill}, \ref{fig:prompt_summary_tree}, and \ref{fig:prompt_summary_sat}. These prompts rely on the query history and previously revealed context, providing no new information to the model.

\begin{figure}[ht]
\begin{tcolorbox}[
  enhanced, 
  colframe=red!75!black, 
  colback=gray!10, 
  coltitle=white, 
  colbacktitle=red!75!black, 
  width=\linewidth, 
  arc=2mm, 
  auto outer arc, 
  boxrule=0.5pt, 
  left=10pt, 
  right=10pt, 
  drop shadow={black!50!white},
  top=10pt, 
  bottom=10pt, 
  title=\textbf{Summarization prompt for \tsearch{} task}, 
  fontupper=\small,      
  fonttitle=\footnotesize\bfseries, 
  title code={\node[rounded corners, fill=blue!75!black, draw=none, text=white] at (frame.title) {\textbf{xxx}};}, 
  attach boxed title to top center={yshift=-2mm}, 
  boxed title style={sharp corners, size=small}, 
]

\#\#\# MISSION HAND-OFF \#\#\#

A previous agent spent \textcolor{blue}{$q$}  queries exploring the domain [0, 10] but was terminated. You are a NEW agent brought in to take over. You must find the global maximum.

**Data Collected by Previous Agent (Sorted by x):**

x=\textcolor{blue}{$x_1$}  $\to$ f(x)=\textcolor{blue}{$y_1$} 

x=\textcolor{blue}{$x_2$}  $\to$ f(x)=\textcolor{blue}{$y_2$} 

\textcolor{blue}{$\cdot$}

\textcolor{blue}{$\cdot$}

\textcolor{blue}{$\cdot$}

**Unexplored Gaps:**

* Interval [\textcolor{blue}{$l_1$} , \textcolor{blue}{$r_1$} ] (Gap size: \textcolor{blue}{$d_1$} ).

* Interval [\textcolor{blue}{$l_2$} , \textcolor{blue}{$r_2$} ] (Gap size: \textcolor{blue}{$d_2$} ).

\textcolor{blue}{$\cdot$}

\textcolor{blue}{$\cdot$}

\textcolor{blue}{$\cdot$}

**Instructions for the New Agent:**

1. Review the history. Did the previous agent get stuck in a local maximum?

2. Formulate a fresh plan to utilize your remaining budget.

You have \textcolor{blue}{$q'$} queries remaining. Review the Hand-off data above and output your next query.

\end{tcolorbox}
\caption{\textbf{The prompt used to summarize model interactions in the \tsearch{} task.} As the model can deduce this information from the query history and previously revealed context, this prompt provides nothing new. Blue text indicates dynamic variables specific to the problem instance or generated queries.}
\vspace{-5pt}
\label{fig:prompt_summary_hill}
\end{figure}
\begin{figure}[ht]
\begin{tcolorbox}[
  enhanced, 
  colframe=red!75!black, 
  colback=gray!10, 
  coltitle=white, 
  colbacktitle=red!75!black, 
  width=\linewidth, 
  arc=2mm, 
  auto outer arc, 
  boxrule=0.5pt, 
  left=10pt, 
  right=10pt, 
  drop shadow={black!50!white},
  top=10pt, 
  bottom=10pt, 
  title=\textbf{Summarization prompt for \ttree{} task}, 
  fontupper=\small,      
  fonttitle=\footnotesize\bfseries, 
  title code={\node[rounded corners, fill=blue!75!black, draw=none, text=white] at (frame.title) {\textbf{xxx}};}, 
  attach boxed title to top center={yshift=-2mm}, 
  boxed title style={sharp corners, size=small}, 
]
\#\#\# MISSION HAND-OFF \#\#\#

A previous agent spent \textcolor{blue}{$q$} queries exploring the graph but was terminated. You are a NEW agent brought in to take over. You must find the maximum value.

**Query History (in order):**

* Node \textcolor{blue}{$v_1$} $\to$ value \textcolor{blue}{$x_1$}

* Node \textcolor{blue}{$v_2$} $\to$ value \textcolor{blue}{$x_2$}

\textcolor{blue}{$\cdot$}

\textcolor{blue}{$\cdot$}

\textcolor{blue}{$\cdot$}

**Current Best Found:**

Node \textcolor{blue}{$v_3$} with Value \textcolor{blue}{$x_3$}.

**Next Nodes to Query (grouped by height):**
Each listed node is actionable (unknown, with a known neighbor). Node order is shuffled within each height.

* height 1: [\textcolor{blue}{$v_4$}, \textcolor{blue}{$v_5$}, \textcolor{blue}{$\cdots$}]

* height 2: [\textcolor{blue}{$v_6$}, \textcolor{blue}{$v_7$}, \textcolor{blue}{$\cdots$}]

\textcolor{blue}{$\cdot$}

\textcolor{blue}{$\cdot$}

\textcolor{blue}{$\cdot$}

**Instructions for the New Agent:**
Pick a node to query next. Do not blindly continue the last path; consider switching to a different height/frontier.
You have \textcolor{blue}{$q'$} queries remaining. Output your next query.
\end{tcolorbox}
\caption{\textbf{The prompt used to summarize model interactions in the \ttree{} task.} As the model can deduce this information from the query history and previously revealed context, this prompt provides nothing new. Blue text indicates dynamic variables specific to the problem instance or generated queries.}
\vspace{-5pt}
\label{fig:prompt_summary_tree}
\end{figure}
\begin{figure}[ht]
\begin{tcolorbox}[
  enhanced, 
  colframe=red!75!black, 
  colback=gray!10, 
  coltitle=white, 
  colbacktitle=red!75!black, 
  width=\linewidth, 
  arc=2mm, 
  auto outer arc, 
  boxrule=0.5pt, 
  left=10pt, 
  right=10pt, 
  drop shadow={black!50!white},
  top=10pt, 
  bottom=10pt, 
  title=\textbf{Summarization prompt for \tmaxsat{} task}, 
  fontupper=\small,      
  fonttitle=\footnotesize\bfseries, 
  title code={\node[rounded corners, fill=blue!75!black, draw=none, text=white] at (frame.title) {\textbf{xxx}};}, 
  attach boxed title to top center={yshift=-2mm}, 
  boxed title style={sharp corners, size=small}, 
]

\#\#\# MISSION HAND-OFF \#\#\#

A previous agent spent \textcolor{blue}{$t$} queries exploring the MAX-SAT black-box but was terminated. You are a NEW agent brought in to take over.

**Progress:** tried \textcolor{blue}{$t$} assignments. Best score so far: **\textcolor{blue}{$s^*$}**.

**Query History (in order):**

* step 1: score=\textcolor{blue}{$s_1$} | assignment=\textcolor{blue}{$bitstring_1$}

* step 2: score=\textcolor{blue}{$s_2$} | assignment=\textcolor{blue}{$bitstring_2$}

\textcolor{blue}{$\cdot$}

\textcolor{blue}{$\cdot$}

\textcolor{blue}{$\cdot$}

**Best Assignment Found (inner x0..xN-1 bitstring):**

\textcolor{blue}{$bitstring^*$}

**Coverage Summary:**

Counts below summarize how often each variable took value 0 or 1 across the queried assignments.

- Variables with fewest 0s observed: x\textcolor{blue}{${a_1}$} (was 0 in \textcolor{blue}{$c_{a_1}$}/\textcolor{blue}{$t$}), x\textcolor{blue}{${a_2}$} (was 0 in \textcolor{blue}{$c_{a_2}$}/\textcolor{blue}{$t$}), \textcolor{blue}{$\cdots$}

- Variables with fewest 1s observed: x\textcolor{blue}{${b_1}$} (was 1 in \textcolor{blue}{$c_{b_1}$}/\textcolor{blue}{$t$}), x\textcolor{blue}{${b_2}$} (was 1 in \textcolor{blue}{$c_{b_2}$}/\textcolor{blue}{$t$}), \textcolor{blue}{$\cdots$}

You have \textcolor{blue}{$q'$} queries remaining.

\end{tcolorbox}
\caption{\textbf{The prompt used to summarize model interactions in the \tmaxsat{} task.} As the model can deduce this information from the query history and previously revealed context, this prompt provides nothing new. Blue text indicates dynamic variables specific to the problem instance or generated queries.}
\vspace{-5pt}
\label{fig:prompt_summary_sat}
\end{figure}

\clearpage






\section{Baseline Details}

In this section, we provide implementation details for the explore–exploit baselines and study the sensitivity of their performance to parameter choices.

\label{app:baseline_details}
\subsection{Pseudocodes}
We present pseudocode for the \texttt{explore-exploit} baselines to ensure reproducibility. These algorithms share the same interface as the language models: they receive the query budget $N$, the task structure, and the query history as input, and output the next query.

\paragraph{\tsearch{}.} Algorithm~\ref{alg:baseline:hill} details the query selection strategy for \tsearch{}. The method partitions the budget into exploration and exploitation phases. During exploration, it uses a stratified sampling strategy. During exploitation, it defines a local window around the best point observed so far and samples uniformly within that region.

\begin{algorithm}[tb]
   \caption{\texttt{explore-exploit} Baseline for \tsearch{}}
   \label{alg:baseline:hill}
\begin{algorithmic}[1]
    \STATE {\bfseries Input:} Query budget $N$, Query history $\mathcal{H} = \{(x_i, y_i)\}_{i=1}^{t}$
    \STATE {\bfseries Output:} Next query point $x_{\text{next}}$
    \STATE {\bfseries Parameters:} Domain $\mathcal{D}=[L, R]$ where $L=0, R=10$, Exploration fraction $\alpha = 0.8$, Window fraction $\beta = 0.05$
    
    \STATE $t \gets |\mathcal{H}|$ \textit{\# Current query count}
    \STATE $N_{\text{explore}} \gets \lfloor \alpha N \rfloor$
    
    \IF{$t < N_{\text{explore}}$}
       \STATE \textit{\# Exploration: Stratified Sampling}
       \STATE $\Delta = (R-L) / N_{\text{explore}}$
       \STATE $l \gets L + t \cdot \Delta$
       \STATE $r \gets l + \Delta$
       \STATE Sample $x_{\text{next}} \sim \text{Uniform}(l, r)$
    \ELSE
       \STATE \textit{\# Exploitation: Local Search}
       \STATE $x^* \gets \operatorname*{argmax}_{x} \{y \mid (x, y) \in \mathcal{H}\}$ \textit{\# Best observed point}
       \STATE $w \gets (R-L) \cdot \beta$
       \STATE Sample $x_{\text{next}} \sim \text{Uniform}(x^* - w/2, x^* + w/2)$
    \ENDIF
    \STATE \textbf{return} $x_{\text{next}}$
\end{algorithmic}
\end{algorithm}

\paragraph{\ttree{}.} Algorithm~\ref{alg:baseline:tree} outlines the node selection process for the \ttree{} baseline. This method identifies frontier nodes (unvisited nodes with visited parents) and assigns a score to each candidate equal to the observed value of its parent. The algorithm then selects the next query via softmax sampling over these scores, controlled by a temperature parameter $\tau$.

\begin{algorithm}[tb]
   \caption{\texttt{explore-exploit} Baseline for \ttree{}}
   \label{alg:baseline:tree}
\begin{algorithmic}[1]
   \STATE {\bfseries Input:} Query budget $N$, Tree $T=(V, E)$ rooted at $r$, Query history $\mathcal{H}$
   \STATE {\bfseries Hyperparameter:} Temperature $\tau$
   \STATE {\bfseries Output:} Next node to query $u_{\text{next}}$
   
   \STATE \textit{\# Identify queried state}
   \STATE Let $V_{\text{obs}} \subseteq V$ be the set of nodes queried in $\mathcal{H}$
   \STATE Let $R: V_{\text{obs}} \to \mathbb{R}$ denote the observed values
   
   \STATE \textit{\# Identify and score candidates}
   \STATE Initialize candidate set $\mathcal{C} \leftarrow \emptyset$
   \FOR{each $u \in V \setminus V_{\text{obs}}$ such that its parent $p(u) \in V_{\text{obs}}$}
       \STATE \textit{\# Score candidate based on parent value}
       \STATE $S_u \leftarrow R(p(u))$
       \STATE $\mathcal{C} \leftarrow \mathcal{C} \cup \{u\}$
   \ENDFOR
   
   \STATE \textit{\# Weighted sampling}
   \STATE Define probability $P(u)$ for each $u \in \mathcal{C}$:
   \STATE \quad $P(u) \propto \exp(S_u / \tau)$
   
   \STATE Sample $u_{\text{next}} \sim P$
   \STATE \textbf{return} $u_{\text{next}}$
\end{algorithmic}
\end{algorithm}

\paragraph{\tmaxsat{}.} Algorithm~\ref{alg:baseline:maxsat} describes the variable assignment procedure for \tmaxsat{}. In the exploration phase, the algorithm generates variable assignments uniformly at random. In the exploitation phase, it identifies the highest-value assignment in the history and flips a single random bit to generate a local neighbor as the next query.

\begin{algorithm}[tb]
   \caption{\texttt{explore-exploit} Baseline for \tmaxsat{}}
   \label{alg:baseline:maxsat}
\begin{algorithmic}[1]
   \STATE {\bfseries Input:} Query budget $N$, Number of variables $n$, Exploration fraction $\alpha$
   \STATE {\bfseries Input:} Query history $\mathcal{H} = \{(x_i, y_i)\}_{i=1}^{t}$
   \STATE {\bfseries Output:} Next query assignment $x_{\text{next}}$

   \STATE $t \gets |\mathcal{H}|$ \textit{\# Current query count}
   \STATE $N_{\text{explore}} \gets \lfloor \alpha \cdot N \rfloor$

   \IF{$t < N_{\text{explore}}$}
       \STATE \textit{\# Exploration Phase: Random Sampling}
       \STATE Sample $x_{\text{next}} \sim \text{Uniform}(\{0, 1\}^n)$
   \ELSE
       \STATE \textit{\# Exploitation Phase: Local Search}
       \STATE $x^* \leftarrow \operatorname*{argmax}_{(x', y') \in \mathcal{H}} y'$ \textit{\# Best found assignment}

       \STATE Sample variable index $k \sim \text{Uniform}(0, n-1)$
       \STATE Generate $x_{\text{next}}$ by flipping the $k$-th bit of $x^*$
   \ENDIF

   \STATE \textbf{return} $x_{\text{next}}$
\end{algorithmic}
\end{algorithm}

\subsection{Parameters}
\label{app:baseline_param}
In this section, we evaluate the sensitivity of the \texttt{explore-exploit} baselines to different parameter settings across all tasks. We conduct 500 evaluation episodes for the baselines using the same problem instances described in \S\ref{subsec:eval}. These results are summarized in Figure~\ref{fig:baseline_param}.

\begin{figure*}[t]
  \centering
  \begin{subfigure}[t]{0.32\textwidth}
    \centering
    \includegraphics[width=\linewidth]{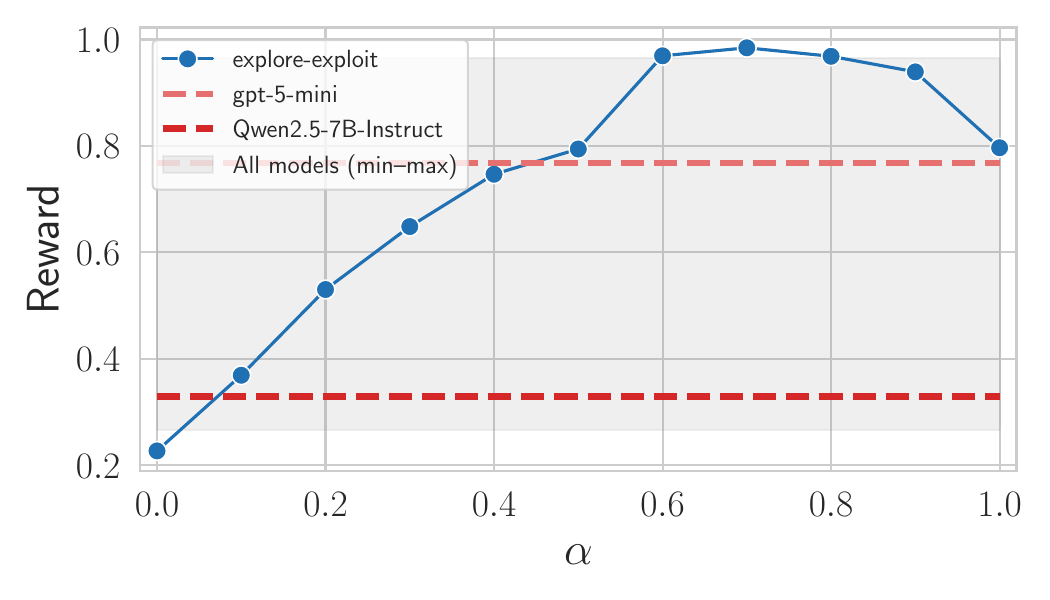}
    \caption{\tsearch}
  \end{subfigure}\hfill
  \begin{subfigure}[t]{0.32\textwidth}
    \centering
    \includegraphics[width=\linewidth]{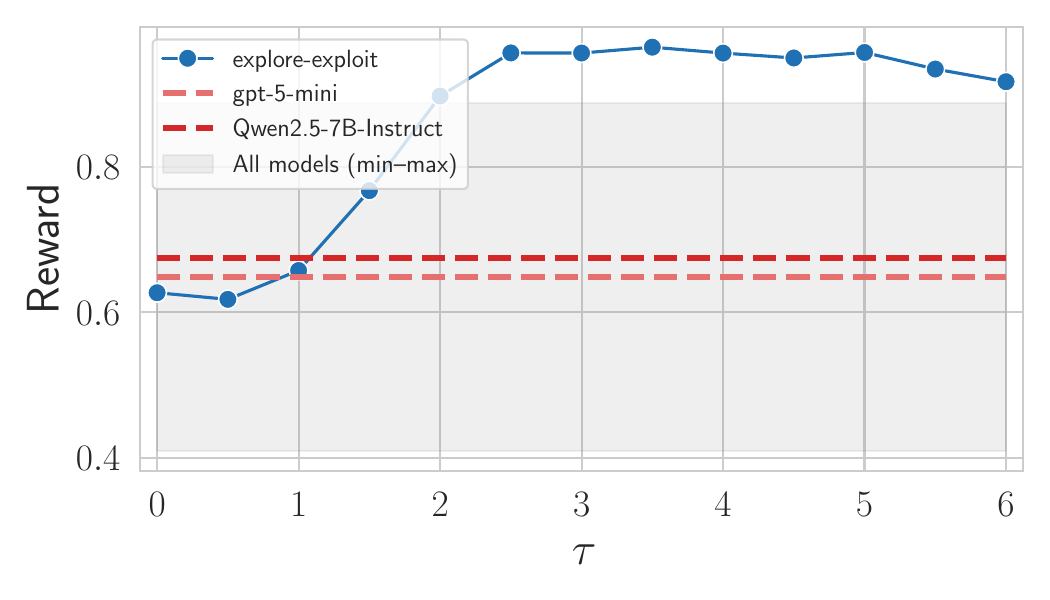}
    \caption{\ttree}
  \end{subfigure}\hfill
  \begin{subfigure}[t]{0.32\textwidth}
    \centering
    \includegraphics[width=\linewidth]{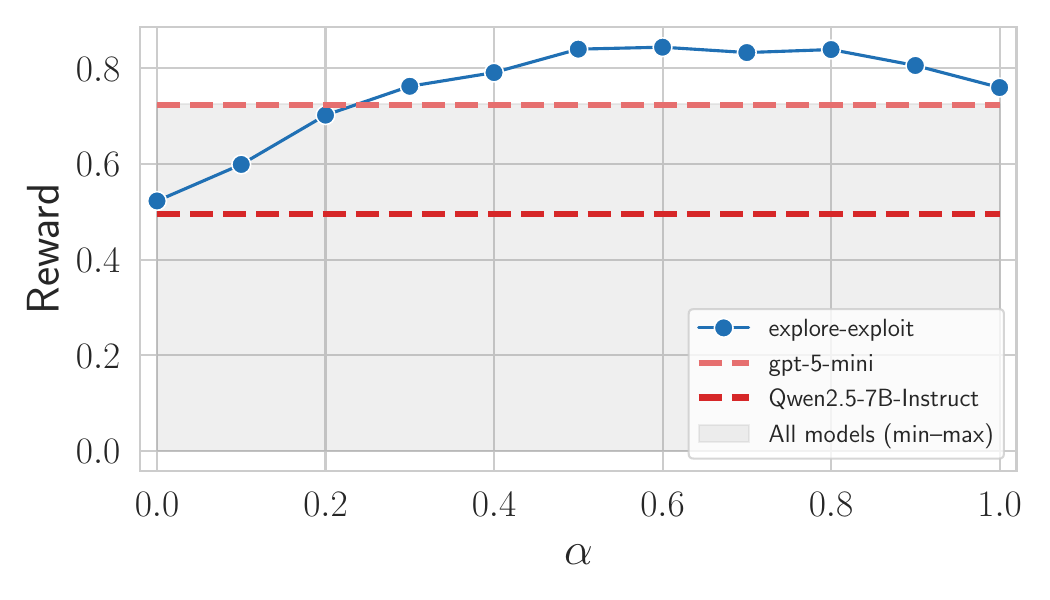}
    \caption{\tmaxsat}
  \end{subfigure}\hfill
    \caption{Performance of the \texttt{explore-exploit} baselines with a query budget of $N=48$ across all tasks under varying parameter settings. The baselines consistently outperform the language models across a wide range of parameter values.}
  \label{fig:baseline_param}
\end{figure*}

\paragraph{\tsearch{}.}
We vary the exploration fraction $\alpha \in \{0, 0.1, \dots, 1\}$, which represents the proportion of queries allocated to stratified random sampling. Throughout the paper, we set $\alpha=0.8$. Figure~\ref{fig:baseline_param}~(a) illustrates the reward as a function of $\alpha$ given a budget of $N=48$. The results indicate that settings with $0.5 \le \alpha \le 1$ outperform the LMs, and any $\alpha \ge 0.1$ yields results superior to the \texttt{Qwen2.5-7B-Instruct} model. This demonstrates that our baseline, despite its simplicity, achieves consistent performance across parameters and can be tuned to maximize results.

\paragraph{\ttree{}.}
We examine the effect of the soft-max temperature $\tau$, which controls the next-node selection. We vary $\tau \in \{0, 0.5, \dots, 6\}$, with $\tau=4$ used as the default in the paper. Figure~\ref{fig:baseline_param}~(b) plots the reward against $\tau$ for a budget of $N=48$. We observe that all variations, including the greedy strategy ($\tau = 0$), perform comparably to \texttt{Qwen2.5-7B-Instruct}, while values in the range $2 \le \tau \le 6$ surpass the performance of the best LM.

\paragraph{\tmaxsat{}.}
We vary the fraction $\alpha \in \{0, 0.1, \dots, 1\}$, representing the proportion of queries where we randomly sample an assignment. We utilized $\alpha=0.5$ for the main experiments. Figure~\ref{fig:baseline_param}~(c) displays the reward for each $\alpha$ with a budget of $N=48$. All variations achieve better performance than \texttt{Qwen2.5-7B-Instruct}, while settings where $\alpha \ge 0.3$ outperform all evaluated LMs.

\section{Theoretical Results}
\label{app:proofs}

We will first restate Theorem~\ref{thm:cxalpha} and prove it. Then for tasks \tmaxsat{} and \ttree{}, we provide theoretical proof that parallelization has no gain over an optimal single-thread strategy.

\begingroup
    \renewcommand{\thetheorem}{\ref{thm:cxalpha}}
    \begin{theorem}
    \thmcxalpha
    \end{theorem}
\endgroup

\begin{proof}[Proof of Theorem~\ref{thm:cxalpha}]
Substituting $q(x) = cx^\alpha$, we observe that $q(x) = p^\alpha q(x/p)$. By defining $y := q(x/p)$, the condition in (\ref{eq:p-success-condition}) can be rewritten in terms of the single sub-trace probability $y$:
\begin{equation*}
    1 - (1 - y)^p > p^\alpha y.
\end{equation*}
We define the function $h(y)$ as the difference between the parallel success probability (LHS) and the single success probability (RHS):
\begin{equation*}
    h(y) = 1 - (1 - y)^p - p^\alpha y.
\end{equation*}
We analyze the shape of $h(y)$ to determine where it is positive. First, we evaluate the function at the origin as $h(0) = 1 - (1)^p - 0 = 0$. Next, we examine the first derivative with respect to $y$: $h'(y) = p(1 - y)^{p-1} - p^\alpha$.
Evaluating the derivative at $y=0$:
\begin{equation*}
    h'(0) = p - p^\alpha.
\end{equation*}
Since $\alpha < 1$ and $p > 1$, we have $p > p^\alpha$. Consequently, $h'(0) > 0$. This implies that for small $y > 0$, $h(y)$ is positive, and parallelization is beneficial. Finally, we examine the second derivative to determine concavity:
\begin{equation*}
    h''(y) = -p(p-1)(1 - y)^{p-2}.
\end{equation*}
Since $p \ge 2$, $h''(y) < 0$ for all $y \in [0,1)$. This indicates that $h(y)$ is strictly concave. Because $h(y)$ starts at 0 with a positive slope ($h'(0) > 0$) and is strictly concave, the derivative $h'(y)$ is strictly decreasing. The function $h(y)$ will rise to a peak and then decrease. It follows that there exists a unique crossing point $y^*_p > 0$ such that $h(y) > 0$ for $y < y^*_p$ and $h(y) \le 0$ for $y \ge y^*_p$.

Mapping $y$ back to the budget $x$ via $y = c(x/p)^\alpha$, we define the threshold budget:
\begin{equation*}
    v_p = p \left( \frac{y^*_p}{c} \right)^{1/\alpha}.
\end{equation*}
Thus, parallelization is beneficial for $x < v_p$ and not beneficial for $x \ge v_p$.

For $p=2$, solving $h(y^*_2)=1 - (1 - y^*_2)^2 - 2^\alpha y^*_2=0$, we get $y^*_2 = 2 - 2^\alpha$. Then
\begin{equation*}
    v_2 = 2 \left( \frac{2 - 2^\alpha}{c} \right)^{1/\alpha}.
\end{equation*}
Using the Taylor expansion for $\alpha \approx 1$, we get $v_2 \approx 4\ln 2 \cdot \frac {1 - \alpha} {c}$.
\end{proof}


\paragraph{\ttree{}: Parallel exploration is a subset of single-thread exploration}
We demonstrate that any set of nodes explored by parallel threads is a valid search space for a single thread with an equivalent total budget.
Let there be $p$ parallel threads, where each thread $i$ has a budget of $N/p$.
Let $\mathcal{S}_i$ be the set of nodes explored by thread $i$.
By definition, each $\mathcal{S}_i$ forms a connected component that includes the root node.
The parallel method returns the best node found across all threads:
\[
r_{\text{parallel}}^* \;=\; \max_{i \in [p]} \; \max_{v \in \mathcal{S}_i} r(v).
\]
Consider a single-thread agent with a total budget $N$.
Let this agent explore the union of the sets visited by the parallel threads, defined as $\mathcal{S} = \bigcup_{i \in [p]} \mathcal{S}_i$.
The total size of this set satisfies $|\mathcal{S}| \leq \sum |\mathcal{S}_i| \leq N$.
Because every subset $\mathcal{S}_i$ is connected and shares the same root, their union $\mathcal{S}$ is also connected and rooted.
Therefore, a valid trajectory exists for a single thread to visit every node in $\mathcal{S}$ within budget $N$, and thus:
\[
\max_{v \in \mathcal{S}} r(v) \;\geq\; \max_{i \in [p]} \; \max_{v \in \mathcal{S}_i} r(v) \;=\; r_{\text{parallel}}^* .
\]
Thus, the parallel strategy does not achieve a better reward than the optimal single-thread strategy.

\paragraph{\tmaxsat{}: Queries of parallel threads can be asked in a single-thread interaction}
Consider $p$ parallel threads, each querying $N/p$ assignments.
Let $\mathcal{X}$ denote the set of all possible truth assignments.
For each thread $i \in [p]$, let $S_i \subseteq \mathcal{X}$ denote the set of queried assignments, with $|S_i| \le N/p$.
Thread $i$ achieves final reward $\max_{\mathbf{x} \in S_i} r(\mathbf{x})$.
We merge the solutions by taking the best across threads, achieving
\[
r_{\text{parallel}}^* \;=\; \max_{i \in [p]} \; \max_{\mathbf{x} \in S_i} r(\mathbf{x}).
\]
Now consider a single-thread interaction with budget $N$ that queries the union of all assignments queried by the parallel threads,
\[
\mathcal{S} \;=\; \bigcup_{i \in [p]} S_i,
\qquad \text{so that } |\mathcal{S}| \leq N.
\]
This single-thread episode achieves reward at least $\max_{\mathbf{x} \in \mathcal{S}} r(\mathbf{x})$, and therefore
\[
\max_{\mathbf{x} \in \mathcal{S}} r(\mathbf{x}) \;\geq\; r_{\text{parallel}}^* .
\]
Thus, selecting the best outcome among $p$ independent threads with total budget $N$ is upper-bounded by what a single-thread interaction with budget $N$ could achieve by querying the same set of assignments. This finishes the proof.




\end{document}